\documentclass[nohyperref]{article}
\usepackage[T1]{fontenc}

\usepackage{PRIMEarxiv}

\usepackage{microtype}
\usepackage{graphicx}
\usepackage{subfigure}
\usepackage{booktabs} 

\usepackage{algorithm}
\usepackage[mathscr]{eucal}

\usepackage{tabularx}
\usepackage{multirow}
\usepackage{enumerate}

\usepackage{colortbl}
\usepackage{dcolumn}
\newcolumntype{.}{D{.}{.}{1.3}}

\newcommand\dashrule{\leavevmode\xleaders\hbox{-}\hfill\kern0pt}

\def\diag{\operatorname{diag}}

\newcommand{\bu}{\bvec{u}}

\newcommand{\bA}{{\bf A}}
\newcommand{\bB}{{\bf B}}
\newcommand{\bC}{{\bf C}}

\newcommand{\bI}{{\bf I}}

\newcommand{\bL}{{\bf L}}

\newcommand{\bQ}{{\bf Q}}

\newcommand{\bS}{{\bf S}}

\newcommand{\bT}{{\bf T}}
\newcommand{\bU}{{\bf U}}
\newcommand{\bV}{{\bf V}}
\newcommand{\bW}{{\bf W}}
\newcommand{\bX}{{\bf X}}
\newcommand{\bY}{{\bf Y}}
\newcommand{\bZ}{{\bf Z}}

\newcommand{\btheta}{\mbox{\boldmath $\theta$}}

\newcommand{\1}{\mbox{\boldmath $1$}}

\newcommand{\be}{\begin{eqnarray}}
\newcommand{\ee}{\end{eqnarray}}

\newcommand{\matrixb}{\left[ \begin{array}}
\newcommand{\matrixe}{\end{array} \right]}

\newcommand{\tr}{\mathop{\rm tr}\nolimits}

\def\*{\circledast}

\newcommand{\bvec}[1]{\boldsymbol{#1}}

\def\vectorize{\operatorname{vec}}
\newcommand{\vtr}[1]{\vectorize\hspace{-.3ex}\left(#1\right)}

 \newcommand{\tensor}[1]{\boldsymbol{\mathscr{\MakeUppercase{#1}}}} 
\newcommand{\tA}{\tensor{A}}
\newcommand{\tB}{\tensor{B}}
\newcommand{\tC}{\tensor{C}}

\newcommand{\tG}{\tensor{G}}

\newcommand{\tU}{\tensor{U}}
\newcommand{\tV}{\tensor{V}}
\newcommand{\tW}{\tensor{W}}
\newcommand{\tX}{\tensor{X}}
\newcommand{\tY}{\tensor{Y}}
\newcommand{\tZ}{\tensor{Z}}

\newcommand{\minitab}[2][l]{\begin{tabular}{@{}#1}#2\end{tabular}}
\usepackage[vlined,ruled,commentsnumbered,algo2e]{algorithm2e}

\usepackage{aurical}

\usepackage{arydshln}

\newcommand{\ttprod}{\bullet}

\usepackage{stackengine,graphicx,scalerel}
\newcommand\CircArrowRight[1]{\stackengine{.25ex}{#1}{\hspace{-.2ex}{\CAR}}{O}{c}{F}{F}{L}}
\newcommand\CircArrowLeft[1]{\stackengine{.25ex}{#1}{\hspace{0.2ex}{\CAL}}{O}{c}{F}{F}{L}}
\newcommand\CAR{\scaleto{\circlearrowright}{1.1ex}}
\newcommand\CAL{\scaleto{\circlearrowleft}{1.1ex}}

\newcommand\circlellbracket{\mathopen{{\CircArrowRight{[}}}}
\newcommand\circlerrbracket{\mathclose{{\CircArrowLeft{]}}}}

\newcommand{\rememberlines}{\xdef\rememberedlines{\number\value{AlgoLine}}}
\newcommand{\resumenumbering}{\setcounter{AlgoLine}{\rememberedlines}}

\usepackage{amsmath}
\usepackage{amssymb}
\usepackage{mathtools}
\usepackage{amsthm}
\usepackage{pythonhighlight}

\newtheorem{theorem}{Theorem}[section]
\newtheorem{lemma}[theorem]{Lemma}
 \newtheorem{definition}[theorem]{Definition}
\theoremstyle{remark}
\newtheorem{remark}[theorem]{Remark}

\usepackage[textsize=tiny]{todonotes}
\usepackage{enumitem}

\graphicspath{{./TC/}{./TT/}{./images/}{./Resnet/}}

\newcounter{example} 
    
\newenvironment{example}
{\refstepcounter{example}\vspace{10pt}\par\noindent 
\textbf{Example \theexample\ }
}
{}%
  
\title{How to Train Unstable Looped Tensor Network}

 \author{
 \minitab[c]{Anh-Huy Phan$^{1}$, Konstantin Sobolev$^{1}$, Dmitry Ermilov$^{1}$, Igor Vorona$^{1}$, \\ Nikolay Kozyrskiy$^{1}$, Petr Tichavsk\'y$^{2}$ and Andrzej Cichocki$^{1}$} \\
  1 Skolkovo Institute of Science and Technology \\
  Moscow, Nobelya Ulitsa 3, 121207, Russia \\
  2 Academy of Sciences of the Czech Republic, Institute of Information Theory and Automation\\
  Prague, Pod vodarenskou vezi 4, 18200, Czech Republic
}

\begin{document} 

\maketitle

\begin{abstract}
A rising problem in the compression of Deep Neural Networks is how to reduce the number of parameters in convolutional kernels and the complexity of these layers by low-rank tensor approximation.
Canonical polyadic tensor decomposition (CPD) and Tucker tensor decomposition (TKD) are two solutions to this problem and provide promising results. However, CPD often fails due to degeneracy, making the networks unstable and hard to fine-tune. 
TKD does not provide much compression if the core tensor is big.
This motivates using a hybrid model of CPD and TKD, a decomposition with multiple Tucker models with small core tensor, known as block term decomposition (BTD). This paper proposes a more compact model that further compresses the BTD by enforcing core tensors in BTD identical. We establish a link between the BTD with shared parameters and a looped chain tensor network (TC). Unfortunately, such strongly constrained tensor networks (with loop) encounter severe numerical instability, as proved by \cite{Landsberg} and \cite{Handschuhth}.
We study perturbation of chain tensor networks, provide interpretation of instability in TC, demonstrate the problem. We propose novel methods to gain the stability of the decomposition results, keep the network robust and attain better approximation. Experimental results will confirm the superiority of the proposed methods in compression of well-known CNNs, and TC decomposition under challenging scenarios.


\end{abstract}

\section{Introduction}
 
%
Despite the outstanding efficiency of convolutional neural networks (CNNs), their practical application is hampered by computational complexity and high resources consumption. Based on the observation that the weights of convolutional networks contain redundant information, they can be compressed without large losses in network performance by structural pruning \cite{Guo_2021_GDP}, sparsification \cite{Singh2020sparse}, quantization \cite{Kryzhanovskiy_2021_QPP} and low-rank approximation \cite{Yin_2021_HTKD_RNN, Kim_2019_CVPR, PhanECCV2020}.
%
Prior works have explored a wide variety of methods to weight factorization \cite{Panagakis2021TMinDL}: singular value decomposition \cite{Kim_2019_CVPR}, Canonical Polyadic decomposition \cite{PhanECCV2020}, Tucker decomposition \cite{Yin_2021_HTKD_RNN} and Tensor Train decomposition \cite{Novikov2015,Yin_2021_TT}.

{\bf Canonical polyadic tensor decomposition (CPD)} was the first low-rank model applied to compress CNN \cite{Denton2014}.  
The CP-convolutional layer composes separable convolution kernel matrices. CPD often encounters degeneracy, 
the estimated model is sensitive to a slight change of the parameters; this makes the entire CNN unstable.

{\bf Tucker-2 decomposition (TKD2) \cite{tucker1963implications}.}
 An alternative method \cite{Kim2016} is to compress the input and output dimensions of the convolutional kernel, $\tY = \tB \times_1 \bA \times_3 \bC$ (see Figure \ref{fig:tensor2tkd}).
Compared to CPD, TKD is more stable, and the ranks of the decomposition can be determined using SVD or VBMF. However, in practice, the dimensions of input and output modes in TKD can be large and make the compression less efficient than CPD.

 {\bf Block-term decomposition (BTD) \cite{Lath-BCM12}} is a hybrid of CPD and TKD, constrains the core tensor $\tG$ to be sparse, block diagonal, and thereby BTD comprises a smaller number of parameters than TKD. More precisely, BTD models the data as sum of multiple Tucker terms, 
\be
\tY = \sum_{t = 1}^{T} \tB_t \times_1 \bA_t \times_3 \bC_t \label{eq_btd_shared}
\ee
where $\tB_t$ are order-3 core tensors of size $R_t \times I_2  \times S_t$, $\bA_t$ and $\bC_t$ are factor matrices of size $I_1 \times R_t$ and $I_3 \times S_t$, respectively, $t = 1, \ldots, T$.
So far, there are no available proper selection criteria of the block size (rank of BTD) and the number of terms.


\begin{figure}[t]
\centering
 \begin{minipage}[b]{.45\linewidth}
\centering
\subfigure[TKD]{\includegraphics[width=.95\linewidth]{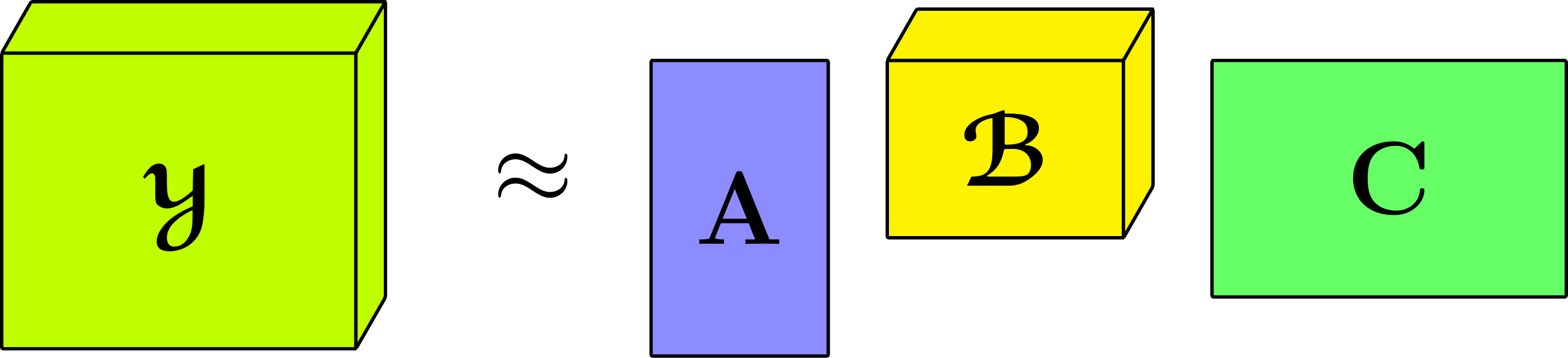}\label{fig:tensor2tkd}}
\\
\subfigure[TC]{
\hfill
\includegraphics[width=.6\linewidth, trim = 0.0cm 0.0cm 0.0cm 0cm,clip=true]{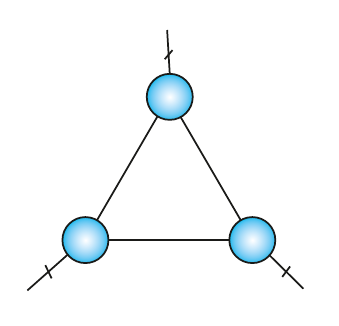}\label{fig:tensor2tc}}
%
\end{minipage}
\hfill
\subfigure[BTD]{\includegraphics[width=.45\linewidth,  trim = 0.0cm 0.0cm 0.0cm 0cm,clip=true]{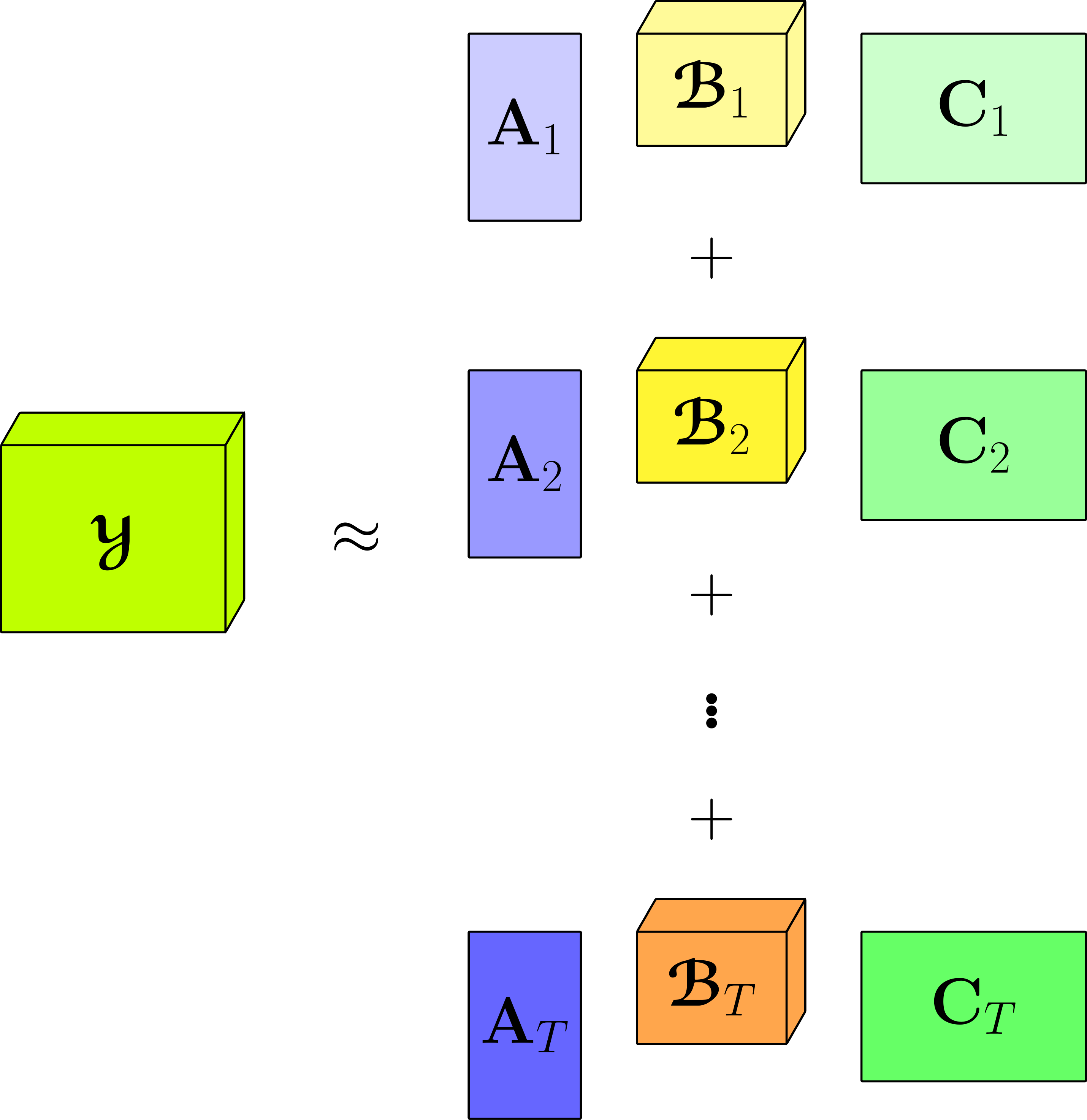}\label{fig:tensor2btd}}
\caption{Approximation of an order-3 tensor using  (a) TKD-2, or (c) BTD, comprising multiple TKDs (c). BTD with shared core tensors, $\tB_1 = \ldots = \tB_T$, forms the Tensor Chain(b).}\label{fig_models}
\end{figure}

{\bf BTD with shared core tensors}. In this paper, we propose a BTD with shared core tensors to reduce parameters in BTD, see illustration in Figure~\ref{fig:tensor2btd}. 
In particular case, all core tensors are identical, i.e., $\tB_1 = \ldots = \tB_T$.
We will show that this parameter-shared BTD is equivalent to a looped tensor network (Tensor chain - TC)\cite{Khoromskij-SC, Espig_2011}.
Using this connection, we propose a sensitivity correction procedure to overcome the problem with instability in this class of tensor models.

Such strongly constrained BTD or TC is not closed, i.e.,  the set of TC tensors of a fixed rank (bond dimension) is not Zariski closed. The openness of the set of fixed rank-r tensors implies that for some rank-$r$ tensors, one can approximate it with arbitrary precision by a tensor of a smaller rank $r_b$. For the canonical rank, the smallest rank $r_b$ is called the border rank of the tensor.
For TC, we refer to Section 3, ``Closedness of tensor network formats'' in \cite{Espig_2011}, 
and the work of \cite{LandsburgQY12}, which addresses the question of L. Grasedyck arising in quantum information theory, whether a tensor network containing a cycle (loop) is Zariski closed.  
One of the important conclusions is that {\emph{``if the tensor network graph $G$ is not a tree (it contains cycles), then the induced tensor network $U_G$ is in general not closed}''}, see \cite{LandsburgQY12},  also Remark 2.1.12. Ph.D. thesis, \cite{Handschuhth}.
 The looped TN leads to severe numerical instability problem in ﬁnding the best approximation, see Theorem 14.1.2.2\cite{Landsberg} and \cite{Handschuhth}.
 

{\bf Contributions.} The problem of instability in the TC model was identified early in \cite{Espig_2011, Handschuhth}.
 However, the problem is not well understood, and there is no method to deal with it. In this paper, we will study the sensitivity in TC and illustrate this type of degeneracy.
We propose novel methods to stabilize the estimated TC results and introduce a new TC layer or shared-parameters BTD convolutional layer. Finally, our primary aim is to propose a new convolutional layer with kernel in the form of TC or BTD with shared core tensors.
The proposed algorithms in this paper can be applied to tensor decomposition in other applications. 

    
    
  
We provide results of extensive experiments to confirm the efficiency of the proposed algorithms. Particularly, we empirically show that the neural network with weights in TC  format obtained using our algorithms is more stable during fine-tuning and recovers faster (close) to initial accuracy.

\section{Looped Tensor Network - Tensor Chain}

For simplicity, we first present TC for order-3 data. Extension of TC to higher-order tensor can be made straightforwardly.
\cite{Khoromskij-SC, Espig_2011} introduced the looped Tensor Chain as an extension of the Tensor Train (TT)\cite{oseledets2010tt}. Since there are no first and last core tensors, TC is expected to overcome the imbalance rank issue in TT decomposition. Tensor Ring is the same tensor network model  
inspired by Tensor Chain. See
illustration in Figure~\ref{fig:tensor2tc}.  We use the name Tensor Chain to honor the original authors who have invented it. 
The TC for an order-3 tensor $\tY$ of size $I_1 \times I_2 \times I_3$ reads 
\be
\tY = \sum_{r_1 = 1}^{R_1} \, \sum_{r_2 = 1}^{R_2}  \sum_{r_{3} = 1}^{R_{3}} \tA(r_1,:,r_2) \circ \tB(r_2,:,r_3) \circ \tC(r_{3},:,r_1), \label{eq_tc3}
\ee
where `$\circ$' represents the outer product, $\tA$, $\tB$ and $\tC$ are core tensors of size $R_1 \times I_1  \times R_2$, $R_2 \times I_2  \times R_3$, $R_3 \times I_3  \times R_1$, respectively, and its element can write as $y_{i j k} = \tr({\tA(:,i,:) \tB(:,j,:) \tC(:,k,:)})$.
We use shorthand notation for TC as
$\tY = \circlellbracket \tA, \tB, \tC\circlerrbracket$.
The links between TC and BTD are revealed in the following Lemma.


\begin{lemma}[Equivalence of BTD with shared core tensors and TC]\label{lem::btd_tc}
The constrained BTD with shared core tensors in (\ref{eq_btd_shared}) is a TC model
{\normalfont
$
\tY =  \circlellbracket \tA, \tB, \tC\circlerrbracket$}
where $\tA$ is of size $R_1 \times I_1 \times R_2$ with slices $\tA(t,:,:) = \bA_{t}$, $\tC$ of size $R_3 \times I_3 \times R_1$ and 
$\tC(:,:,t) = \bC_{t}$, $t = 1, \ldots, R_1$. (Proof is provided in Appendix~\ref{secA::proof_btd_tc}.)
\end{lemma}

The TC for a higher order tensor, $\tY$, of size $I_1 \times I_2 \times \cdots \times I_N$ can be generalized as  
\be
\tY = \sum_{r_1 = 1}^{R_1} \cdots \sum_{r_{N} = 1}^{R_{N}} \tA_1(r_1,:,r_2) \circ \tA_2(r_2,:,r_3) \circ \cdots \circ \tA_N(r_{N},:,r_1), \notag
\ee
where $\tA_1$, \ldots, $\tA_N$ are core tensors of size $R_n \times I_n \times R_{n+1}$, $R_{N+1} = R_1$. 
We can regard TC of order-$N$ as a nested TC of order-3 whose the core tensor $\tB$ in (\ref{eq_tc3}) is a Tensor-Train of ($N-2$) core tensors $\tA_2$, \ldots, $\tA_{N-1}$, i.e.,  
$\tB = \tA_{2} \ttprod \tA_{3} \ttprod \cdots \ttprod \tA_{N-1}$, 
where the train-contraction ``$\ttprod$'' is defined in Appendix \ref{def_boxtime}.

\subsection{Algorithms for TC decomposition}\label{sec::algs}

\cite{Espig_2011} proposed nonlinear block Gauss-Seidel algorithms including alternating least squares (ALS), density-matrix renormalization group (DMRG), and adaptive cross approximation for contracted tensor networks, in which TC is a special case. 
Thanks to links between TC and BTD, TC and structured TKD, we can also use algorithms for BTD, e.g., the nonlinear least squares (NLS) \cite{SDF}, Krylov-Levenberg-Marquardt (KLM) \cite{PetrKLM} algorithms or the OPT algorithm based on   Limited-memory BFGS method\cite{TRWOPT}. 

Like other tensor decompositions, the ALS\cite{Espig_2011,Espig2012,Handschuhth} is still considered the best algorithm for TC, especially when DMRG cannot be applied. Various variants of these two update schemes were proposed e.g., for the tensor completion problem \cite{zhao2016tensor,wang2017efficient,asif2020low,mickelin2018algorithms,
he2019remote, huang2020robust,
ahad2020hierarchical, 
yu2020low,  ding2020tensor}, hyperspectral super-resolution \cite{xu2020hyperspectral}. For determination of bond dimensions,we refer to  \cite{cheng2020novel,Farnaz2021}. 
Nevertheless, all existing TC algorithms are not robust to perturbation of parameters. 
For the first time, we propose
the algorithm that provides optimal TC decomposition with low(est) sensitivity and considerably alleviates the problem of stacking in local minima of algorithms for TC.




\subsection{Major problem: Instability}

 The TC was inspired by two successful models, TT and BTD, to overcome the high intermediate ranks in TT and enforce a more compact model for BTD. 
 As mentioned earlier, \cite{Landsberg} and \cite{Espig_2011} pointed out the problem with TC. 
  We show that any TC model can be unstable with very high intensity and sensitivity. First, we define TC intensity. 
 
\begin{remark} The TC model, 
\normalfont{$\tY = \circlellbracket \tA, \tB, \tC\circlerrbracket$}, is not unique up to scaling 
{\normalfont
\be
\tY =  \, \circlellbracket \alpha_1 \tA, \alpha_2\tB, \alpha_3\tC\circlerrbracket
\ee
}
with arbitrary factors $\alpha_1$, $\alpha_2$, and $\alpha_3$ such that $\alpha_1 \alpha_2 \alpha_3 = 1$.
 \end{remark}

\begin{remark} The TC model, 
\normalfont{$\tY = \circlellbracket \tA, \tB, \tC\circlerrbracket$}, is also non-unique up to rotation 
{\normalfont
\be
\tY =  \, \circlellbracket \tA \ttprod \bQ,  \bQ^{-1} \ttprod \tB, \tC\circlerrbracket
\ee}
where $\bQ$ is an arbitrary invertible matrix of size $R_2 \times R_2$.
 \end{remark}
 
\begin{definition}[TC intensity] 
For a given TC model, \normalfont{$\tY = \circlellbracket \tA, \tB, \tC\circlerrbracket$}, we can always normalize core tensors to unit norm, $\tilde{\tA} = {\tA}/\|\tA\|_F$, $\tilde{\tB} = {\tB}/\|\tB\|_F$,$\tilde{\tC} = {\tC}/\|\tC\|_F$
then $\tY = \alpha \, \circlellbracket \tilde{\tA}, \tilde{\tB}, \tilde{\tC}\circlerrbracket$
where $\alpha = \|\tA\|_F \|\tB\|_F \|\tC\|_F$ is called the TC intensity.
\end{definition}

\begin{lemma}[TC Degeneracy] \label{lem_tc_degenerarcy}
For a given TC model, \normalfont{$\tY = \circlellbracket \tA, \tB, \tC\circlerrbracket$}, there is always a sequence of equivalent TC models with diverging TC intensities. (Proof is provided in Appendix~\ref{sec::proofTCdegeneracy}.) 
\end{lemma}

\begin{remark}[TC instability]
The first observation is that TC models estimated by any iterative algorithms can encounter large TC-intensity. In many cases, the TC-intensity increases quickly with the iterations. Without proper processing, the algorithm gets stuck in a false local minimum. The decomposition is more challenging, especially when the dimension of a core tensor is smaller than its ranks, e.g., $R_1 R_2 > I_1$, or when components of the core tensors are highly collinear, or decomposition with missing entries.
Such a type of degeneracy in TC happens quite often and is similar to that in CPD. 
\end{remark}






\begin{figure}[t]
\subfigure[]{\includegraphics[width=.50\linewidth, trim = 0.0cm 0cm 0.0cm 0cm,clip=true]{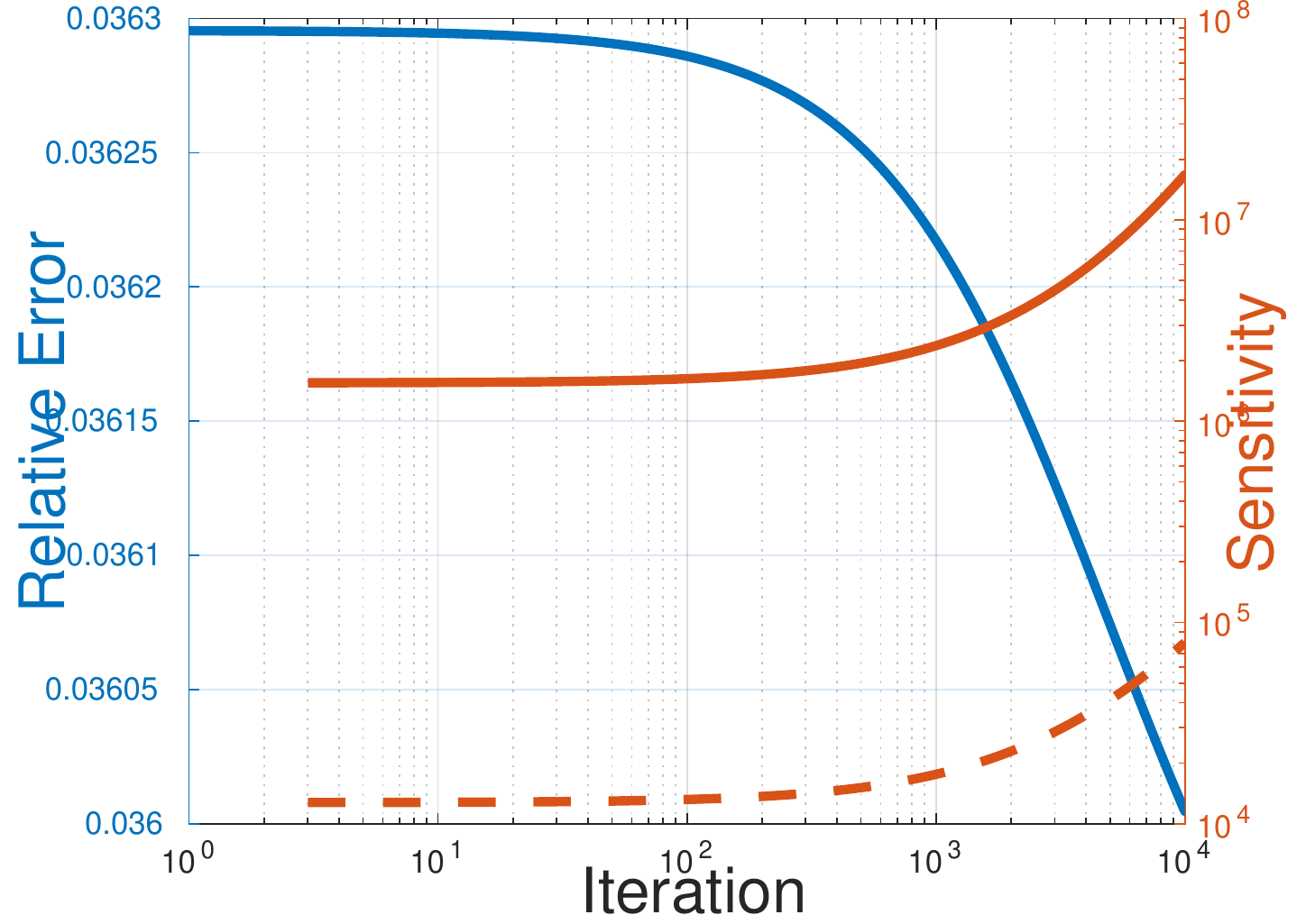}\label{fig_ex777_relerr}}
\hfill
\subfigure[]
{\includegraphics[width=.45\linewidth, trim = 0.0cm 0cm 0cm 0cm,clip=true]{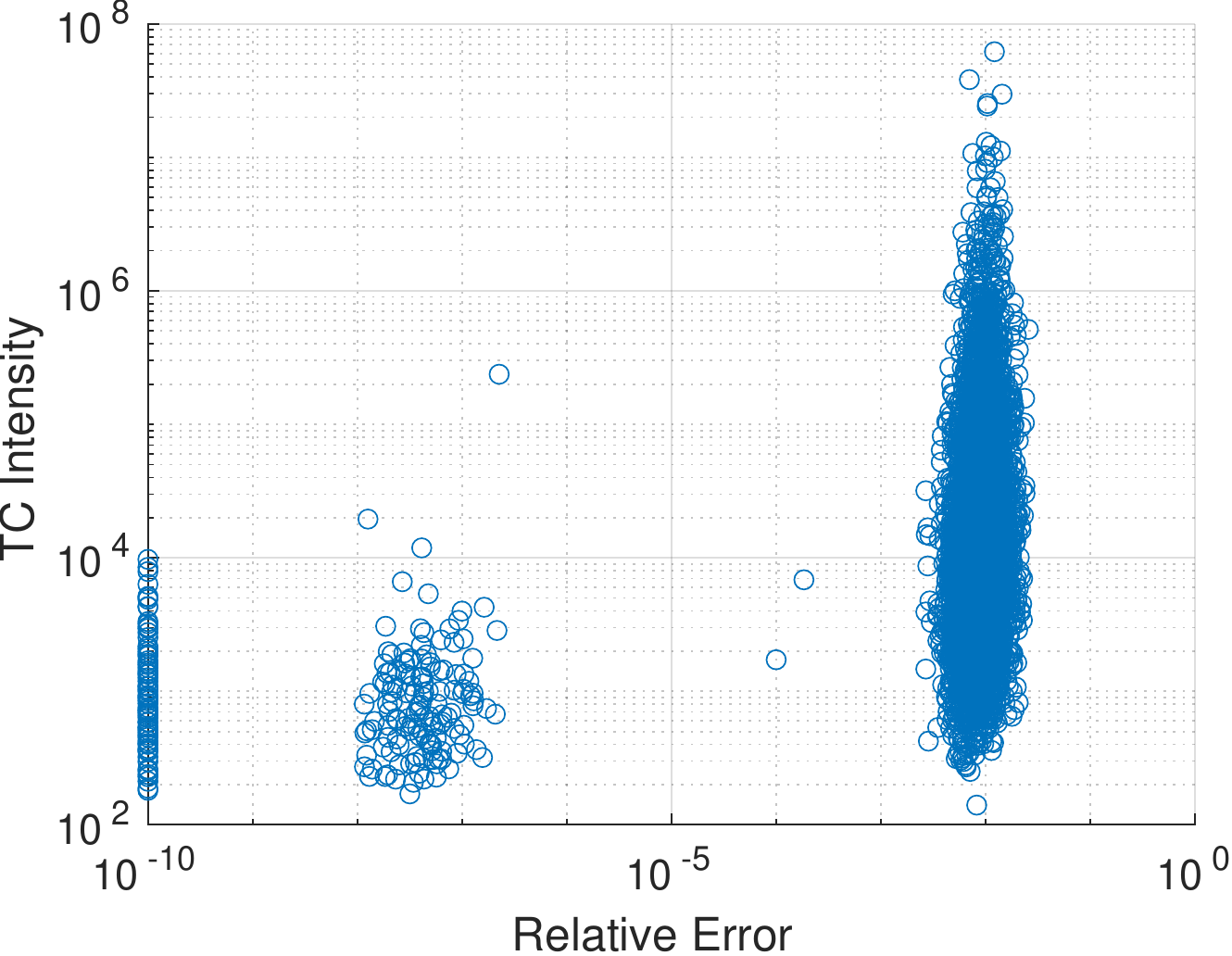}\label{fig_tc_uni_I7R3_als_normcore}}
\caption{TC decomposition of tensors of size $7\times 7 \times 7$ in Example~\ref{ex_7x7x7}. 
(a) Relative error of the ALS in one TC decomposition. The red solid and dashed curves show the TC sensitivity and intensity, increasing drastically during the estimation. (b) The Scatter plot shows coherence between high intensity and high approximation error (bad decomposition).}
\label{fig_successrate_7x7x7}
\end{figure}

\begin{example}[{TC with bond dimensions exceeding tensor dimensions}]\label{ex_7x7x7} We provide an illustrative example for TC degeneracy in the decomposition of noise-free synthetic tensors of size $7\times 7 \times  7$ with bond dimension $(3-3-3)$, composed from 3 core tensors randomly generated. 
The decomposition using the ALS algorithm in 5000 iterations succeeds in less than 3\% in 10000 independent TC decompositions, see Figure~\ref{fig_ex777_relerr}.

{\bf Why ALS and many other algorithms for TC fail?}
 The TC intensity of the estimated tensor quickly increases after several thousand iterations, as seen in Figure~\ref{fig_ex777_relerr} for illustration of relative errors in one run. The TC intensity exceeds $10^7$ after 10000 iterations, making the algorithm converge to local minima with a relative approximation error of 0.036. 
 Scatter plot of intensity and relative approximation errors over 10000 TC decompositions in Figure~\ref{fig_tc_uni_I7R3_als_normcore} indicates coherence between bad decomposition results with high intensity.  
 The above example shows a difficult case when core tensors, $\tU_n$, have fat factor matrices, $\bU_n$, $R_{n-1} R_{n} = 9 > 7$.
 \end{example}

\begin{example}[{TC with highly collinear loading components}]\label{ex_27x27x27_rank25}
 We decompose order-3 tensors of size $27 \times 27 \times 27$ which admit the TC model, $\tY =  \, \circlellbracket \tU_1, \tU_2, \tU_3\circlerrbracket$, with bond dimensions $(5-5-5)$. The factor matrices ${\bU}_n$ of size $27 \times 25$ have highly collinear loading components, $0.97 \leq {\bU}^{T}_n(:,r) {\bU}_n(:,s)\leq 0.99$, $n=1, 2, 3$.
 The relative approximation errors of the ALS shown in Figure~\ref{fig_ex272727}(a) indicate that ALS failed in this example. The intensity (dashed red curve) of the estimated TC tensors increased quickly and exceeded $2.7 \times 10^6$, whereas its sensitivity (dotted blue curve) passed the level of $10^6$ after 6000 iterations as shown in Figure~\ref{fig_ex272727}(b). 
 The OPT(WOPT) \cite{TRWOPT} and NLS algorithms \cite{Sorber-tensorlab} also failed. 
\end{example}

\begin{figure}[t]
\centering
 \subfigure[]
 {\includegraphics[width=.47\linewidth, trim = 3.5cm 8.5cm 4cm 8.9cm,clip=true]{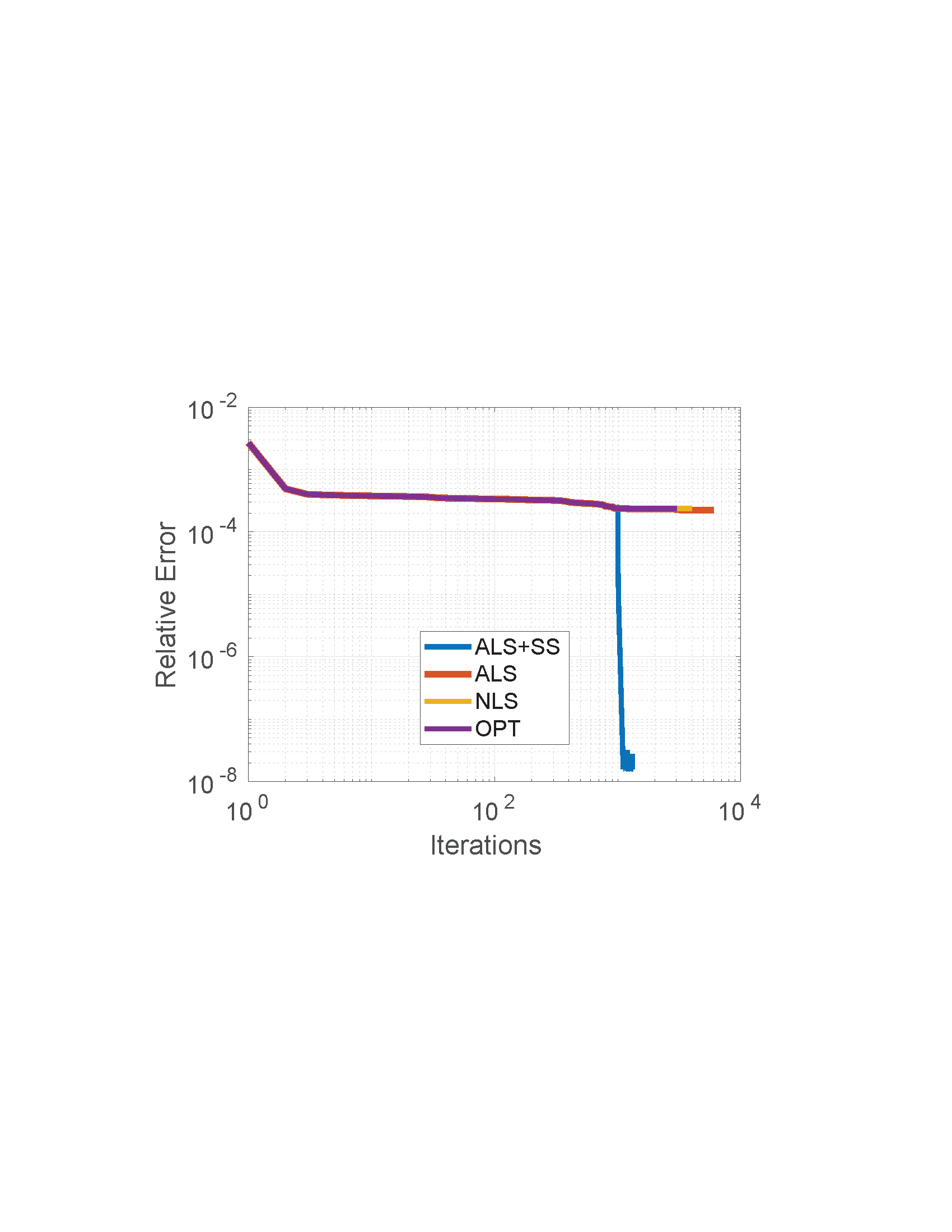}
 \label{fig_tc_collinear_N3_R25_err}} 
 \subfigure[]
 {\includegraphics[width=.49\linewidth, trim = 0.0cm 0cm 0cm 0cm,clip=true]{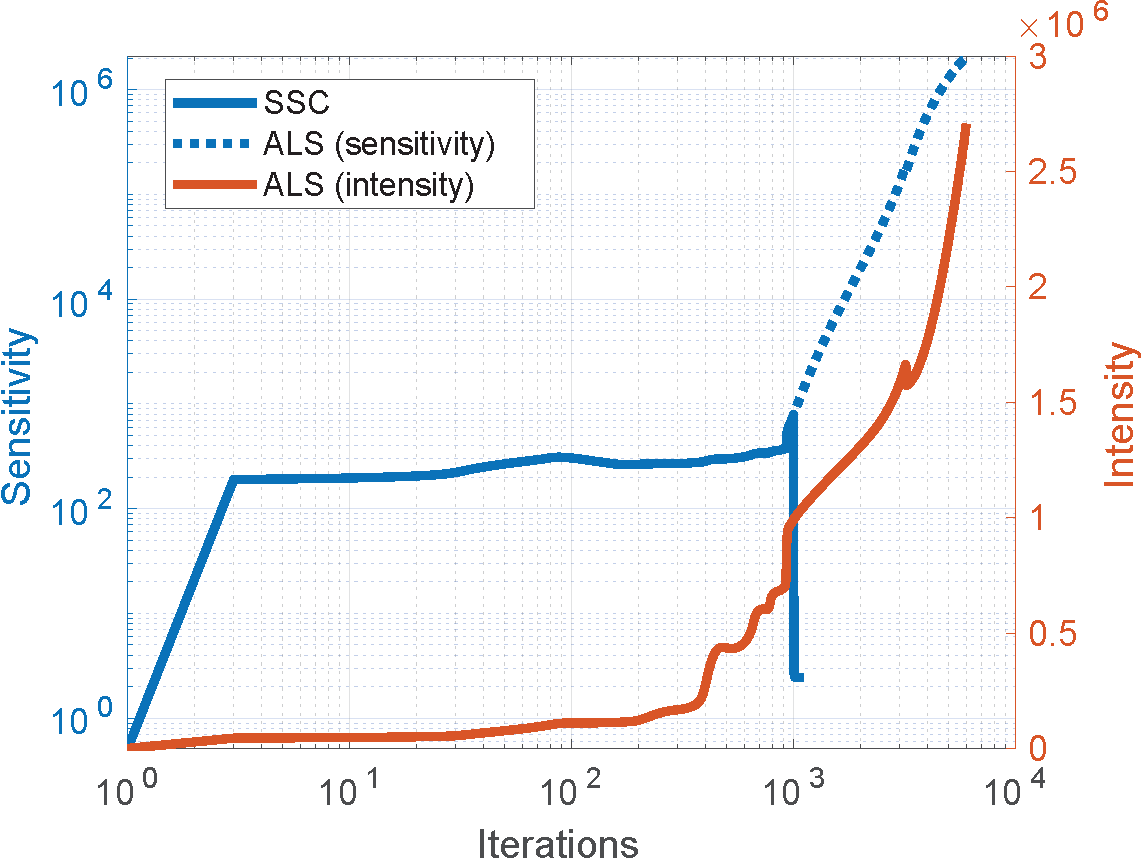}
 \label{fig_tc_collinear_N3_R25_ss}}
 \caption{Illustration for performances obtained by algorithms used in Example~\ref{ex_27x27x27_rank25}. ALS+SS is the ALS with sensitivity correction. Without the SS correction, the estimated tensors have very high sensitivity and TC intensity.}\label{fig_ex272727}
\end{figure}

\begin{example}[TC for incomplete data]
\label{ex_incompletTC}
We demonstrate a simple TC decomposition for tensors of size $9 \times 9 \times 9$ with bond dimensions $(3-3-3)$. 
The considered tensors can be factorized quickly. However, when 50\%  of the tensor elements are randomly removed, the tensors are challenging to any TC algorithms. The success rate for OPT\cite{TRWOPT} and ALS is less than 11\%. Figure~\ref{fig_incompleteTC} illustrates the convergence of OPT in one TC decomposition and scatter plot of the sensitivity and relative approximation errors. The algorithms get stuck in local minimal and cannot attain exact decomposition.
\end{example}

\begin{figure}[t]
\centering
 {\includegraphics[width=.48\linewidth, trim = 0.0cm 0cm 0cm 0cm,clip=true]{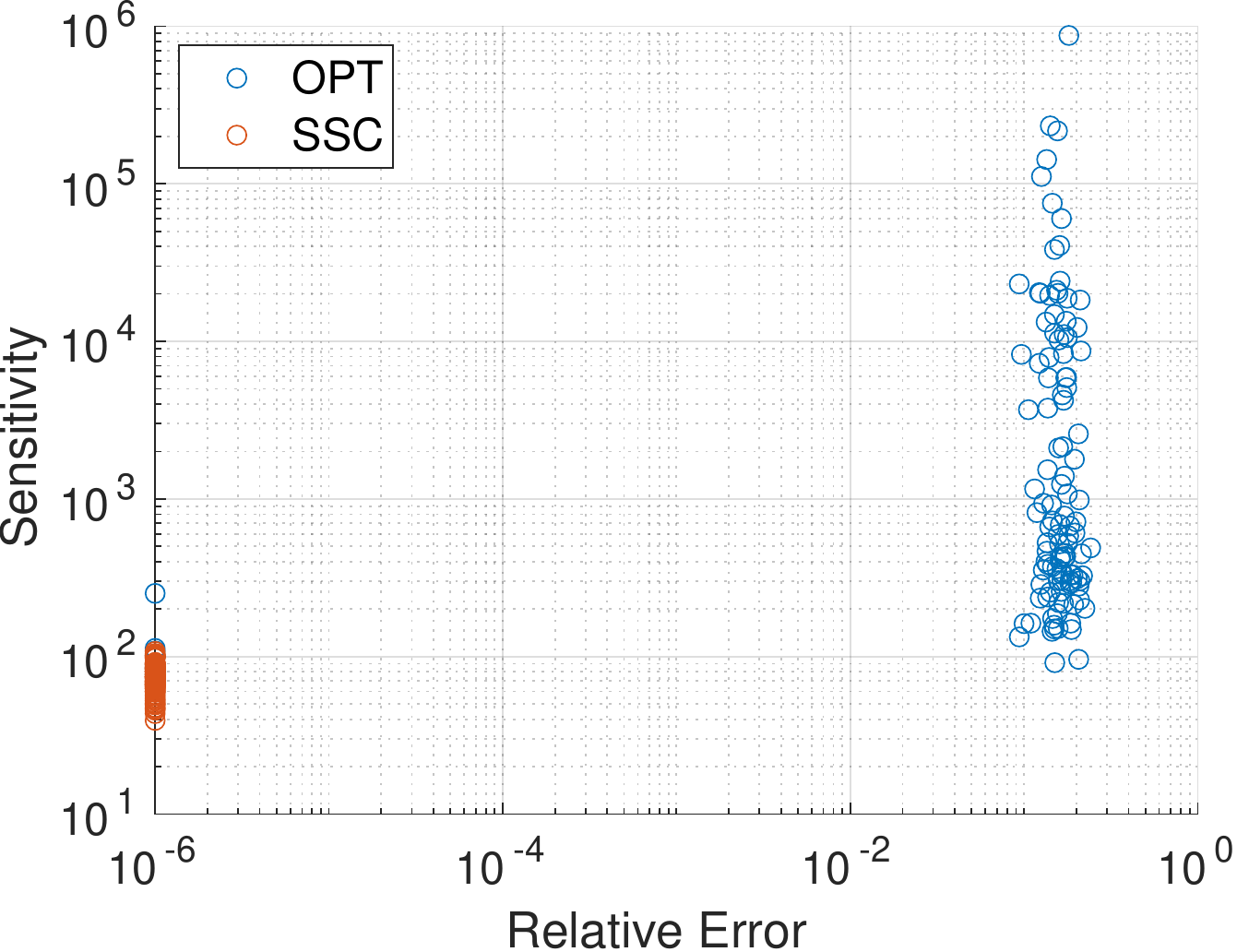}} 
 {\includegraphics[width=.48\linewidth, trim = 0cm 0cm 0cm 0cm,clip=true]{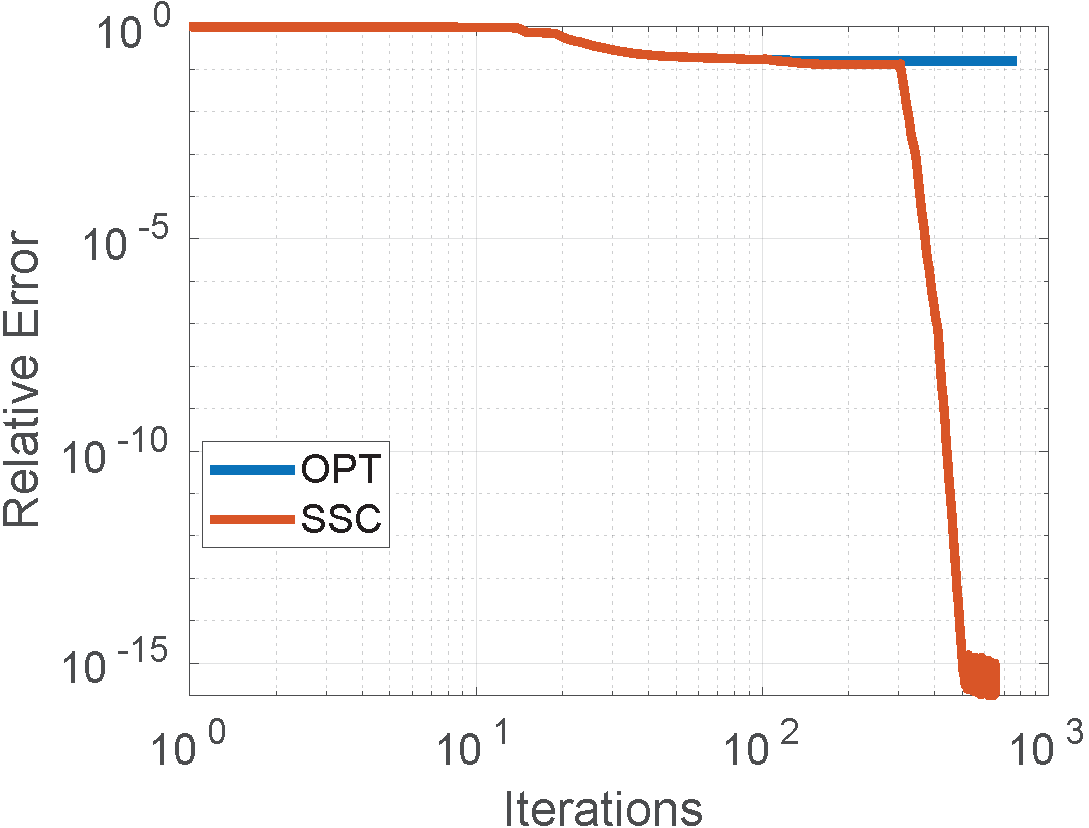}}
 \caption{Illustration for performances obtained by OPT in Example~\ref{ex_incompletTC} for incomplete tensor decomposition. With sensitivity correction, the decomposition can obtain the exact model. }\label{fig_incompleteTC}
\end{figure}

In Appendix~\ref{sec::synexample}, we provide more examples for which the TC algorithms fail to decompose higher-order tensors. Besides high TC-intensity, the sensitivity of the estimated TC model significantly increased, and it prevented the algorithm from converging to the exact model. 

\section{Sensitivity for TC}
 
 We earlier show that any TC model can be unstable since its TC-intensity can grow to infinity. This section introduces the sensitivity of the TC model and its properties. 
 
\begin{definition}[\bf Sensitivity (SS)] Given a TC model {\normalfont{$\tY = \circlellbracket \tA_1, \dots, \tA_N \circlerrbracket$}}.
Denote by $\delta_{A_1}, \dots, \delta_{A_N}$ random Gaussian distributed perturbations with element distributed independently with zero mean and variance $\sigma^2$.
Sensitivity of the TC model {\normalfont{$\circlellbracket \tA_1, \dots,\tA_N \circlerrbracket$}} is defined as 
\be
ss(\tA_1, \dots,\tA_N) =  \lim_{\sigma^2 \rightarrow 0}\frac{1}{\sigma^2} E\left\{\|\tY - \tY_{\delta}\|_F^2 \right\} \label{eq_sstc} ,
\ee
where $\tY_{\delta} = \circlellbracket \tA_1 +\delta_{A_1}, \ldots,\tA_N + \delta_{A_N} \circlerrbracket$. 
\end{definition}
 %
 {{The above sensitivity measure is standard and widely used in the analysis of perturbation of a model or function to the weights e.g., in \cite{ChoiC92}, \cite{Yeung2010}
(Section 2.3.2), \cite{XiangZWLY21}. In principle, it measures the mean of the total error variance. 

The sensitivity of the function can also be computed as the average
Frobenius norm of the Jacobian (or Jacobian norm) as in \cite{Novak2018}, \cite{Pizarroso2020}. The latter definition is approximate of the
Frobenius norm of the output difference. Both would lead to a similar
compact formula for computing the sensitivity.

We will show that the final expression of the sensitivity is relatively simple and
can be written in quadratic forms for each core tensor. It 
allows us
to formulate sub-optimization problems for updating core tensors as
constrained quadratic programming that can be solved in closed-form.
}}

\begin{lemma}\label{lem_ss}
{\bf Sensitivity (SS)} of a TC model, {\normalfont{$\circlellbracket \tA_1, \dots,\tA_N \circlerrbracket$}}, is computed as 
{\normalfont
\be
 ss(\tA_1,\dots,\tA_N) = \sum_{n=1}^N I_n \, \| \tA_{-n}  \|_F^2.
 \label{eq_sstc2}
\ee}  
where $\tA_{-n} = \tA_{n+1} \ttprod \cdots  \ttprod \tA_N \ttprod \tA_1 \ttprod \cdots \ttprod \tA_{n-1}$.
Proof is provided in Appendix~\ref{sec::ssc_highorder}. 
\end{lemma}

In principle, TCs with high sensitivity are less stable than those with a smaller SS. A simple normalization that scales the core tensors can reduce the SS. 

\begin{lemma}[Balanced norm for minimal sensitivity]
\label{lem_balancednorm_ss}
A TC model {\normalfont{$\circlellbracket  \tA_1, \dots, \tA_N \circlerrbracket$}} can be scaled to give a new equivalent model {\normalfont{$\circlellbracket  \tA_1, \dots, \tA_N \circlerrbracket = \circlellbracket\alpha_1 \tA_1, \dots, \alpha_N \tA_N \circlerrbracket$}}
with the minimal sensitivity, where $\alpha_n = \frac{\beta_n}{\beta}$, 
{\normalfont{$\beta_n = \sqrt{I_n} \, \|  \tA_{-n} \|_F$}} and 
$\beta = \sqrt[N]{\prod_{n=1}^N \beta_n}$.
\end{lemma}
Next section presents more efficient algorithms for sensitivity correction.  
\begin{remark} TC intensity is an upper bound of SS of the model. 
From SS in (\ref{eq_sstc}), we have 
\be
ss(\btheta) \le \sum_{n=1}^N {I_n} \, \prod_{k=1, k \neq n}^N \|\tA_k\|_F^2 .\label{eq_ss_intensity}
\ee
where $\btheta$ is the vector of parameters.
\end{remark}

\section{How to deal with Instability in TC- Sensitivity Correction Method}

 TC's instability in TC is similar to degeneracy in Canonical Polyadic  Decomposition (CPD), which is hard to avoid.
 We propose to correct the unstable estimated model by seeking a new tensor, $\hat{\tY}$, which preserves the approximation error but has smaller sensitivity.

\subsection{Rotation method}\label{sec::rotation}
A simple method is that we scale core tensors following Lemma~\ref{lem_balancednorm_ss}. 
An alternative method is that we rotate two consecutive core tensors by invertible matrices such that the new TC representation has minimum sensitivity 
\be
\tY = \circlellbracket \bQ_3^{-1} \, \ttprod \tA \ttprod \bQ_1, \bQ_1^{-1} \ttprod \tB \ttprod \bQ_2, \bQ_2^{-1} \ttprod \tC \ttprod \bQ_3  \circlerrbracket \,.
\ee
For simplicity, we derive the algorithm to find the optimal matrix, $\bQ$ of size $R_2 \times R_2$ which rotates the first two core tensors, $\tA$ and $\tB$, and gives a new equivalent TC tensor $\tY_{\bQ} = \circlellbracket \tA \ttprod \bQ, \bQ^{-1} \ttprod \tB, \tC  \circlerrbracket$. 

The optimal matrix $\bQ$ which minimizes the sensitivity of $\tY_{\bQ}$ is found in the following optimization
\begin{align}
\min_{\bQ} \;  ss(\tY_{\bQ}) = & I_3 \|\tA \ttprod \tB\|_F^2  + I_1 \|\bQ^{-1} \ttprod \tB \ttprod \tC\|_F^2 \notag \\
& + I_2 \|\bC \ttprod \tA \ttprod \bQ\|_F^2 \,.\label{eq_rotss}
\end{align}
We represent the product $\bQ\bQ^T = \bU \bS \bU^T$ in form of EVD, where $\bU$ is an orthogonal matrix of size $R_2 \times R_2$ and $\bS = \diag(s_1, \ldots, s_{R_2})$.
Instead of seeking $\bQ$, we find an orthogonal matrix $\bU$ and eigenvalues $s_r$.

The optimal eigenvalues are given in closed form as $\displaystyle s_r^{\star} = \sqrt{ \frac{I_1 \bu_r^T \bT_1 \bu_r}{I_2 \bu_r^T \bT_2 \bu_r}}$, for $r = 1, \ldots, R_2$, 
where 
$\bT_1 = \sum_{i_2 = 1}^{I_2}  \tB(:,i_2,:) (\bC_{(1)} \bC_{(1)}^T) \tB(:,i_2,:)^T$, and 
$\bT_2 = \sum_{i_1 = 1}^{I_1}  \tA(:,i_1,:)^T (\bC_{(3)} \bC_{(3)}^T) \tA(:,i_1,:)$, 
$\bC_{(1)}$ and $\bC_{(3)}$ are mode-1 and -$3$ unfoldings of the core tensor $\tC$.
The optimization problem to find the matrix $\bU$ is simplified to 
\be
\min_{\bU \in  St_{R_2}} \quad \sum_{r = 1}^{R_2} \sqrt{(\bu_{r}^T \bT_1 \bu_{r}) (\bu_{r}^T \bT_2 \bu_{r})} \,
\label{eq_obj_U}
\ee
and can be solved using the conjugate gradient algorithm on the Stiefel manifold \cite{Zaiwen_2012}. 
Algorithm~\ref{alg_rotation} in Appendix summarizes pseudo-codes of the proposed method, which first rotates $\tA$ and $\tB$, then performs cyclic-shift of dimensions in the tensor $\tY$ to give $\circlellbracket \tA ,  \tB, \tC  \circlerrbracket = \circlellbracket \tB ,  \tC, \tA  \circlerrbracket$.
A complete derivation of the rotation method for {\emph{higher order tensors}} is presented in Appendix~\ref{sec::rotation_full}.


\subsection{Alternating Sensitivity Correction Method}\label{sec::assc}

Both scaling and rotation methods preserve the approximation, transform a TC tensor with high SS to a new equivalent one with a smaller SS. 
This section proposes another algorithm for sensitivity correction, which updates one core tensor in each iteration.
The new algorithm further suppresses the sensitivity to a much lower value.
Similar to the rotation method, we formulate the problem of sensitivity correction as minimization of sensitivity with a bound constraint which for order-3 tensor is given by
\begin{align}
    \min  \quad & ss(\btheta) = I_1 \|\tB \ttprod \tC\|_F^2 + I_2 \|\tC \ttprod \tA\|_F^2 +  I_3 \|\tA \ttprod \tB\|_F^2 \,  \label{eq_ssmin}\\
\text{s.t.}\quad &  c(\btheta) =  \|\tY - \hat{\tY}\|_F^2 \le \delta^2 \notag ,
\end{align}
where 
$\hat{\tY} =  \circlellbracket \tA, \tB, \tC  \circlerrbracket$.
$\delta$ can be the approximation error of the current TC model, i.e., $\delta = \|\tY - \hat{\tY}_0\|_F$.

The objective and constraint functions are nonlinear in all 
core tensors. In order to solve (\ref{eq_ssmin}), we rewrite the objective function and the constraint function for a single core tensor and solve it using the alternating update scheme. For example, the optimization problem to update the core tensor $\tA$ is given by
\be
\min_{\tA} \quad && ss(\btheta) = I_1 \|\tB \ttprod \tC\|_F^2 +  \tr(\bQ \bA_{(2)}^T  \bA_{(2)})    \label{eq_ssA}\\ 
\text{s.t.}\quad &&\|\bY_{(1)} - \bA_{(2)}  \bZ^T \|_F^2 \le \delta^2
\notag 
\ee
where $\bQ = I_2 (\bI_{R_2} \otimes \bC_{(3)}\bC_{(3)}^T) + I_3 (\bB_{(1)} \bB_{(1)}^T \otimes \bI_{R_1})$, $\bZ$ is mode-(1,4) unfolding of the sub-network $\tB \ttprod \tC$,
$\bA_{(2)}$ is mode-2 unfolding of the core tensor $\tA$, 
$\bY_{(1)}$ is mode-1 unfolding of the tensor $\tY$.
The above optimization problem is quadratic.  
 $\tA$ can be found in closed form as in Spherical Constrained Quadratic Programming (SCQP) \cite{GANDER1989815, Phan_EPC, Phan_QPS}.

 We summarize pseudo code of the alternating SS correction (SSC) in Algorithm~\ref{alg_ACEP}. After each update, we perform cyclic shift to the tensor $\tY$ and $\hat{\tY}$. The order of the TC tensor will be $\circlellbracket  \tB, \tC, \tA\circlerrbracket$. 
 The algorithm for higher-order tensor is presented in Appendix~\ref{sec::ssc_highorder}. A similar algorithm can be applied to correct the TC intensity, where $\bQ$ in (\ref{eq_ssA}) is an identity matrix. We often correct intensity before correction of sensitivity.
 
 The entire procedure for efficient TC decomposition with SS control is listed in Algorithm~\ref{alg_ALSSSC}. One can start the decomposition with any algorithm in Section~\ref{sec::algs}.
 When SS of the estimated tensor is high exceeds a predefined value, e.g., $10^7, 10^8$, the decomposition will converge slowly, and the model tends to be unstable. Algorithm~\ref{alg_ALSSSC} will execute the sensitivity correction in Algorithm~\ref{alg_ACEP}. The TC decomposition will resume from a new tensor after SSC.


\begin{algorithm}[t] 
\caption{{\tt{Sensitivity Correction (SSC)}}\label{alg_ACEP}}
\DontPrintSemicolon \SetFillComment \SetSideCommentRight
\KwIn{Tensor $\tY$:  $(I_1 \times I_2 \times I_3)$,  and bond dimensions $R$ and error bound $\delta$}
\KwOut{$\hat{\tY} = \circlellbracket  \tA, \tB, \tC\circlerrbracket$:\;
$\min ss(\hat{\tY})$  \; s.t. \;  $\|\tY - \hat{\tY} \|_F^2  \le \delta^2 $}
\rememberlines
\Begin{
 Initialize $\hat{\tY}$ \; 

\Repeat{a stopping criterion is met}{
Apply balanced normalization,
and Rotation method in Algorithm~\ref{alg_rotation}\;

\For{$n = 1, 2, 3$}
	{
	    $\tZ = \tB \ttprod \tC$\;

	    $\bQ = I_2 (\bI_{R_2} \otimes \tC_{(3)}\tC_{(3)}^T) + I_3 (\tB_{(1)} \tB_{(1)}^T \otimes \bI_{R_1})$\;
	    
	    Solve $\tA = \arg\min_{\bX}   \tr(\bX \bQ \bX^T)
		\quad  \textrm{s.t.} \;\|\bY_{(1)} - \bX \bZ_{(1,4)} \|_F^2 \le \delta^2$\;

		Cyclic-shift of dimensions in $\tY$ and $\hat{\tY}$\;
	}
}
}
\end{algorithm}

%
\begin{algorithm}[t]
\caption{{\tt{TC with SS control (SSCTrl)}}\label{alg_ALSSSC}}
\DontPrintSemicolon \SetFillComment \SetSideCommentRight
\KwIn{tensor $\tY$:  $(I_1 \times I_2 \times I_3)$,  and bond dimensions $R$}
\KwOut{$\hat{\tY} = \circlellbracket  \tG_1, \tG_2, \tG_3\circlerrbracket$}
\resumenumbering
\Begin{
\Repeat{a stopping criterion is met}
	{
	    Perform TC decomposition $\hat{\tY}$: $\min \;  \|\tY - \hat{\tY}\|_F^2$\;
	    
	    \If {$ss(\btheta) \ge ss_{max}$}{
	         Apply Alg.~\ref{alg_ACEP} with $\delta = \|\tY - \hat{\tY}\|_F^2$\;
	        }
	}
}
\end{algorithm}



  

\section{TC convolutional layer}\label{sec::tclayer}



%
We apply the proposed algorithms for sensitivity correction in the application for CNN compression. 
In \cite{wide_compression}, the authors replace the convolutional kernel and fully connected kernels with the TC model and train the model from scratch. In the case of the pre-trained network given, Aggarwal et al. perform the decomposition of the kernels with parameters randomly generated from a Gaussian distribution.
We propose a sequence of 5 layers to replace a convolutional layer. 
Our method for CNN compression includes the following main steps:
\begin{enumerate}[noitemsep]
    \item  Each convolutional kernel is approximated by a tensor decomposition (TC or CPD ).
    \item  The TC decompositions with diverging components are corrected. The result is a new TC model with minimal sensitivity. CP is also corrected to have minimal sensitivity\cite{Phan_EPC}.
    \item  An initial convolutional kernel is replaced with a tensor in TC or CPD format, which is equivalent to replacing one convolutional layer with a sequence of convolutional layers with a smaller total number of parameters.
    \item  The entire network is then fine-tuned. 
\end{enumerate}

\textbf{TC Block} results in three convolutional layers $W_1, W_2, W_3$ with shapes ($C_{in} \times R_1R_2 \times 1 \times1$),~~3D ($R_2 \times R_3 \times 1 \times D \times D$) and ($R_3R_1 \times C_{out} \times 1 \times1$), respectively, and two permute/reshape layers $T_1, T_2$ with transform (from $R_1R_2 \times H \times W$ to $R_2 \times R_1 \times H \times W$) and (from $R_3 \times R_1 \times H \times W$ to $ R_1R_3 \times H \times W$), respectively, where $H$ and $W$ are the input dimensions, $R_1, R_2, R_3$ are TC ranks and $D$ is kernel dimension. Layer order is following: $W_1, T_1, W_2, T_2, W_3$. See Figure~\ref{fig:tc_layer_structure}  in Appendix~\ref{sec::tclayer_2}).
 
\textbf{CPD Block} results in three convolutional layers with shapes ($C_{in} \times R \times 1 \times1$),
depthwise ($R \times R \times D \times D$) and ($R \times C_{out} \times 1 \times1$), respectively. Here $R$ is CP rank and $D$ is kernel dimension. (See Figure~\ref{fig:cp_layer_structure} in Appendix~\ref{sec::tclayer_2}).

\textbf{Rank Search Procedure.} For CP, the smallest rank is chosen such that drop after single layer fine-tuning does not exceed a predefined threshold EPS. TC ranks are selected over the grid with a constraint that the model compressed with TC has fewer FLOPs than the corresponding model compressed with CP.
\section{Experiments}


\paragraph{Datasets and Computational Resources}
We test our algorithms on two representative CNN architectures for image classification: \textit{VGG-16} \cite{SimonyanZ14a}, \textit{ResNet-18}\cite{HeZRS15}. The networks after fine-tuning are evaluated through  \textit{top 1} and \textit{top 5} accuracy on  \textit{ILSVRC-12} \cite{imagenet_cvpr09} and \textit{CIFAR-10} \cite{cifar09}. 
The experiments were conducted with the popular neural networks framework \textit{Pytorch} on a GPU server with NVIDIA V-100 GPUs. As a baseline for \textit{CIFAR-10}, we used a pre-trained model with 95.17\% top-1 accuracy.

For fine-tuning, we used SGD with weight decay of $5\times 10^{-4}$. For the single layer fine-tuning model (Example~\ref{ex_rsn18_chkpoint95}), we fine-tuned the model for 30 epochs with a $10^{-4}$ learning rate. For full model compression (Example~\ref{ex_rsn18_chkpoint95_full}), models were fine-tuned for 120 epochs by SGD with initial learning rate $10^{-3}$ and decreased every 30 epochs by 10.

\begin{example}[Single layer compression in ResNet-18 trained on ILSVRC-12]\label{ex_resnet92}
We decomposed all convolutional kernels except layers 1, 8, and 13 were decomposed using ALS. The ranks-$(R_1, R_2, R_3)$ were chosen to give the number of model parameters close to those in CPD with rank-200. We then applied SSC (Algorithm~\ref{alg_ACEP}) after 3000 ALS updates. 
Figure~\ref{fig:resnet18_}(a) compares the relative approximation errors obtained by ALS in 13000 iterations and ALS+SSC. 
Figure~\ref{fig:resnet18_}(c) illustrates the convergence of the TC decomposition using ALS and ALS+SSC. After the SSC, the decomposition converged to a lower approximation error.

The relative change in approximation error, shown in Figure~\ref{fig:resnet18_}(b), can be 10\% for layers 2-6 and smaller for the other layers. We can say that the kernels of convolutional layers 2, 3, \ldots, 7, are well represented by low-rank tensors. Kernels in the last layers have much higher ranks. In our experiment, the approximation error for the last layer was even 0.7489. For this case, SSC is not much helpful. 

Figure~\ref{fig:resnet18_}(b) compares the SS of TC models estimated by ALS with and without SSC. The SS was very high for the last layers' decomposition.
The SS grows rapidly to a large number while the approximation error is still high. 
Approximation of those kernels by a relatively low-rank model will cause degeneracy. 
Using SSC, we can significantly reduce 
SS in all decompositions.
In summary, SSC will improve the approximation for tensors with good TC approximation and make the decomposition results more stable. 
\end{example}

\begin{figure*}
\centering
{\includegraphics[width=.29\linewidth]{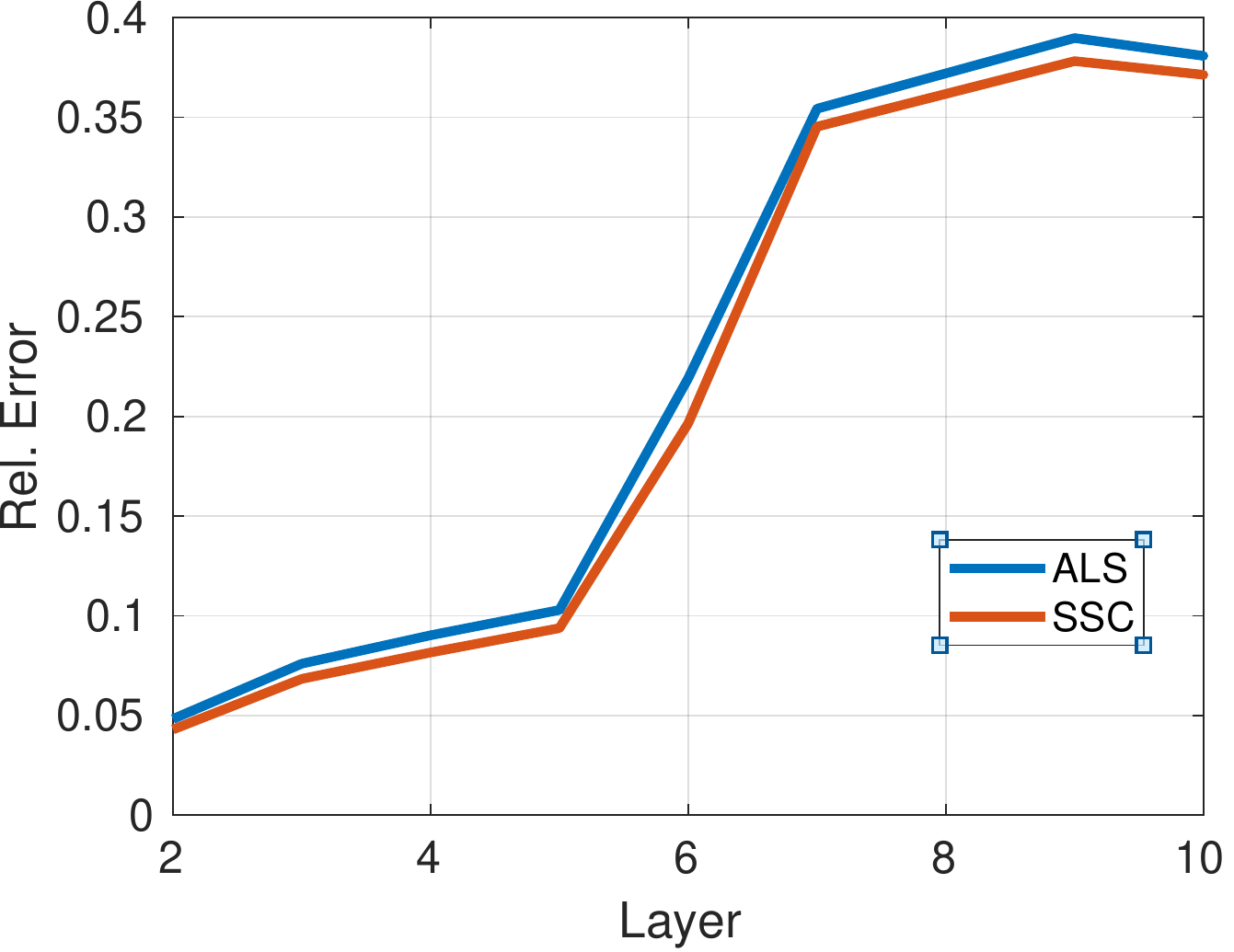}}\hfill
{\includegraphics[width=.30\linewidth]{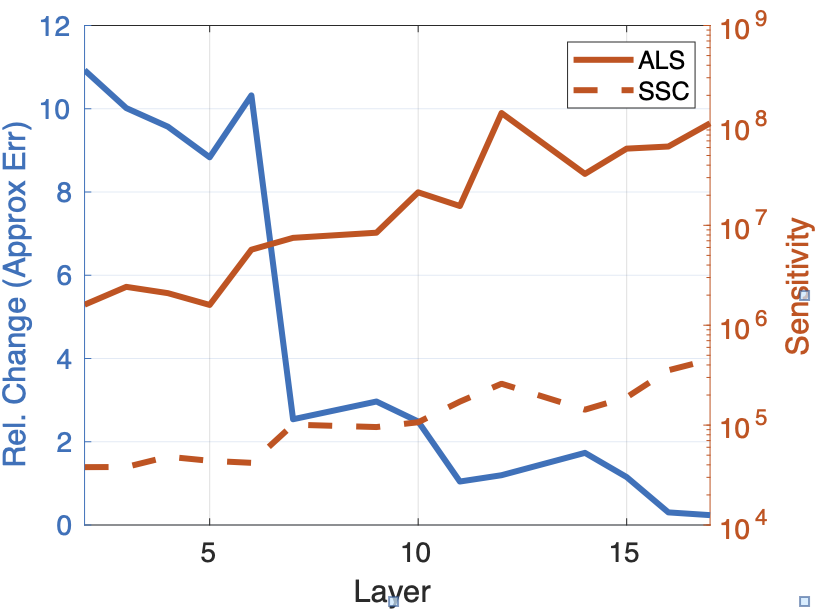}}\hfill
{\includegraphics[width=.29\linewidth]{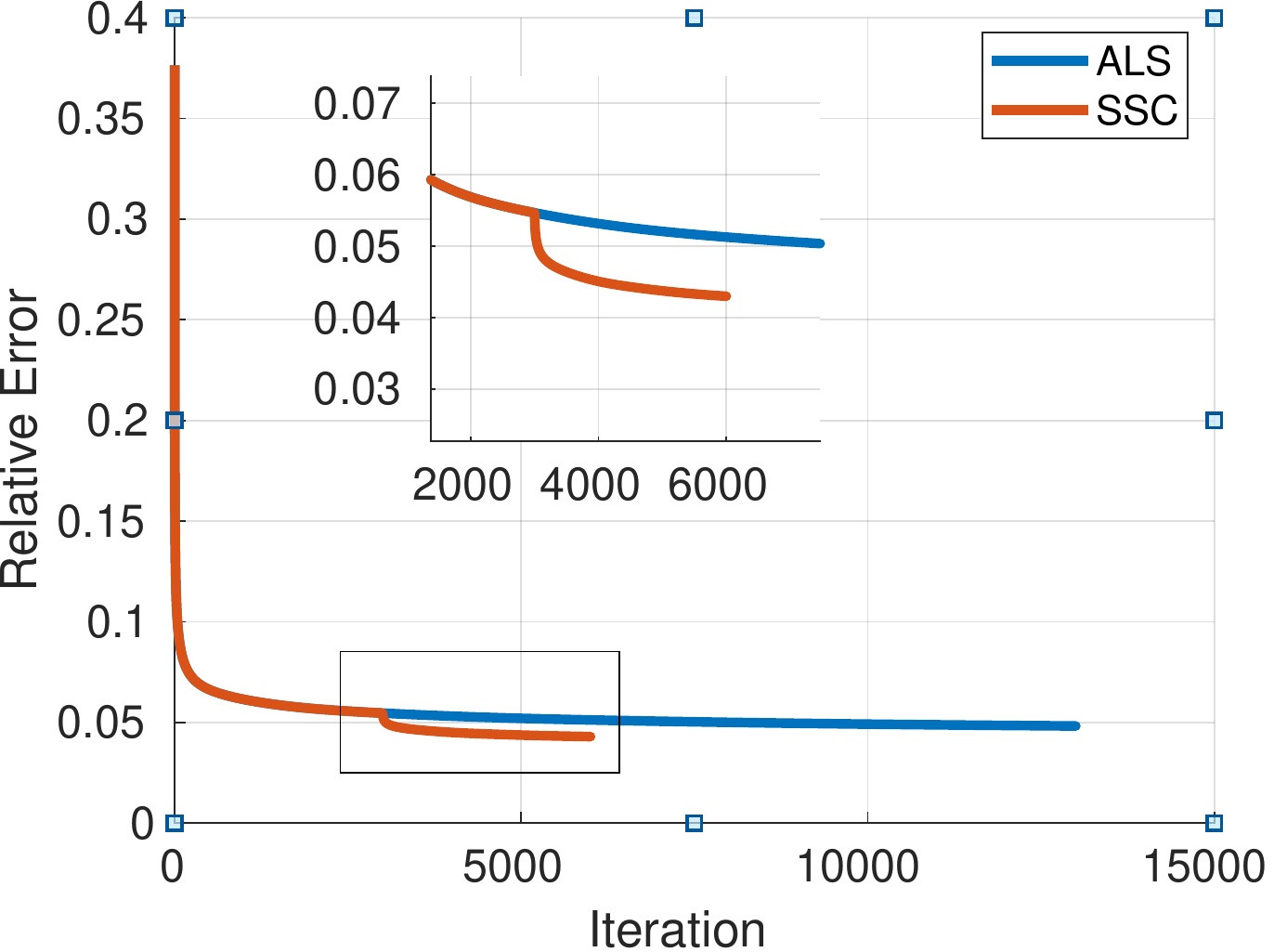}}
\caption{Decomposition of convolutional kernels in ResNet-18 in Example~\ref{ex_resnet92}.
(a) Relative errors obtained by ALS and ALS+SSC, (b) relative changes of the approximation error using SSC, and the SS of estimated tensors, (c) Convergence shown with and without SSC.}
    \label{fig:resnet18_}
\end{figure*}

\begin{figure*}[!ht]
\centering
\subfigure[]{\includegraphics[width=.33\linewidth, height = .2\linewidth]{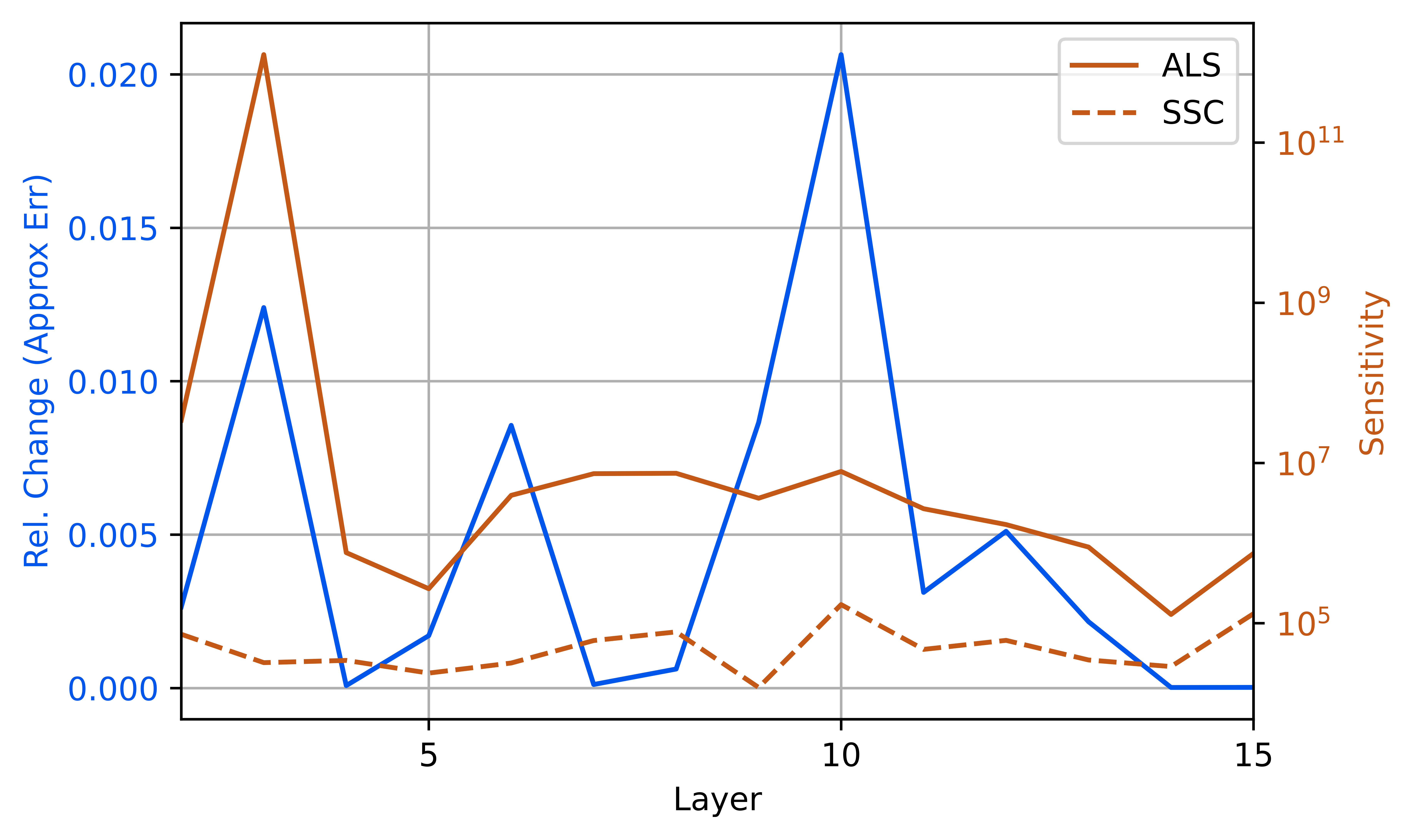}\label{fig:ex4_approx_err}}
    \subfigure[]{\includegraphics[width=.3\linewidth,height = .2\linewidth]{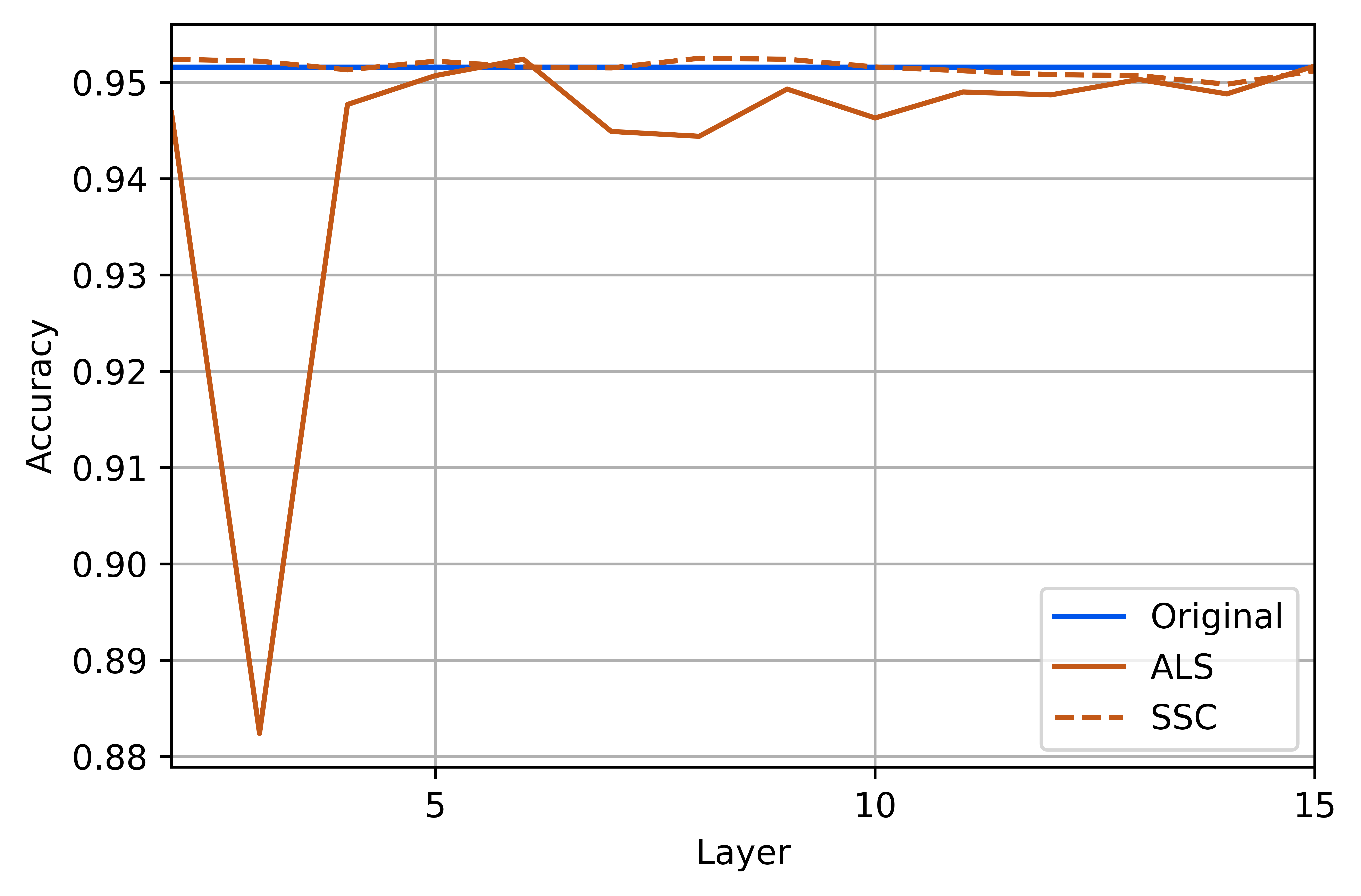}\label{fig:ex4_acc}}
    \subfigure[]{\includegraphics[width=.3\linewidth,height = .2\linewidth]{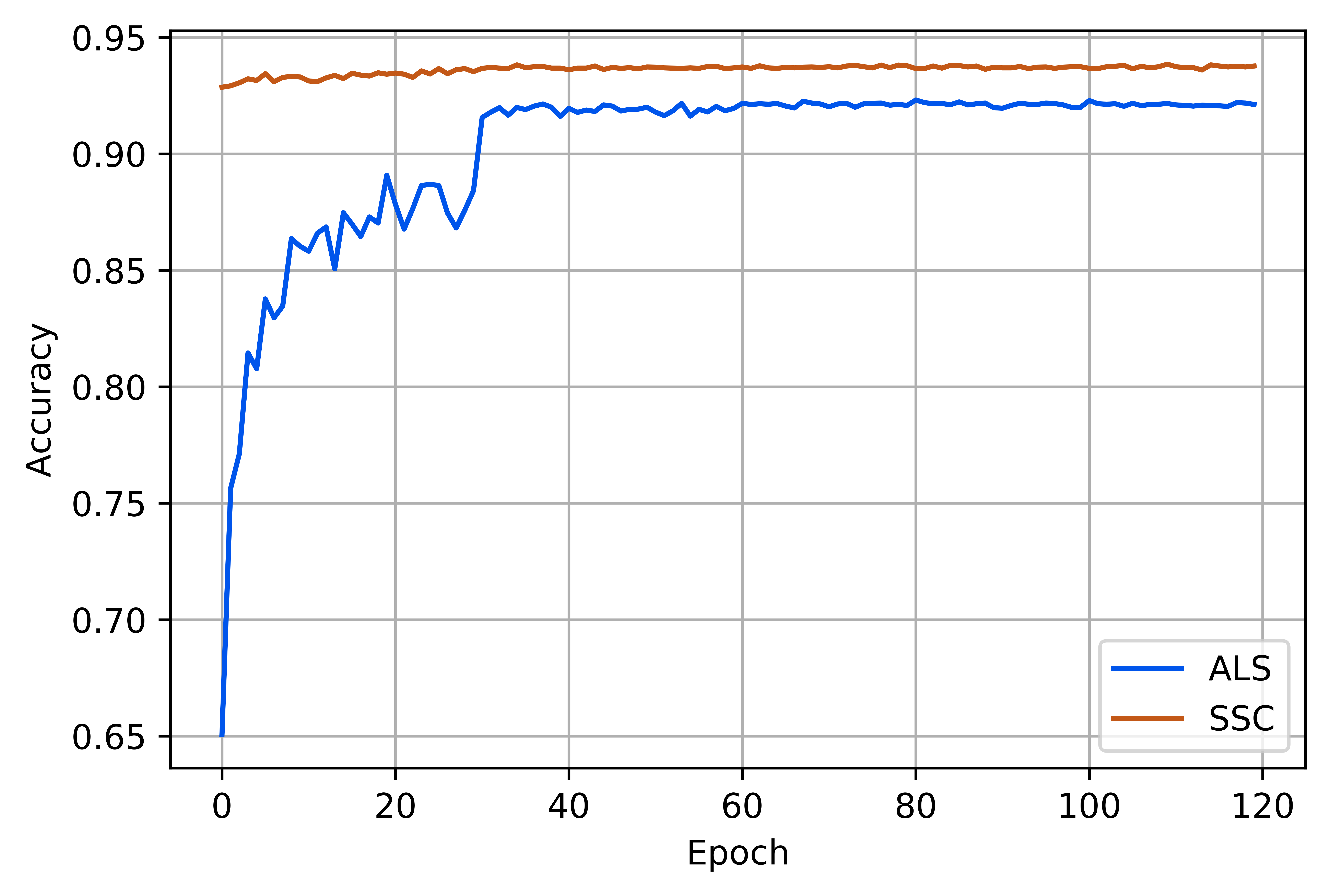}\label{fig:ex5_acc}}
    \caption{ ResNet-18 trained on CIFAR-10. (a) Sensitivity and relative change in approximation error of layers. (b) Accuracy of per layer fine-tuning. (c) Accuracy of full model compression and fine-tuning.}
    \end{figure*}

\begin{figure}[t]
\centering
\subfigure[Accuracy]{\includegraphics[width=.45\linewidth]{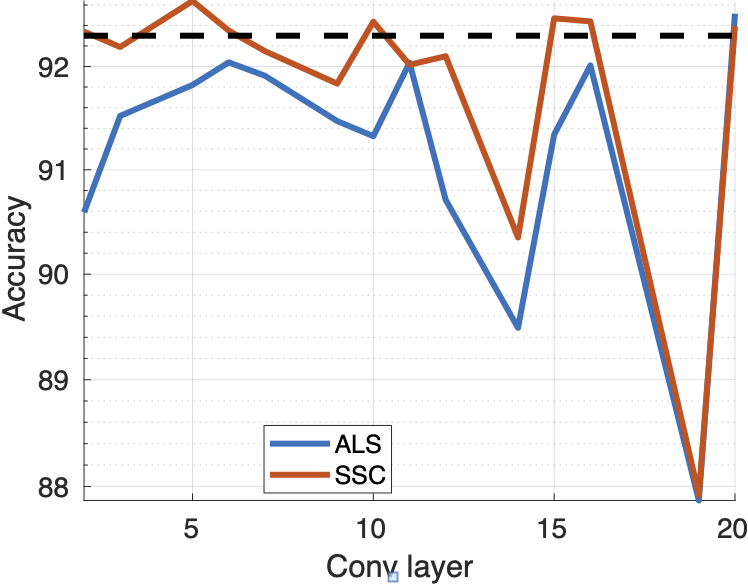}\label{fig:resnet18_ex3a}}
\subfigure[Layer 2]{\includegraphics[width=.45\linewidth]{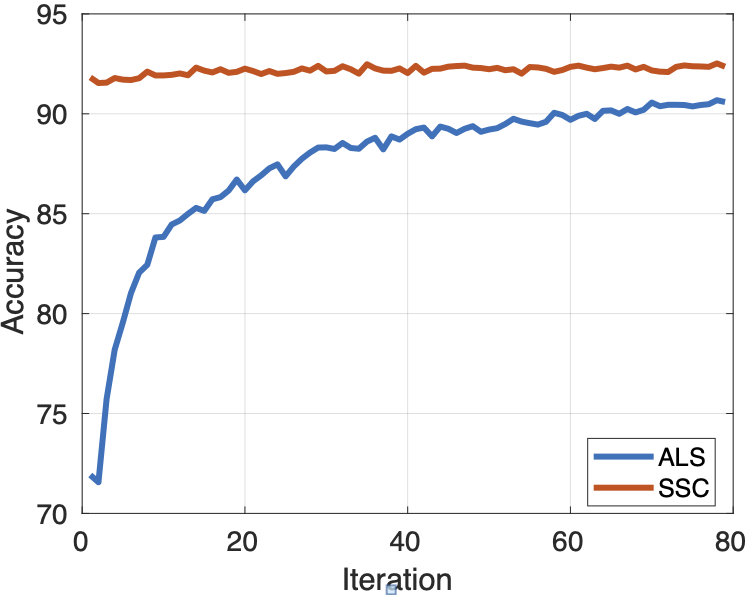}\label{fig:resnet18_ex3b}}
\caption{Performance comparison for ALS and (ALS+)SSC in the task for compression of single layers in ResNet-18 finetuned on CIFAR-10. The results are reported for Example~\ref{ex_resnet18_cifar10_v2}.}
\label{fig:resnet18_ex3}
\end{figure}


\begin{example}[Single layer compression of ResNet-18 on CIFAR-10]\label{ex_resnet18_cifar10_v2}
We follow the decomposition of kernels in Example~\ref{ex_resnet92}.
The last layer in ResNet-18, trained on ILSVRC-12, is modified and finetuned to work with the CIFAR-10 dataset.
Each convolutional layer is replaced with a TC layer presented in Section~\ref{sec::tclayer}. 
The new ResNet-18 is finetuned on the CIFAR-10 dataset
to update the TC-layer, while the other convolutional layers are frozen. We compared the accuracy of the new type ResNet-18 with TC-layer initialized by the results obtained with ALS and another network initialized by TC tensor obtained with SSC.

The original accuracy of this ResNet-18 for CIFAR-10 is 92.90\%. In addition, we do not compress 1x1 convolutional layers, which are layers 8 and 13.
TC approximates kernels in layers 2-11 with ranks-(10-10-10), 
while kernels in layers 12, 14-20 are compressed as in Example~\ref{ex_resnet92}. 

Figure~\ref{fig:resnet18_ex3} compares accuracy of the two ResNet-18 models. SSC improves the approximation errors and stabilizes the TC network, making the finetuned network attain the best accuracy faster than ResNet-18 using TC with ALS, as seen in Figure~\ref{fig:resnet18_ex3}(b) convergence of the accuracy for the 2nd convolutional layer.


\end{example}



\begin{example}[Single layer compression of ResNet-18 on CIFAR-10 with estimation of TC ranks]\label{ex_rsn18_chkpoint95}
%
We train the PyTorch ResNet-18 network adopted for the CIFAR-10 dataset for this experiment. We compute CP ranks for convolutional kernels and TC ranks using the rank selection described in Section \ref{sec::tclayer}. SSC gives smaller approximation error and smaller sensitivity than ALS (Figure~\ref{fig:ex4_approx_err}). Each TC layer is finetuned. Figure~\ref{fig:ex4_acc} shows that SSC helps to get better accuracy than ALS.

\end{example}



\begin{example}[Full network compression on CIFAR-10]\label{ex_rsn18_chkpoint95_full}
%
For this experiment we replace all convolutional layers of ResNet-18 except for \textit{conv1}, \textit{layer1.0.conv1} and \textit{layer4.1.conv2} by
TC layer with ALS and SSC as described in Example \ref{ex_rsn18_chkpoint95}, remaining convolutional layers are replaced by CPD layers \cite{PhanECCV2020}.  
Figure \ref{fig:ex5_acc} show that the model with SSC not only converges faster than ALS but also has a significantly higher final accuracy (93.77 \% vs. 92.12\%).

\end{example}

\begin{example}[Full network compression on ILSVRC-12]\label{ex::resnet18_ILSVRC}
 We provide extra comparison of full network compression for ResNet-18 and VGG-16 on  ILSVRC-12 summarized in Table~\ref{table:overall_results}. SSC  showed compression results comparable to existing methods. It opens a new direction for combined architectures with CPD-EPC and TC-SSC 
(See Example~\ref{ex_rsn18_chkpoint95_full}). 
 In addition, we validate our evaluation of TC-SSC and TC-ALS for single layer fine-tuning on ILSVRC-12 (see Figure \ref{fig::resnet18_imagennet_layerwise} in Appendix). Thank to lower sensitivity, TC-SSC exhibits stable convergence to a higher accuracy than TC-ALS.
\end{example}




\begin{table}[!h]
\caption{Comparison of different model compression methods on ILSVRC-12 validation dataset.}\label{table:overall_results}
%
{
\begin{tabular*}{1\linewidth}{@{\extracolsep{\fill}}l@{\hspace{1ex}}l@{}c@{}c@{}c@{}c}
\textbf{NN}
& \textbf{Method}                     & \textbf{$\downarrow$ FLOPs} & \textbf{$\downarrow$ Params} & \textbf{$\Delta$ top-1} & \textbf{$\Delta$ top-5} \\ \cline{1-6}

\multirow{5}{*}{\rotatebox[origin=c]{90}{VGG-16} }
&  {Asym}     & $\approx5.00$                                    &   - & -    &  -1.00     \\ 
                                &  {TKD+VBMF}     &  4.93  & - &   -    &  -0.50     \\
&   {CPD-EPC}     &  5.24    & 1.10  &    -0.94   &  -0.33   \\ 
&  {\bf SSC[Ours]}     &  5.30 & 1.10 &   -6.68   &  -3.93    \\
&  {\bf SSC[Ours]}     &  3.76  &  1.09 &  -1.47  &  -0.61    \\\cline{1-6}

\multirow{8}{*}{\rotatebox[origin=c]{90}{ResNet-18 }}
& {CG} & 1.61 & - & -1.62 & -1.03  \\
& {DCP} & 1.89 & - & -2.29 & -1.38\\
& {FBS} & 1.98 & - & -2.54 & -1.46\\
& {MUSCO} & 2.42 & - & -0.47 & -0.30\\
&  CPD-EPC & 3.09 & 3.82 & -0.69 & -0.15 \\ 
&  {\bf SSC[Ours]}     &  3.15 & 4.05 &   -1.97   &  -0.92    \\
&  {\bf SSC[Ours]}    &  2.49  &  3.76 &  -0.86   &  -0.3    \\\cline{1-6}
\end{tabular*}}
\begin{tabular}{@{}ll}
    Asym. \cite{zhang15accelerating} & TKD+VBMF \cite{Kim2016} \\
    DCP  \cite{zhuang2018discrimination} & CPD-EPC \cite{PhanECCV2020}\\
    FBS \cite{gao2018dynamic} & MUSCO \cite{gusak2019automated}\\
    CG \cite{hua2018channelGating} & 
\end{tabular}
%
%
\end{table}

\section{Related Works}

Since TC was introduced in \cite{Khoromskij-SC, Espig_2011}, many algorithms have been developed for various applications, see Section~\ref{sec::algs}. However, most studies do not realize the instability problem with TC. The decomposition for incomplete data is even more challenging, as seen in Example~\ref{ex_incompletTC}. 
No existing algorithm for TC is related to our proposed methods. 

Regarding the application for compression of CNN, the authors in \cite{wide_compression} encountered the problem of obtaining a good TC decomposition; the authors carefully chose the variance of the initial parameters. However, they were unaware of the numerical instability problem in their decomposition results and did not propose a decomposition method for the TC. Figure 6 in \cite{wide_compression} shows that the networks converged slowly and can take 80000 iterations. Slow convergence with the neural network training with ordinary TC layers, i.e., without sensitivity correction, can also be observed in Figure~\ref{fig:resnet18_ex3a} and Figure~\ref{fig:ex4_acc} for training ResNet-18 with CIFAR-10 dataset, Figure~\ref{fig::resnet18_imagennet_layerwise}(in Appendix) for ResNet-18 trained on ILSVRC-12 dataset.
With Sensitivity Correction, we can train the compressed neural networks quickly and obtain good performances, which are very close to the accuracy of the original neural networks, see, for example, Figure~\ref{fig:resnet18_ex3a} and Figure~\ref{fig:ex4_acc}.
\cite{pan2019compressing} compressed Recurrent Neural Networks with TC layer and implemented the layer as a sum of TT layers.
Despite the similarity in applications, the main targets in our work and other studies are different.

 \section{Conclusions}
 
 This paper presents a novel work on the Block term decomposition with shared core tensors (sBTD) and the Tensor Chain. We show that any TC/sBTD model can be unstable with diverging intensity, see Lemma~\ref{lem_tc_degenerarcy}. 
 We proposed sensitivity for TC as a measure of stability, 
and confirm the analysis in examples for synthetic data, images, and decomposition of convolutional kernels in ResNet-18. The most important contribution is the novel algorithms that can stabilize the TC/sBTD model and improve the convergence of the decomposition. For compression of CNNs, we proposed a new TC/sBTD layer, which comprises 3 convolutional layers. We show that our proposed methods can help the compressed CNN quickly attain the original accuracy in a few iterations. In contrast, the compressed network cannot be fine-tuned or converge very slowly without sensitivity correction, thereby demanding many iterations.  
 



\bibliographystyle{unsrt}  
\bibliography{BIBTENSORS2018,egbib}

\newpage
\appendix
\onecolumn

This supplementary presents proofs of Lemmas introduced in the main manuscript and provides detailed derivation of the Rotation method introduced in Section~4.1 and more illustrative figures for Examples~1-4.

\section{Tensor Contraction}

\begin{definition}[Tensor train contraction]\label{def_train_contract}\label{def_boxtime}
performs a tensor contraction between the last mode of $\tA$ and the first mode of $\tB$, 
to yield a tensor $\tC = \tA \bullet \tB$ of size $I_1  \times \cdots \times I_{N-1} \times J_2 \times \cdots \times J_K$
 the elements of which are given by 
\be
c_{i_1,\ldots,i_{N-1},j_2, \ldots,j_K} = \sum_{i_N = 1}^{I_N}  a_{i_1,\ldots,i_{N-1},i_N} \, b_{i_N,j_2, \ldots,j_K}  \notag 
.
\ee
A Tensor train can be expressed as train contraction of core tensors, $\tX = \tX_1 \bullet \tX_2 \bullet \cdots  \bullet \tX_N$. 
\end{definition}


\section{Proof of Lemma~\ref{lem::btd_tc} (Equivalence of BTD with shared core tensors and TC)}\label{secA::proof_btd_tc}

\begin{lemma}[Equivalence of BTD with shared core tensors and TC]
The constrained BTD with shared core tensors in (\ref{eq_btd_shared}) is a Tensor chain model
{\normalfont
$
\tY =  \circlellbracket \tA, \tB, \tC\circlerrbracket$}
where $\tA$ is of size $R_1 \times I_1 \times R_2$ with horizontal slices $\tA(t,:,:) = \bA_{t}$, $\tC$ of size $R_3 \times I_3 \times R_1$ and 
$\tC(:,:,t) = \bC_{t}$, $t = 1, \ldots, R_1$.
\end{lemma}

\begin{proof}
\end{proof}

The proof is straightforward from the definitions of BTD and TC models 
\be
\tY &=& \sum_{t = 1}^{R_1} \sum_{r=1}^{R_2} \sum_{s = 1}^{R_3}  \bA_t(:,r) \circ \tB(r,:,s) \circ \bC_t(:,s) \notag\\
&=& \sum_{t = 1}^{R_1} \sum_{r=1}^{R_2} \sum_{s = 1}^{R_3}  \tA(t,:,r) \circ \tB(r,:,s) \circ \tC(s,:,t).
\ee

\section{Proof of Lemma 2 (TC Degeneracy)}\label{sec::proofTCdegeneracy}

\begin{lemma}[TC Degeneracy]
For a given TC model, {\normalfont{$\tY = \circlellbracket \tA, \tB, \tC\circlerrbracket$}}, there is always a sequence of equivalent TC models with diverging TC intensities. 
\end{lemma}

\begin{proof} 
\end{proof}
\normalfont
We provide an example as proof for the TC model with rank $R_2 = 2$. The other cases can be seen straightforwardly.

Consider the sub-network $\tA \ttprod \tB$, apply the DMRG-like update rule to split it to a sequence of three cores,  
$$\tA \ttprod \tB = \tU  \ttprod \bS \ttprod \tV$$
where $\bU \, \bS \, \bV$ is thin-SVD of unfolding of $\tA \ttprod  \tB$ to a matrix of size $R_1 I_1 \times I_2 R_3$, 
$\bS = \diag(s_1, s_2)$ is a diagonal matrix of $R_2 = 2$ leading singular values, 
$\bU$ and $\bV$ are unfoldings of $\tU$ and $\tV$, respectively.
The tensor $\tY$ has an equivalent TC model
$\tY = \circlellbracket  \tU, \bS \ttprod \tV, \tC \circlerrbracket$.

We next define a matrix $\bQ = \left[\begin{matrix}
1 & x\\
x & 1
\end{matrix}\right]$. The tensor $\tY$ has another equivalent TC model given by
{\normalfont
\be
\tY = \circlellbracket  \tU \ttprod \bQ , \bQ^{-1} \bS \ttprod \tV, \tC \circlerrbracket \, 
\ee
}
but with an intensity  
\be
\alpha &=& \|\tU \ttprod \bQ\|_F  \, \|\bQ^{-1} \bS \ttprod \tV\|_F \, \|\tC\|_F \notag \\
&=& \|\bQ\|_F  \, \|\bQ^{-1} \bS\|_F \, \|\tC\|_F \notag \\
&=& \frac{(1+x^2) \sqrt{2(s_1^2 + s_2^2)}}{|1-x^2|} \|\tC\|_F \,.
\ee 
It is obvious that when $x$ approaches 1, the intensity $\alpha$ goes to infinity. 
For the general case, the proof can be derived similarly with a symmetric matrix $\bQ$ of size $R_2 \times R_2$ which has ones on the diagonal and two non-zero off-diagonal elements $x$.

\section{Proof of Lemma 3 (Sensitivity for TC)}

\begin{proof}[Proof of Lemma 3 (Sensitivity of TC model)]
\end{proof}
Consider the error tensor
\be
 &&\circlellbracket \tA_1 + \delta_{\tA_1}, \tA_2 + \delta_{\tA_2}, \ldots, \tA_N + \delta_{\tA_N} \circlerrbracket - \circlellbracket \tA_1, \tA_2, \ldots, \tA_N \circlerrbracket    \notag \\
 &=& \circlellbracket \delta_{\tA_1}, \tA_2, \ldots, \tA_N \circlerrbracket + 
  \circlellbracket {\tA_1}, \delta_{\tA_2}, \tA_{3}, \ldots, \tA_N \circlerrbracket
  + \cdots + 
  \circlellbracket {\tA_1}, {\tA_2}, \ldots, \delta_{\tA_N} \circlerrbracket
  +  \circlellbracket \delta_{\tA_1}, \delta_{\tA_2}, {\tA_3}, \ldots, {\tA_N} \circlerrbracket
  \notag \\
 && +
 \circlellbracket \delta_{\tA_1}, {\tA_2}, \ldots, \delta_{\tA_N} \circlerrbracket+
 \cdots
 + \circlellbracket \delta_{\tA_1}, \delta_{\tA_2}, \ldots, \delta_{\tA_N} \circlerrbracket
 \notag .
\ee
TC terms in the above expression are uncorrelated, and the expectation of the terms consisting of two or more $\delta_{\tA_n}$ is zero.
Hence, the expectation in (\ref{eq_sstc}) is rewritten as 
\begin{align}
&E\{\|\tY - \circlellbracket \tA_1 + \delta_{\tA_1}, \tA_2 + \delta_{\tA_2}, \ldots, \tA_N + \delta_{\tA_N} \circlerrbracket\|_F^2\} \notag \\
&= 
E\{\|\circlellbracket \delta_{\tA_1}, \tA_2, \ldots, \tA_N \circlerrbracket\|_F^2\} + 
  E\{\|\circlellbracket {\tA_1}, \delta_{\tA_2}, \tA_{3}, \ldots, \tA_N \circlerrbracket\|_F^2\}
  + \cdots + 
  E\{\|\circlellbracket {\tA_1}, {\tA_2}, \ldots, \delta_{\tA_N} \circlerrbracket\|_F^2\} \,  \notag \\
&=
E\{\|\circlellbracket \delta_{\tA_1}, \tA_{-1} \circlerrbracket\|_F^2\} + 
  E\{\|\circlellbracket  \delta_{\tA_2}, \tA_{-2} \circlerrbracket\|_F^2\}
  + \cdots + 
  E\{\|\circlellbracket \delta_{\tA_N}, \tA_{-N} \circlerrbracket\|_F^2\}
  \label{eq_ss2}.
\end{align}
Thank to looping structure of the TC tensor, we can cyclic shift $\delta_{\tA_n}$ to the first core tensor, and the rest part of the tensor is the TT-tensor $\tA_{-n}$. 

We reshape the TC tensor, $\circlellbracket  \delta_{\tA_n}, \tA_{-n} \circlerrbracket$, to mode-1 unfolding and expand its Frobenius norm
as
\be
E\{\|\circlellbracket \delta_{\tA_1}, \tA_{-1} \circlerrbracket\|_F^2\} &=& E\{ \|[\delta_{\tA_1}]_{(2)} [\tA_{-1}]_{(1,N)}\|_F^2 \} \notag  \\
 &=& E\{ \tr([\delta_{\tA_1}]_{(2)}^T [\delta_{\tA_1}]_{(2)}) ([\tA_{-1}]_{(1,N)} [\tA_{-1}]_{(1,N)}^T )) \} =  \sigma^2 I_1 \tr([\tA_{-1}]_{(1,N)} [\tA_{-1}]_{(1,N)}^T )) \notag \\
&=& \sigma^2 I_1  \|\tA_{-1}\|_F^2 \notag .\ee
Together with (\ref{eq_ss2}), we finally complete the proof.

\section{Proof of Lemma 4 (Balanced norm for minimal sensitivity)}


\begin{proof}

Since $\alpha_1 \alpha_2 \cdots \alpha_N = 1$, we have  two equivalent TC models {\normalfont{$\circlellbracket  \tA_1, \tA_2, \ldots,  \tA_N \circlerrbracket = \circlellbracket\alpha_1 \tA_1, \alpha_2 \tA_2,  \ldots, \alpha_N\tA_N \circlerrbracket$}}. Sensitivity of the new model is given by
\be
ss( \circlellbracket \alpha_1 \tA_1, \alpha_2 \tA_2,  \ldots, \alpha_N\tA_N \circlerrbracket) &=&
\sum_{n = 1}^{N} I_n \left(\prod_{k \neq n} \alpha_k^2\right) \, \| \tA_{-n}\|_F^2  \notag \\
 &=&  \sum_{n = 1}^{N} \frac{\beta_n^2}{\alpha_n^2} \notag   \\
&\ge & N \sqrt[N]{ \frac{ \beta_1^2 \, \beta_2^2 \, \cdots  \beta_N^2}{\alpha_1^2 \alpha_2^2 \cdots \alpha_N^2}} =  N \beta^2 \, . \notag 
\ee
The inequality is between the arithmetic mean and the geometric mean, and the equality in it holds when all terms are equal each to the other, i.e.,  
\be
\frac{\beta_1}{\alpha_1} = \frac{\beta_2}{\alpha_2} = \cdots =  \frac{\beta_N}{\alpha_N} = \sqrt[N]{\frac{\beta_1 \, \beta_2 \cdots \, \beta_N}{\alpha_1\, \alpha_2\, \cdots \alpha_N}} = \beta \, .
\ee
Hence $\alpha_n = \frac{\beta_n}{\beta}$. This completes the proof.
 \end{proof}

\section{Rotation method for Sensitivity Correction}\label{sec::rotation_full}

This section presents the complete derivation of the Rotation algorithm in Section~\ref{sec::rotation}. Due to space limitations, we present a brief derivation of the Rotation method in the main manuscript.

Due to non uniqueness of the model up to rotation, we can rotate core tensors by invertible matrices such that the new representation of the TC tensor has minimum sensitivity 
\be
\tY = \circlellbracket \bQ_N^{-1} \, \ttprod \tA_1 \ttprod \bQ_1, \bQ_1^{-1} \ttprod \tA_2 \ttprod \bQ_2, \bQ_2^{-1} \ttprod \tA_3 \ttprod \bQ_3, \ldots, \bQ_{N-1}^{-1} \tA_N \bQ_N  \circlerrbracket \,.
\ee
For simplicity, we derive the algorithm to find the optimal matrix, $\bQ$ of size $R_2 \times R_2$ which rotates the first two core tensors, $\tA_1$ and $\tA_2$, and gives a new equivalent TC tensor $\tY_{\bQ} = \circlellbracket \tA_1 \ttprod \bQ, \bQ^{-1} \ttprod \tA_2, \tA_3, \ldots , \tA_N  \circlerrbracket$.

The optimal matrix $\bQ$  minimizes the sensitivity of $\tY_{\bQ}$
\be
\min_{\bQ} \quad ss(\tY_{\bQ}) = I_1 \|\bQ^{-1} \tA_{-1}\|_F^2  + I_2 \|\tA_{-2} \bQ \|_F^2 + \sum_{n = 3}^{N} I_n \|\tA_{-n}\|_F^2 \,\label{eq_rotss_N}
\ee
where $\tA_{-n} = \tA_{n+1} \ttprod \cdots \ttprod \tA_N \ttprod \tA_1 \ttprod \cdots \ttprod \tA_{n-1}$.

We next define two matrices, $\bX_1$ of size $R_3 \times R_3$ and $\bX_2$ of size $R_1 \times R_1$, as self contraction of the tensor $\tA_{-(1,2)} = \tA_{3} \ttprod \cdots \ttprod \tA_N$ along all modes but mode-1 and mode-$N$, respectively \be
\bX_1 = [\tA_{-(1,2)}]_{(1)} [\tA_{-(1,2)}]_{(1)}^T \,, \quad 
\bX_2 = [\tA_{-(1,2)}]_{(N)} [\tA_{-(1,2)}]_{(N)}^T \quad 
\ee
and two square matrices, $\bT_1$ and $\bT_2$, of size $R_2 \times R_2$
\be
\bT_1 &=& \sum_{i_2 = 1}^{I_2}  \tA_2(:,i_2,:) \bX_1 \tA_2(:,i_2,:)^T  , \qquad \quad
\bT_2 = \sum_{i_1 = 1}^{I_1}  \tA_1(:,i_1,:)^T \bX_2 \tA_1(:,i_1,:) \,.
\ee
See illustration for efficient computation of $\bX_1$ and $\bX_2$ in Figure~\ref{fig_X12}.

\begin{figure}[t]
\centering
\subfigure[$\bX_1$: self-contraction of the TT-tensor, $\tA_{3:N}$, along modes-$2,3,\ldots, N$]{\includegraphics[width=.45\linewidth,trim = 0.0cm 6cm 0cm 0cm,clip=true]{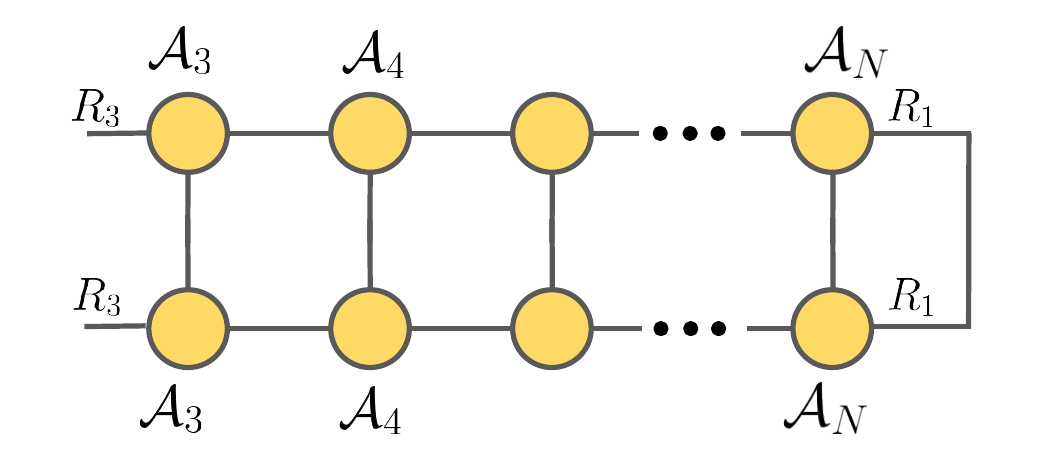}\label{fig:X1}}
\hfill
\subfigure[$\bX_2$: self-contraction of the TT-tensor, $\tA_{3:N}$, along modes-$1,3,\ldots, N-1$]{\includegraphics[width=.45\linewidth, trim = 0.0cm 0cm 0cm 6cm,clip=true]{TC/selfcontraction_X}\label{fig:X2}}
\caption{Computation of $\bX_1$ and $\bX_2$ as self-contraction of the tensor $\tA_{-(1,2)} = \tA_{3:N} = \tA_3 \ttprod \tA_4 \ttprod \cdots \ttprod \tA_N$. For efficient computation, we can compute self-contraction of core tensors $\tA_3$, $\tA_4$, \ldots, $\tA_N$ along their second modes to give matrices, $\bB_n = \sum_{i = 1}^{I_n}
 \tA_n(:,i,:) \otimes \tA_n(:,i,:)$ of size $R_n^2 \times R_{n+1}^2$. Then $\bX_1$ is given by $\vtr{\bX_1} = \bB_3 \cdots \bB_{N-1} \bB_{N} \1$
 and $\vtr{\bX_2} = \bB_N^T \bB_{N-1}^T \cdots \bB_{4}^T  \bB_3^T \1$.
}\label{fig_X12}
\end{figure}

We represent the matrix $\bQ\bQ^T = \bU \bS \bU^T$ in form of its eigenvalue decomposition (EVD), where $\bU$ is an orthogonal matrix of size $R_2 \times R_2$ and $\bS = \diag(s_1, \ldots, s_{R_2})$. The sensitivity in (\ref{eq_rotss_N}) is then computed as
\be
ss(\tY_{\bQ}) &=& I_1 \|\bQ^{-1} \tA_{-1}\|_F^2  + I_2 \|\tA_{-2} \bQ \|_F^2 + \sum_{n = 3}^{N} I_n \|\tA_{-n}\|_F^2  \notag \\
&=&  \sum_{n = 3}^{N} I_n \|\tA_{-n}\|_F^2  + I_1 \tr( \bT_1 \bQ^{-1} \bQ^{-1 T})  + I_2 \tr(\bT_2 \bQ \bQ^T)  \notag \\
&=& \sum_{n = 3}^{N} I_n \|\tA_{-n}\|_F^2  + I_1 \tr( (\bU^T \bT_1 \bU) \bS^{-1} )  + I_2 \tr((\bU^T \bT_2 \bU) \bS)  \notag .
\ee
Instead of seeking $\bQ$, we find an orthogonal matrix $\bU$ and a diagonal matrix $\bS$
\be
ss(\tY_{\bQ}) =  \sum_{n = 3}^{N} I_n \|\tA_{-n}\|_F^2  + \sum_{r = 1}^{R_2} I_1 (\bu_{r}^T \bT_1 \bu_{r}) \frac{1}{s_{r}} + I_2 (\bu_{r}^T \bT_2 \bu_{r}) s_r \notag \\
\ge 
\sum_{n = 3}^{N} I_n \|\tA_{-n}\|_F^2  + \sum_{r = 1}^{R_2} 2 \sqrt{I_1I_2  (\bu_{r}^T \bT_1 \bu_{r}) (\bu_{r}^T \bT_2 \bu_{r}) } \,.
\ee
The equality holds when 
\be
\displaystyle s_r^{\star} = \sqrt{ \frac{I_1 \bu_r^T \bT_1 \bu_r}{I_2 \bu_r^T \bT_2 \bu_r}}
\ee
for $r = 1, \ldots, R_2$. Given the optimal $s_r^{\star}$, we find the orthogonal matrix $\bU$ in the following optimization problem
\be
\min_{\bU \in  St_{R_2}} \quad \sum_{r = 1}^{R_2} \sqrt{(\bu_{r}^T \bT_1 \bu_{r}) (\bu_{r}^T \bT_2 \bu_{r})} \,,
\label{eq_obj_U_2}
\ee
which can be solved using the conjugate gradient algorithm on the Stiefel manifold \cite{Zaiwen_2012}.

{\bf Initialization.} Applying the Cauchy-Schwarz inequality, the objective function in (\ref{eq_obj_U}) is bounded above by $\frac{1}{2} \sum_{r = 1}^{R_2} (\bu_{r}^T \bT_1 \bu_{r}) + (\bu_{r}^T \bT_2 \bu_{r}) = \frac{1}{2} \tr(\bU^T (\bT_1 + \bT_2) \bU)$. We can initialize $\bU$ by eigenvectors of $(\bT_1 + \bT_2)$.

The rotation method is then applied to the next pair $\tA_2$ and $\tA_3$, $\tA_3$ and $\tA_4$, \ldots, $\tA_N$ and $\tA_1$, \ldots until the update reaches a stopping criterion.
Pseudo-codes of the proposed algorithm for order-3 TC are listed in Algorithm~\ref{alg_rotation}.

\begin{algorithm}[H] 
\caption{\tt{Rotation Method for SSC}\label{alg_rotation}}
\DontPrintSemicolon \SetFillComment \SetSideCommentRight
\KwIn{TC tensor $\tY = \circlellbracket  \tA, \tB, \tC \circlerrbracket$:  $(I_1 \times I_2 \times I_3)$,  and bond dimensions $R$}
\KwOut{$\hat{\tY} = \tY$ such that 
$\min \quad ss(\hat{\tY})$}
\resumenumbering
\Begin{
\Repeat{a stopping criterion is met}{
\For{$n = 1, 2, 3$}
	{
	    $\bT_1 = \sum_{i_2 = 1}^{I_2}  \bB_{i_2} (\bC_{(1)} \bC_{(1)}^T) \bB_{i_2}^T$, \; 
        $\bT_2 = \sum_{i_1 = 1}^{I_1}  \bA_{i_1}^T (\bC_{(3)} \bC_{(3)}^T) \bA_{i_1}$\;
 
	    Solve $\bU^{*} =  \arg\min_{\bU \in  St_{R_2}} \quad \sum_{r = 1}^{R_2} \sqrt{(\bu_{r}^T \bT_1 \bu_{r}) (\bu_{r}^T \bT_2 \bu_{r})}$\;
        
        \lFor{$r = 1 \ldots, R_2$}{ $\displaystyle s_r^{\star} = \sqrt{ \frac{I_1 \bu_r^T \bT_1 \bu_r}{I_2 \bu_r^T \bT_2 \bu_r}}$} 
        
        Rotate        $\tA \leftarrow \tA \ttprod \bU \diag(\sqrt{s_1}, \ldots,  \sqrt{s_{R_2}}, \ldots) \bU^T$
        
        Rotate $\tB \leftarrow \bU \diag(1/\sqrt{s_1}, \ldots, 1/\sqrt{s_{R_2}}, \ldots) \bU^T \ttprod \tB$
        
		Cyclic-shift of dimensions $\hat{\tY} = \circlellbracket  \tA, \tB, \tC \circlerrbracket \leftarrow \circlellbracket  \tB, \tC, \tA \circlerrbracket$\;
	}
    }
}
\end{algorithm}

\section{Sensitivity Correction for Higher Order TC}
\label{sec::ssc_highorder}

The optimization problem for Sensitivity correction for higher order TC is formulated in a similar form for TC of order-3, i.e., minimizing the SS of the model while keeping the approximation error bounded
\begin{align}
    \min  \quad & ss(\btheta) = \sum_n I_n \|\tA_{-n} \|_F^2 \,  \label{eq_ssmin_ho}\\
\text{s.t.}\quad &  c(\btheta) =  \|\tY - \hat{\tY}\|_F^2 \le \delta^2 \notag ,
\end{align}
where $\hat{\tY} =  \circlellbracket \tA_1, \tA_2,  \ldots, \tA_N  \circlerrbracket$ and
$\delta$ can be the approximation error of the current TC model, i.e., $\delta = \|\tY - \hat{\tY}_0\|_F$.

In order to update $\tA_1$, we rewrite the sensitivity function as function of the core tensor $\tA_1$ 
\begin{align}
    ss(\btheta) &= I_1 \|\tA_{-1}\|_F^2 +  \sum_{n = 2}^{N} I_n \|\tA_{-n}\|_F^2  \notag \\
    &= I_1 \|\tA_{-1}\|_F^2 +  \sum_{n = 2}^{N} I_{n} \|\tA_{n+1:N} \ttprod \tA_1 \ttprod \tA_{2:n-1} \|_F^2  \notag \\
    &= I_1 \|\tA_{-1}\|_F^2 +  \sum_{n = 2}^{N} I_{n} \,  \|\bL_{n}^T \ttprod \tA_1 \ttprod \bS_n \|_F^2  %
\notag \\
    &= I_1 \|\tA_{-1}\|_F^2 +  \tr(\bA_1^T \bA_1  \bQ_1)  \notag \\
\notag 
\end{align}
where $\bA_1$ is mode-2 unfolding of $\tA_1$ or the factor matrix of this core tensor in the equivalent TKD/TT decomposition, and  
\be
\bQ_1 = \sum_{n = 2}^{N} I_{n}\left(\bS_n \bS_n^T \otimes \bL_n \bL_n^T \right)
\ee
 $\tA_{n+1:N} = \tA_{n+1} \ttprod \tA_{n+2} \ttprod \cdots \ttprod \tA_N$ is a TT-tensor of order $(N-n+2)$ and size $R_{n+1} \times I_{n+1} \times I_{n+2} \times \cdots \times I_N \times R_1$, 
and 
$\tA_{2:n-1} = \tA_2 \ttprod \tA_{3} \ttprod \cdots \ttprod \tA_{n-1}$ is a TT-tensor of order $n$ and size $R_2 \times I_2 \times I_{3} \times \cdots \times I_{n-1} \times R_n$.
$\bL_n$ is unfolding along the last mode of $\tA_{n+1:N}$,  
and $\bS_n$ is unfolding along the first mode of $\tA_{2:n-1}$.

The product $\bL_n \bL_n^T$ is self-contraction of the TT-tensor, $\tA_{n+1:N}$, along all modes but the last mode.
$\bS_n \bS_n^T$ is self-contraction of the TT-tensor, $\tA_{2:n-1}$, along all modes but the first mode. Efficient computation of similar self-contraction product is explained in Figure~\ref{fig_X12}.

Next, we define $\bZ$ unfolding along with the first and last mode of the tensor $\tA_{-1}$. The optimization problem in (\ref{eq_ssmin_ho}) is rewritten as constrained quadratic programming 
\be
\min_{\bA_1} \quad &\tr(\bA_1^T \bA_1 \bQ_1) + I_1 \|\tA_{-1}\|_F^2 \\
\text{s.t} \quad & \|\bY_{(1)} - \bA_1 \bZ^T \|_F^2 \le \delta^2 \notag,
\ee
$\bY_{(1)}$ is mode-1 unfolding of $\tY$.
We note that the TT-tensor $\tA_{-1}$ has no the term $\bA_1$.
The same update rule is applied to other core tensors. For tensors with mixing entries, the above optimization problem can be extended by incorporating a binary indicator tensor $\tW$ in the constraint function, i.e.
$\|\bW_{(1)}.*(\bY_{(1)} - \bA_1 \bZ^T)\|_F^2 \le \delta^2$. Elements of the tensor $\tW$ specify the missing elements by zeroes, and ones for the observed ones. Decomposition of incomplete data is not in the main focus of our paper.  

\section{TC Layer Implementation}
\label{sec::tclayer_2}
Our implementation of TC-layer and CP-layer is shown in Figure~\ref{fig:decomposed_layers}.

\begin{figure}[!ht]
\centering
\subfigure[]{\includegraphics[width=.7\linewidth]{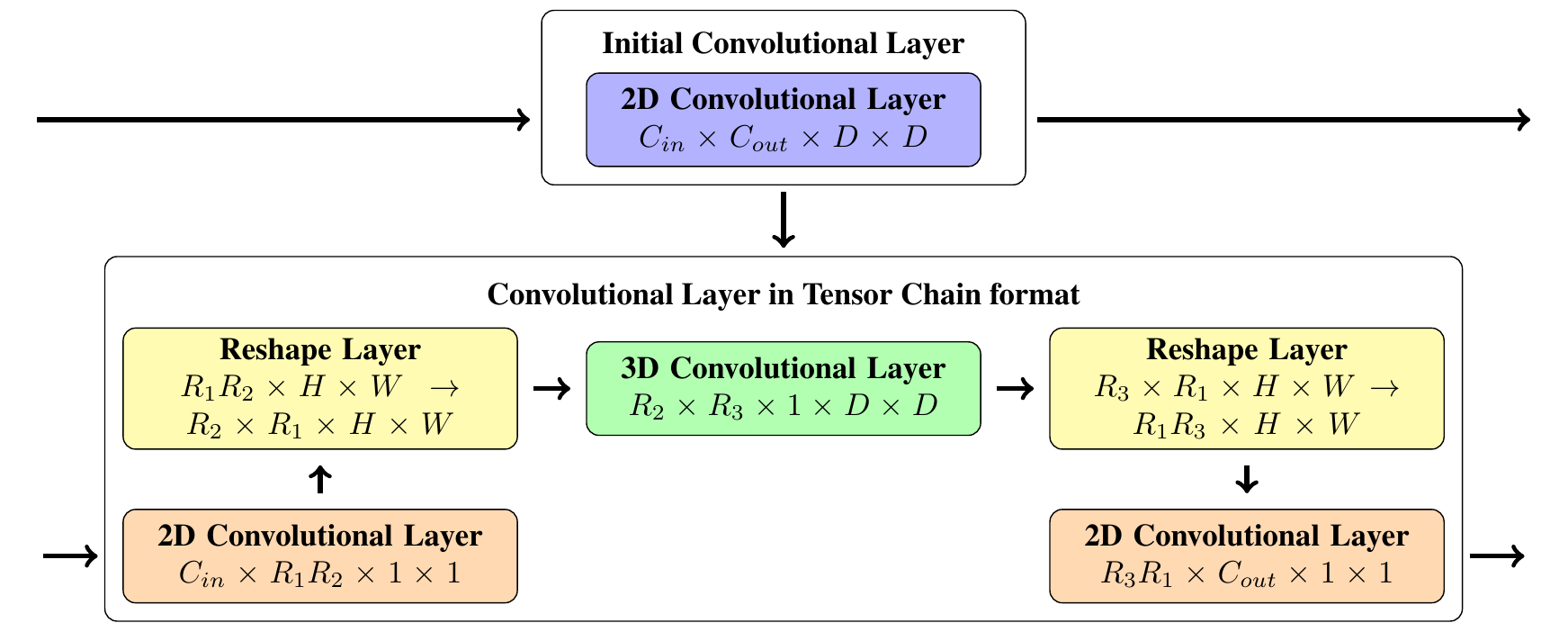}\label{fig:tc_layer_structure}}
\hfill
\subfigure[]{\includegraphics[width=.7\linewidth]{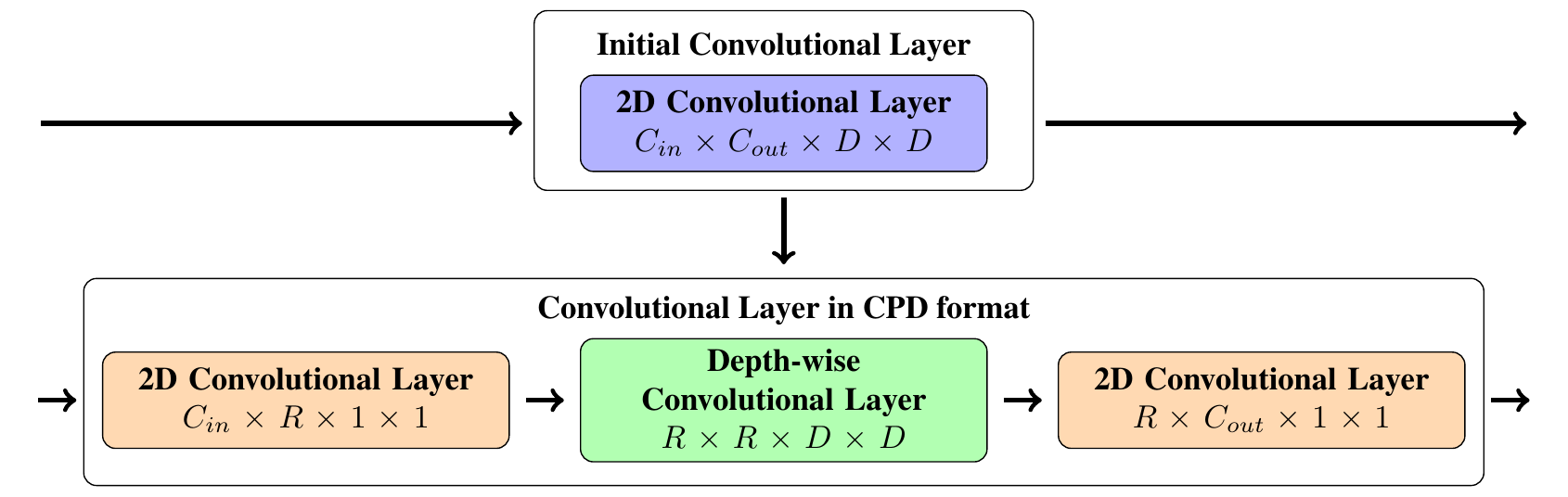}\label{fig:cp_layer_structure}}

    \caption{Graphical illustration to the proposed layer formats that show how decomposed factors are used as new weights of the compressed layer. $C_{in}$,$C_{out}$ are the number of input and output channels,  $D$ is a kernel size. (a) TC layer format, $R_1, R_2, R_3$ are TC ranks, $H$ and $W$ are the input dimensions. (b) CPD layer format, $R$ is a CPD rank}
    \label{fig:decomposed_layers}
\end{figure}

\subsection{TC Layer Python Implementation}
\newpage
\begin{python}

class TC_layer(nn.Module):
    def __init__(self, layer, factors):
        super(TC_layer, self).__init__()
        
        self.factors = [torch.tensor(U, dtype=torch.float32) for U in factors]
        self.c_in = self.factors[1].shape[1]
        self.c_out = self.factors[0].shape[1]
        self.r1 = self.factors[1].shape[0]
        self.r2 = self.factors[1].shape[2]
        self.r3 = self.factors[2].shape[2]
        
        self.h = int(np.sqrt(self.factors[2].shape[1]))
        self.w = int(np.sqrt(self.factors[2].shape[1]))
        
        self.padding = layer.padding
        self.stride = layer.stride
        self.dilation = layer.dilation
        self.kernel_size = layer.kernel_size
        self.is_bias = layer.bias is not None
        if self.is_bias:
            self.bias = layer.bias
        
        self.conv1 = nn.Conv2d(in_channels=self.c_in, out_channels=self.r1*self.r2,
                              kernel_size=(1, 1), bias=False)
        self.conv2 = nn.Conv3d(in_channels=self.r2, out_channels=self.r3, 
                              kernel_size=(1, self.h, self.w), 
                              padding=(0, self.padding[0], self.padding[1]), 
                              stride = (1, self.stride[0], self.stride[1]), bias=False)
        self.conv3 = nn.Conv2d(in_channels=self.r1*self.r3, out_channels=self.c_out,
                              kernel_size=(1, 1), bias=False)
        
        self.__replace__()
        
    def __replace__(self):
        C_out, C_in, C_ker = self.factors
        with torch.no_grad():
            self.conv1.weight = nn.Parameter(torch.tensor(C_in).permute(0, 2, 1).reshape(self.r1*self.r2, self.c_in, 1, 1))
            self.conv2.weight = nn.Parameter(torch.tensor(C_ker).permute(2, 0, 1).reshape(self.r3, self.r2, 1, self.h, self.w))
            self.conv3.weight = nn.Parameter(torch.tensor(C_out).permute(1, 2, 0).reshape(self.c_out, self.r1*self.r3, 1, 1))
            if self.is_bias:
                self.conv3.bias = nn.Parameter(self.bias)

    def forward(self, x):
        out1 = self.conv1(x)
        H, W = out1.shape[2], out1.shape[3]
        out1_reshaped = out1.view((-1, self.r1, self.r2, H, W)).permute(0, 2, 1, 3, 4)
        out2 = self.conv2(out1_reshaped)
        out2_reshaped = out2.permute(0, 2, 1, 3, 4).reshape((-1, self.r1*self.r3, 
                                                             int(H / self.stride[0]), 
                                                             int(W / self.stride[1])))
        out3 = self.conv3(out2_reshaped)
        return out3
\end{python}

\newpage


\section{Additional Experimental Results}
Due to space limitations, some figures for Examples in the main manuscript are presented in Appendix. We also provide more examples and more convincing comparison between our proposed method and the existing algorithms for TC and BTD with shared coefficients.
Examples in this manuscript are summarized in Table~\ref{tab_examples}.

\begin{table}[t]
    \centering
    \caption{List of Examples.}
    \label{tab_examples}
    \begin{tabular}{c l l l  }
    {\bf Ex.no.} & {\bf  Tensor size} & {\bf Description} & {\bf No. runs}  \\
    \hline
    \multicolumn{2}{l}{For synthetic tensors } \\
    \ref{ex_7x7x7} & $7 \times 7 \times 7$, bonds $(3-3-3)$ &  cores with bond exceeding dimensions & 10000 \\
    \ref{ex_27x27x27_rank25} & $27 \times 27 \times 27$, bonds $(5-5-5)$ & cores with collinear factor & 50 \\
    \ref{ex_incompletTC} & $9 \times 9 \times 9$, bond ($3-3-3$) & incomplete tensors with 50\% missing elements & 100\\
    \ref{ex::A1} &  & full version of Example~\ref{ex_7x7x7} \\
    \ref{ex_tcorder3_v2} &   $10 \times 10 \times 10$ with bond $(4-4-4)$ & extension of Example~\ref{ex_7x7x7} & 100\\
     & $15 \times 15 \times 15$ with bond $(6-6-6)$ & & 100\\
     & $20 \times 20 \times 20$ with bond $(8-8-8)$&& 100\\
     & $25 \times 25 \times 25$ with bond $(10-10-10)$&& 100\\
     & $30 \times 30 \times 30$ with bond $(12-12-12)$&& 100\\
     & $35 \times 35 \times 35$ with bond $(14-14-14)$&& 100\\
    \ref{ex_7x7x7x7_rank5} &  $7 \times 7 \times 7 \times 7$ with bond $(5-5-5-5)$  & \minitab{higher order tensors with \\bonds exceeding tensor dimensions} & 100 \\
    \ref{ex_27x27x27x27_rank25} &    $27 \times 27 \times 27 \times 27$ with bond ($5-5-5-5$)   & 
    \minitab{higher order tensors with \\ collinear components} & 50 \\
    \ref{ex_tcorder4_} &  \minitab{$I \times I \times I \times I$  with bond $(I-I-I-I)$\\
    where $I = 10, 15$} & higher order tensors & 200 \\
    \ref{ex_tcorder5_7} & \minitab{$7 \times 7 \times 7 \times 7 \times 7$ \\ with bond $13-13-\cdots-13$}  & order-5 tensors &  100\\
    & \minitab{$3 \times 3 \times 3 \times 3 \times 3 \times 3 \times 3$ \\ with bond $8-8-\cdots-8$}  & order-7 tensors& 100\\
    \hline
    \multicolumn{2}{l}{For images approximation} \\
    \ref{ex_image_tc} &  $128 \times 128 \times 3$ with bond $(R_1-R_2-R_1)$ & Six images of size $128 \times 128 \times 3$ & 2331  
    \\
    \hline
        \multicolumn{2}{l}{For compression of CNNs} \\
    \ref{ex_resnet92} & \minitab{Bonds are constrained so that the model size \\
    is close to one in CPD with rank-200} & \minitab{Single layer compression in ResNet-18 \\ trained on ILSVRC-12} \\
    \ref{ex_resnet18_cifar10_v2} & 
    \minitab{Kernels in layers 2-11 with bond $10-10-10$ \\
    Other kernels as in Example \ref{ex_resnet92}} & \minitab{Single layer compression in ResNet-18 \\ finetuned on CIFAR-10} \\
    \ref{ex_rsn18_chkpoint95} & Rank selection described in Example~\ref{sec::tclayer} & \minitab{Single layer compression in ResNet-18 \\ trained on CIFAR-10} \\
    \ref{ex_rsn18_chkpoint95_full} & As in Example~\ref{ex_rsn18_chkpoint95} + CPD as in \cite{PhanECCV2020} & \minitab{Full network compression of ResNet-18 \\ trained on CIFAR-10}\\
    \ref{ex::resnet18_ILSVRC}  & As in Example~\ref{ex_rsn18_chkpoint95} + CPD as in \cite{PhanECCV2020} & \minitab{Full network compression of ResNet-18 \\ and VGG-16 trained on ILSVRC-12}\\
    \ref{ex_resnet92_ext} &  & extended from Example~\ref{ex_resnet92} \\
    \ref{ex_resnet18_cifar10_v2_ext} & & extended from Examaple~\ref{ex_resnet18_cifar10_v2}
    \end{tabular}
    \label{tab:my_label}
\end{table}

\begin{figure}[t]
\centering
\subfigure[Convergence of ALS and SSC]{\includegraphics[width=.31\linewidth, trim = 0.0cm 0cm 0cm 0cm,clip=true]{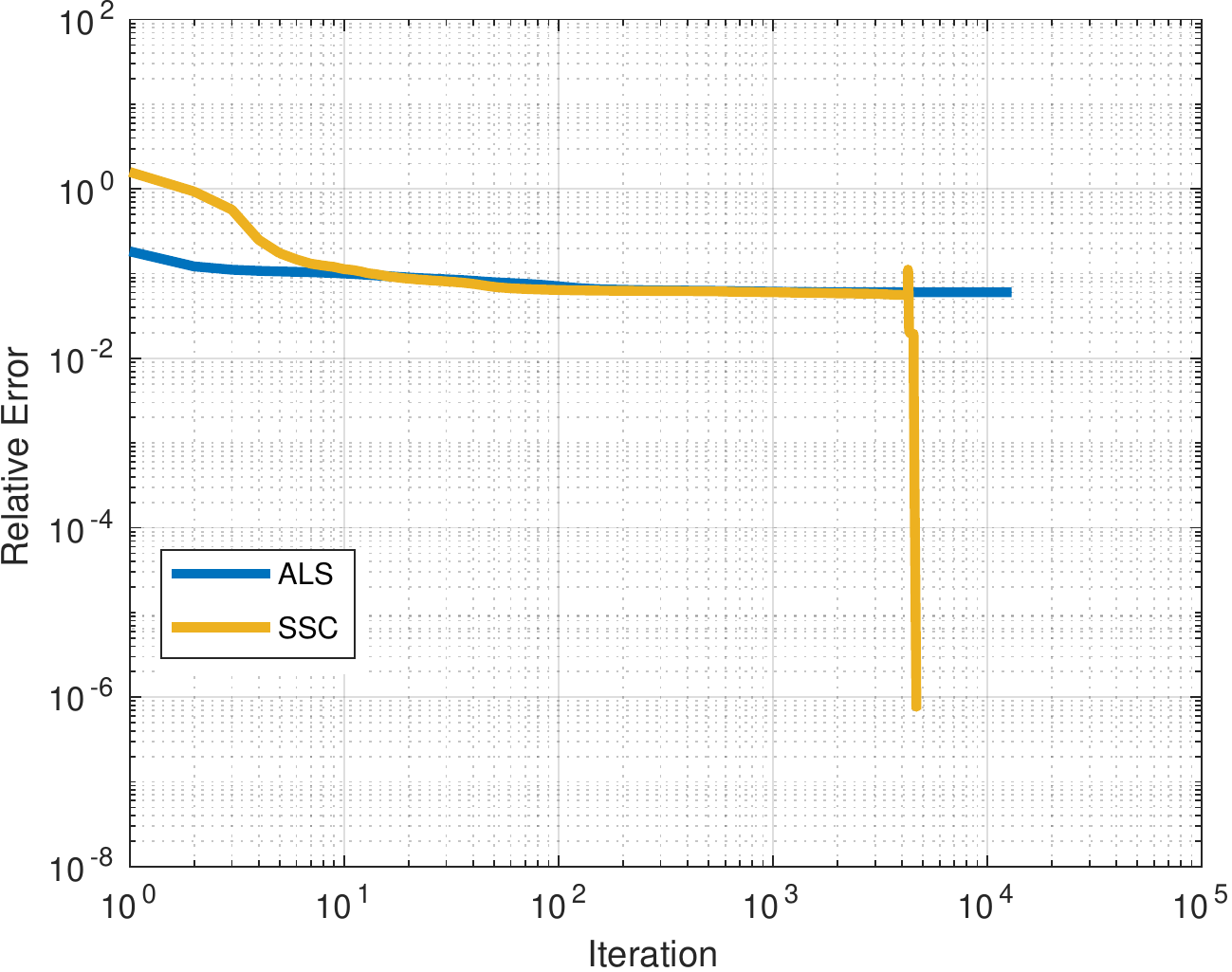}}
\subfigure[Success rate of ALS and SSC]{\includegraphics[width=.31\linewidth, trim = 0.0cm 0cm 0cm 0cm,clip=true]{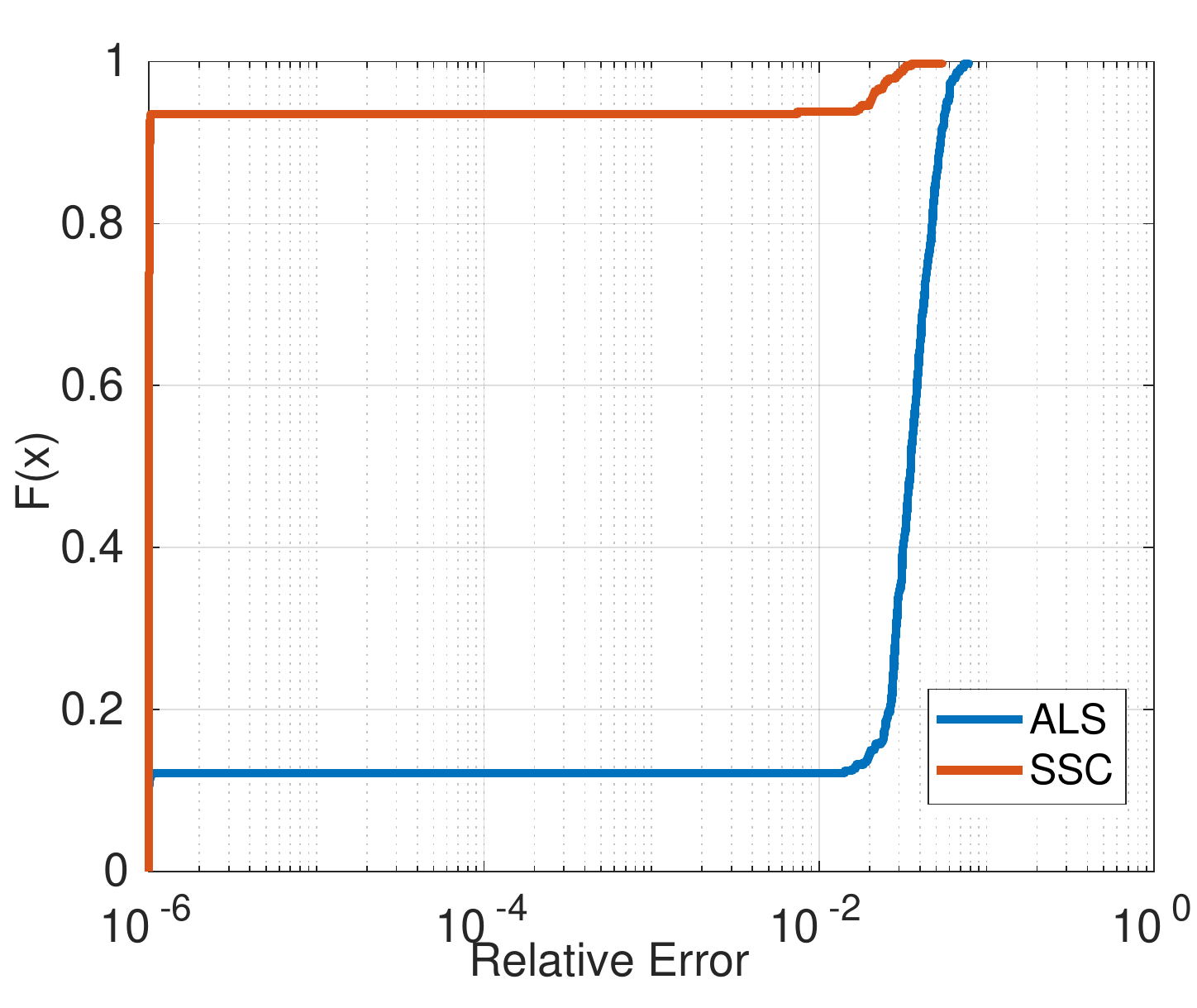}}
\subfigure[Sensitivity vs Relative error]{\includegraphics[width=.31\linewidth, trim = 0.0cm 0cm 0cm 0cm,clip=true]{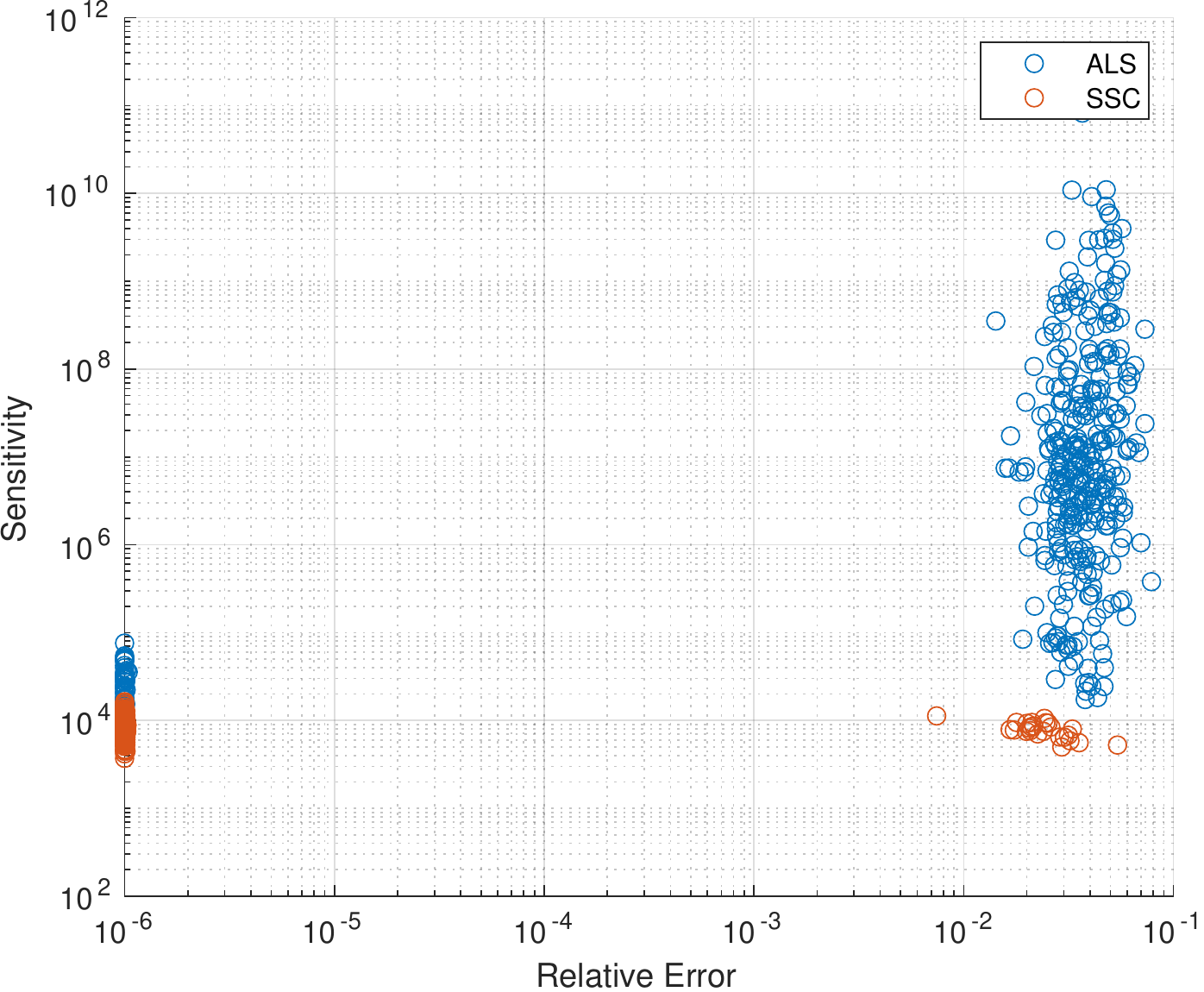}}
\caption{Performance comparison of ALS and SSC for TC decomposition in Example~\ref{ex::A1}. (a) the decomposition with SSC after 3000 ALS updates converges quickly to the exact model. (b) empirical CDF of relative errors as measure of success rate of the decomposition. SSC attains a success rate of 92\% in 6000 iterations, while ALS succeeds in less than 12\% with 13000 iterations. Note that the success rate of ALS in 5000 iterations is only 3\%.}
\label{fig_ex1_2}
\end{figure}

\subsection{Examples for decomposition of synthetic tensors under difficult scenarios}\label{sec::synexample}

\newcounter{appendixexample}
\renewcommand\theappendixexample{A\the\value{appendixexample}}
\renewenvironment{example}
{\refstepcounter{appendixexample}\vspace{10pt}\par\noindent\textbf{Example \theappendixexample\ }}{}

\begin{example}[Decomposition of tensor of size $7 \times 7 \times 7$ with bond $(3-3-3)$.] \label{ex::A1}

This is an extension of Example~\ref{ex_7x7x7} in the main manuscript. 
 For the same tensors as in Example~\ref{ex_7x7x7}, we applied sensitivity correction after 3000 ALS updates, then continued the decomposition, the ALS+SSC converged quickly to the exact solution.%
 Figure~\ref{fig_ex1_2}(a) compares convergences between ALS and SSC in one decomposition of the tensor. 
Similar convergence behavior can be observed in many other decompositions.
SSC improves the convergence of the TC decomposition and gives a success rate of 92\% as shown in  Figure~\ref{fig_ex1_2}(b). A decomposition is a success if it achieves a relative approximation error smaller than $10^{-6}$
\be
\frac{\|\tY - \hat{\tY}\|_F^2}{\|\tY\|_F^2} \le 10^{-6}\,.
\ee
We use the empirical cdf of the relative approximation error as a measure of the success rate. There are 10000 decompositions for 100 tensors; each tensor is decomposed 100 times with different initial points.

Figure~\ref{fig_ex1_2}(c) shows scatter plots of the sensitivity measures and relative approximation errors of estimated tensors. The results obtained by ALS often have very high sensitivity, making the algorithm hard to explain the data fully.

\end{example}

\begin{figure}[!t]
\centering
\subfigure[Tensor dimensions $I = 10$ with bond $R = 4$]{\includegraphics[width=.32\linewidth, trim = .0cm 0cm 0cm 0cm,clip=true]{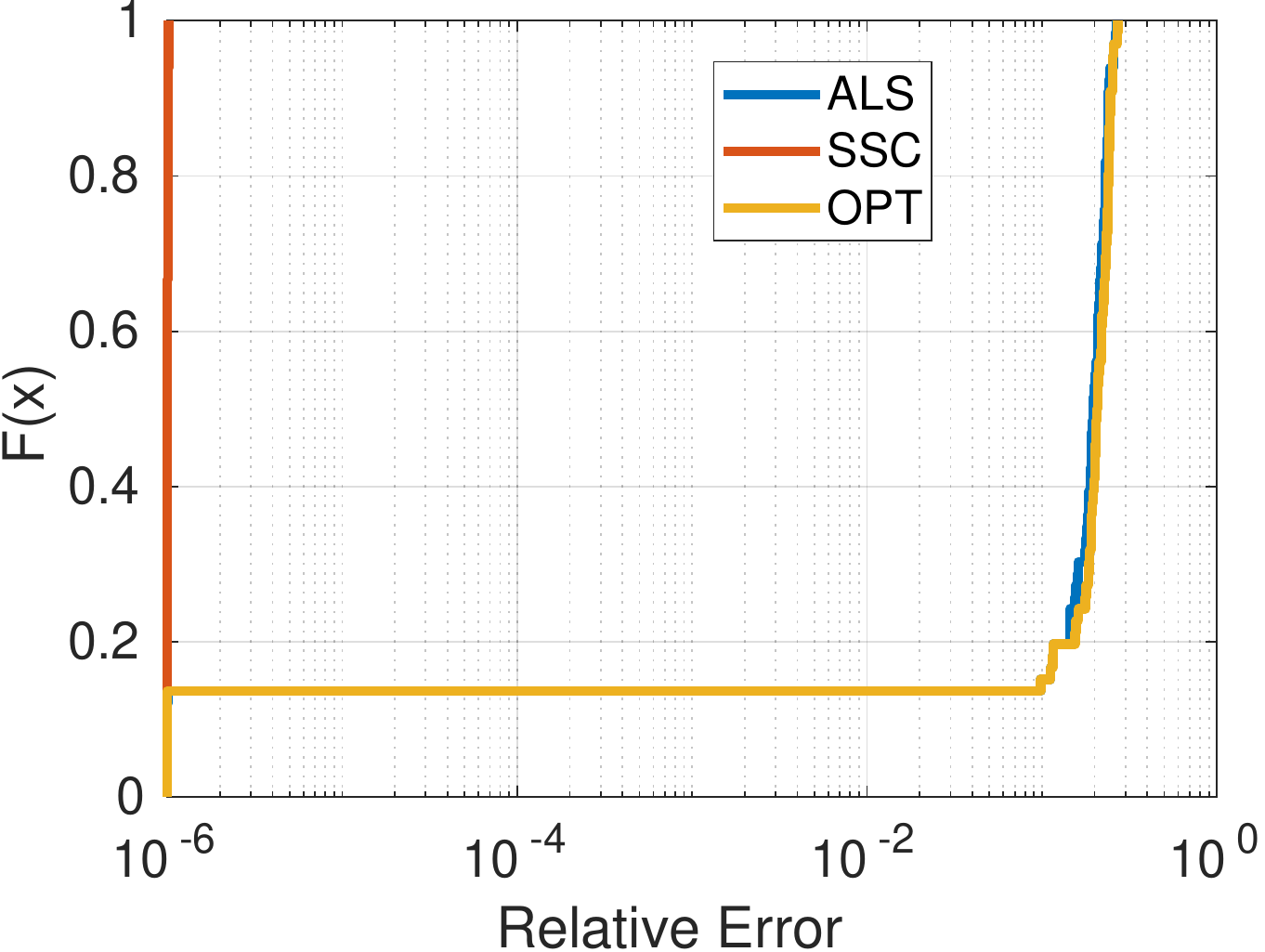}}  
 \subfigure[Tensor dimensions $I = 15$ with bond $R = 6$]{\includegraphics[width=.32\linewidth, trim = 0.0cm 0cm 0cm 0cm,clip=true]{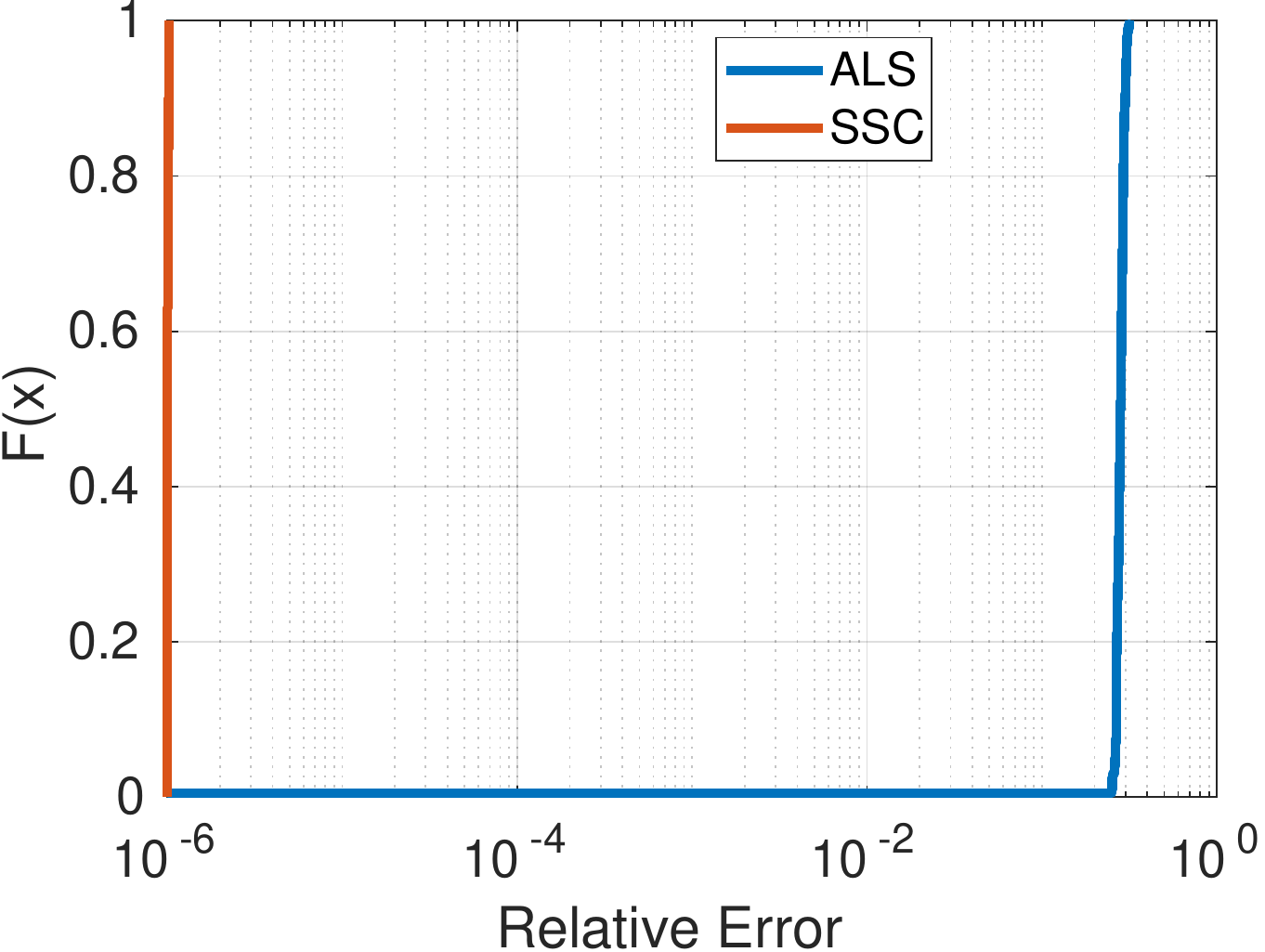}}
 \subfigure[Tensor dimensions $I = 20$ with bond $R = 8$]{\includegraphics[width=.32\linewidth, trim = 0.0cm 0cm 0cm 0cm,clip=true]{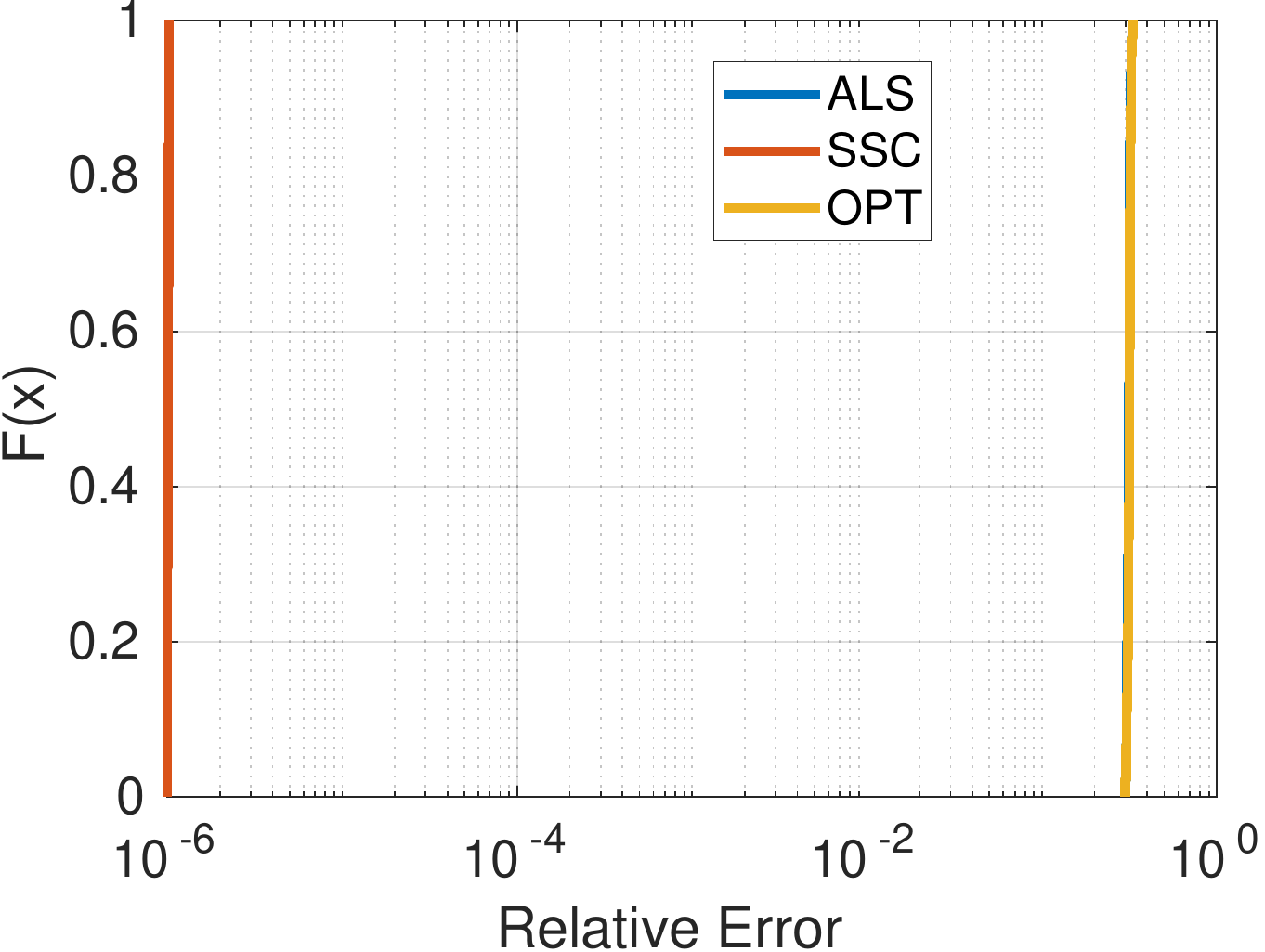}}
 \subfigure[Tensor dimensions $I = 25$ with bond $R = 10$]{\includegraphics[width=.32\linewidth, trim = 0cm 0cm 0cm 0cm,clip=true]{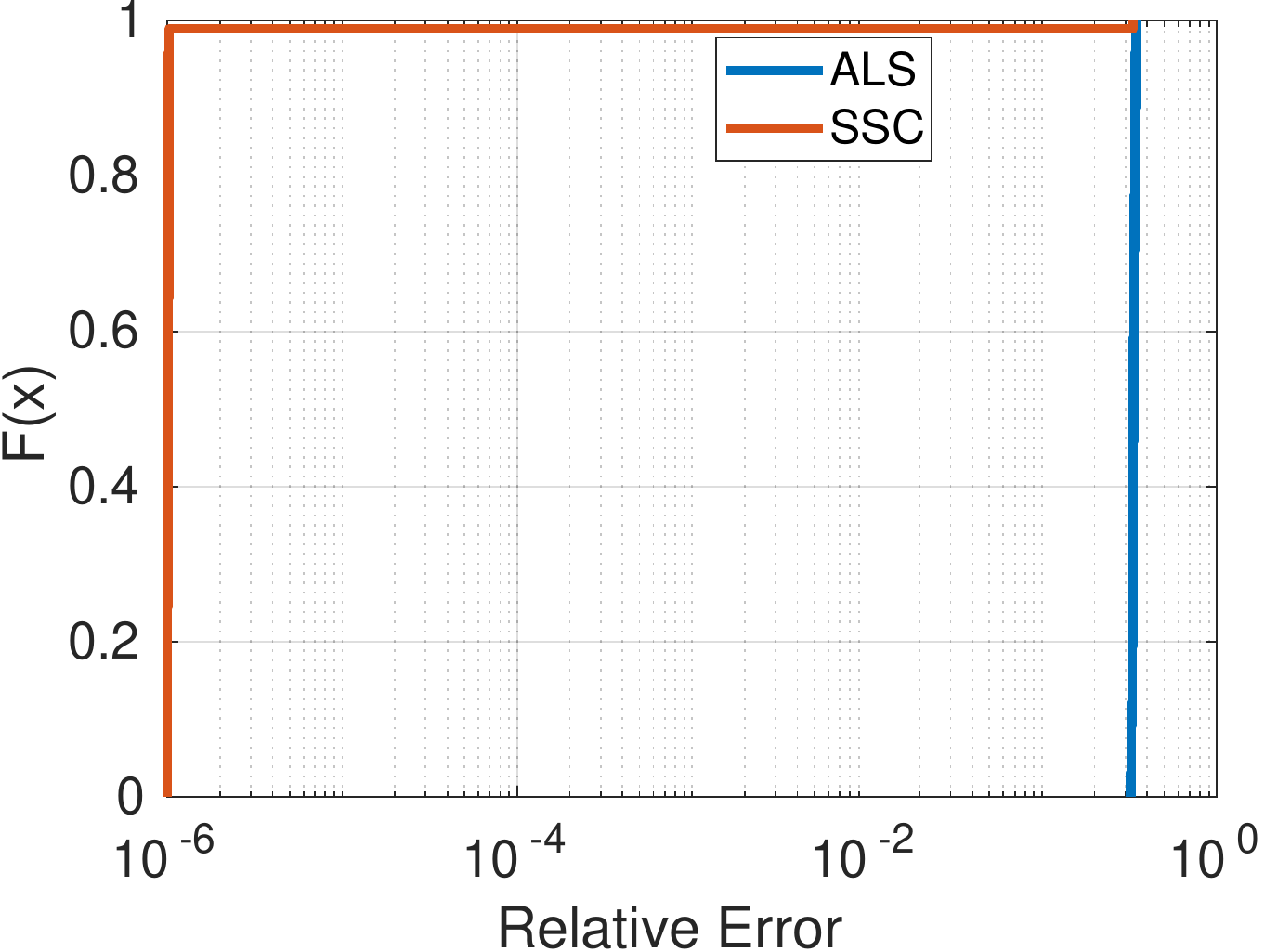}}
\subfigure[Tensor dimensions $I = 30$ with bond $R = 12$]{\includegraphics[width=.32\linewidth, trim = 0.0cm 0cm 0cm 0cm,clip=true]{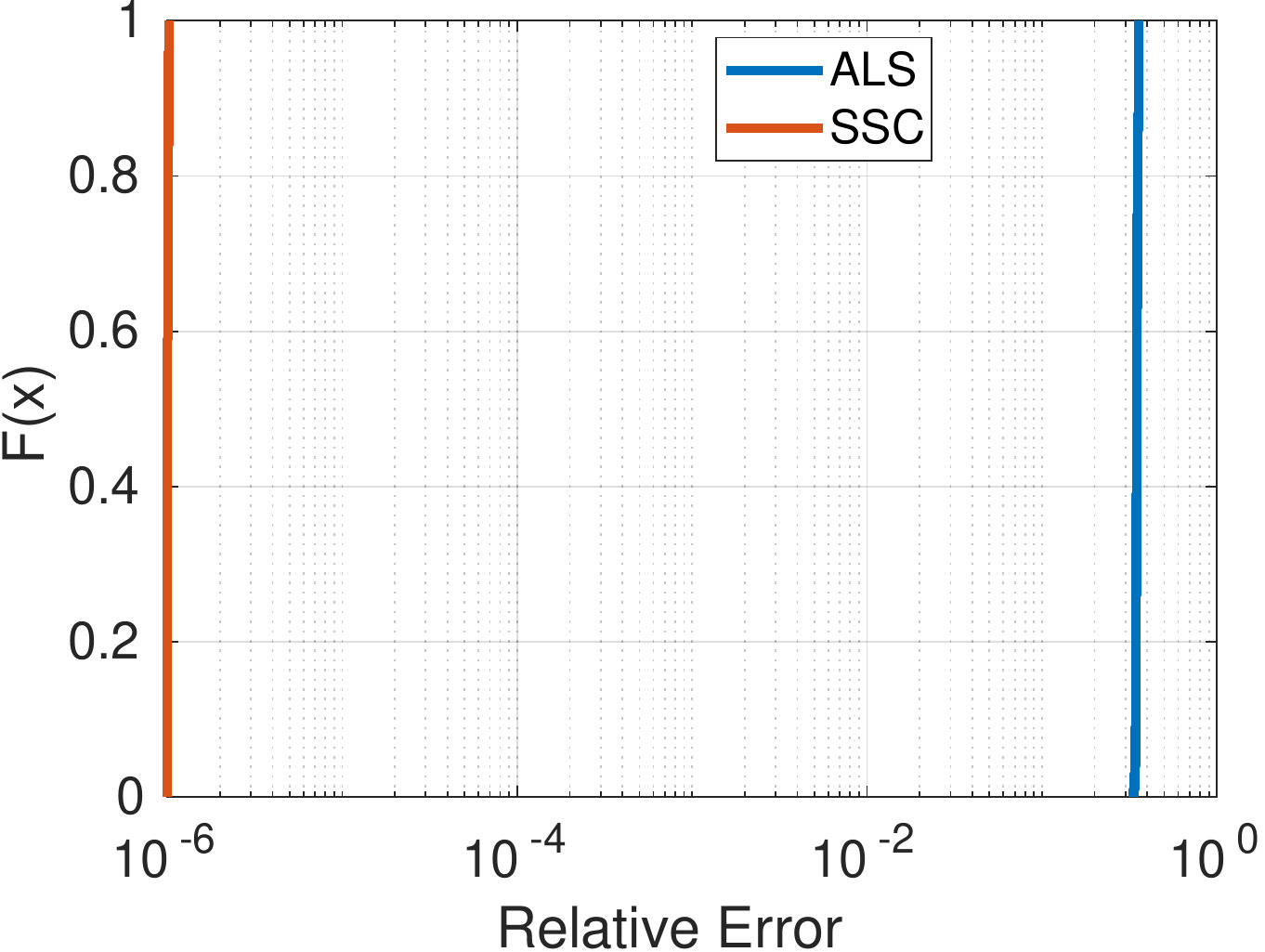}}
\subfigure[Tensor dimensions $I = 35$ with bond $R = 14$]{\includegraphics[width=.32\linewidth, trim = 0.0cm 0cm 0cm 0cm,clip=true]{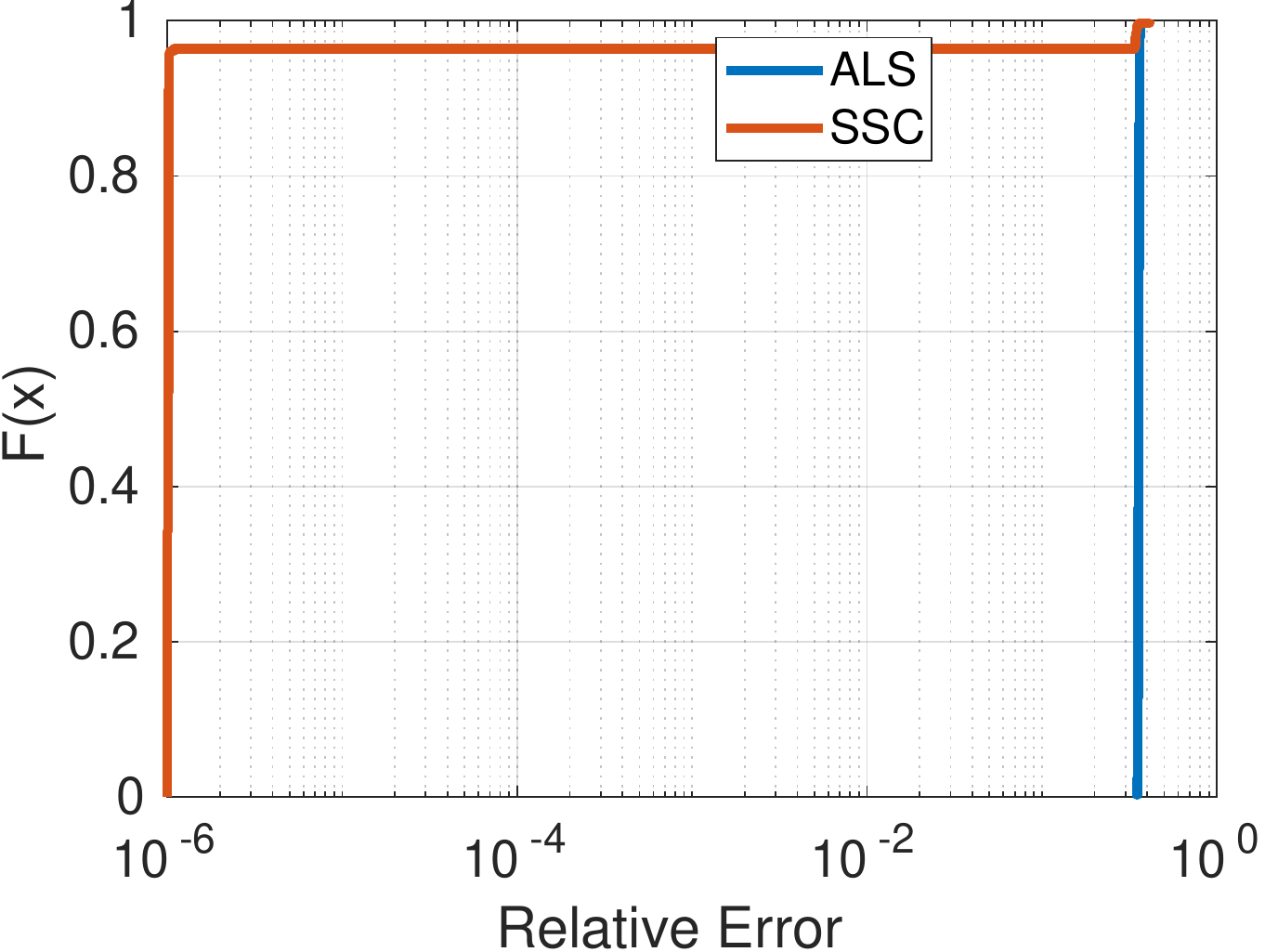}}
  \caption{Success rates for TC decomposition of tensors of size $I \times I \times I$ with bond dimensions $(R-R-R)$ in Example~\ref{ex_tcorder3_v2}. ALS and OPT failed in our example. We can decompose the tensors with very high success rates using the proposed algorithm for sensitivity correction and sensitivity control.}\label{fig_tcodrder3_v2}
\end{figure}

\begin{example}[More TC decomposition for tensors of order-3]\label{ex_tcorder3_v2}

Similar to Example~\ref{ex::A1}, we decomposed order-3 tensors of dimensions $I \times I \times I$ and bond dimensions $(R-R-R)$. We applied ALS and OPT algorithms in the examples and ran the two algorithms within 6000 iterations. SSC was used to correct the sensitivity of the estimated tensors after 1000 ALS updates.
Figure~\ref{fig_tcodrder3_v2} reports success rates of the considered algorithms. 
For tensors of size $10 \times 10 \times 10$ with bond dimensions $(4-4-4)$, ALS and OPT have low success rates, less than 14\% over 100 decompositions.
For bigger tensors with larger bond dimensions, the two algorithms completely fail. Since OPTs are more expensive than ALS and its performances are not much different from ALS, we provide simulation results for OPT for tensors with small sizes.

The proposed algorithm, SSC, has significantly improved the performances of the ALS algorithm. In all the tests, we obtained nearly perfect decomposition with relative approximation error less than $10^{-6}$. 
 \end{example}

\begin{example}[TC decompositions for order-4 tensors]\label{ex_7x7x7x7_rank5}

Similar to Example~\ref{ex::A1}, we show TC decomposition for order-4 tensors of size $7 \times 7 \times 7 \times 7$ with bond dimensions (5-5-5-5), i.e., the corresponding factor matrices 
have more columns than rows. Figure~\ref{fig_ex7777}(a) shows ALS and OPT\cite{TRWOPT,SAND2010-1422} got stuck in false local minima after several hundreds of iterations, whereas intensity (red curve) and sensitivity (blue curve) of the estimated tensors grow to $10^8$ shown in Figure~\ref{fig_ex7777}(b). 
This example confirms the instability problem for TC of higher order. 

We applied the Sensitivity correction method, and suppressed the sensitivity of the estimated tensor from $10^7$ to several hundred, then continued the decomposition from a new tensor with much smaller sensitivity. The correction method helped us achieving a perfect decomposition (see Figure~\ref{fig_ex7777}(a)). 
\end{example}

\begin{figure}\centering
 \subfigure[]{\includegraphics[width=.40\linewidth, trim = 3.0cm 8cm 4cm 8.6cm,clip=true]{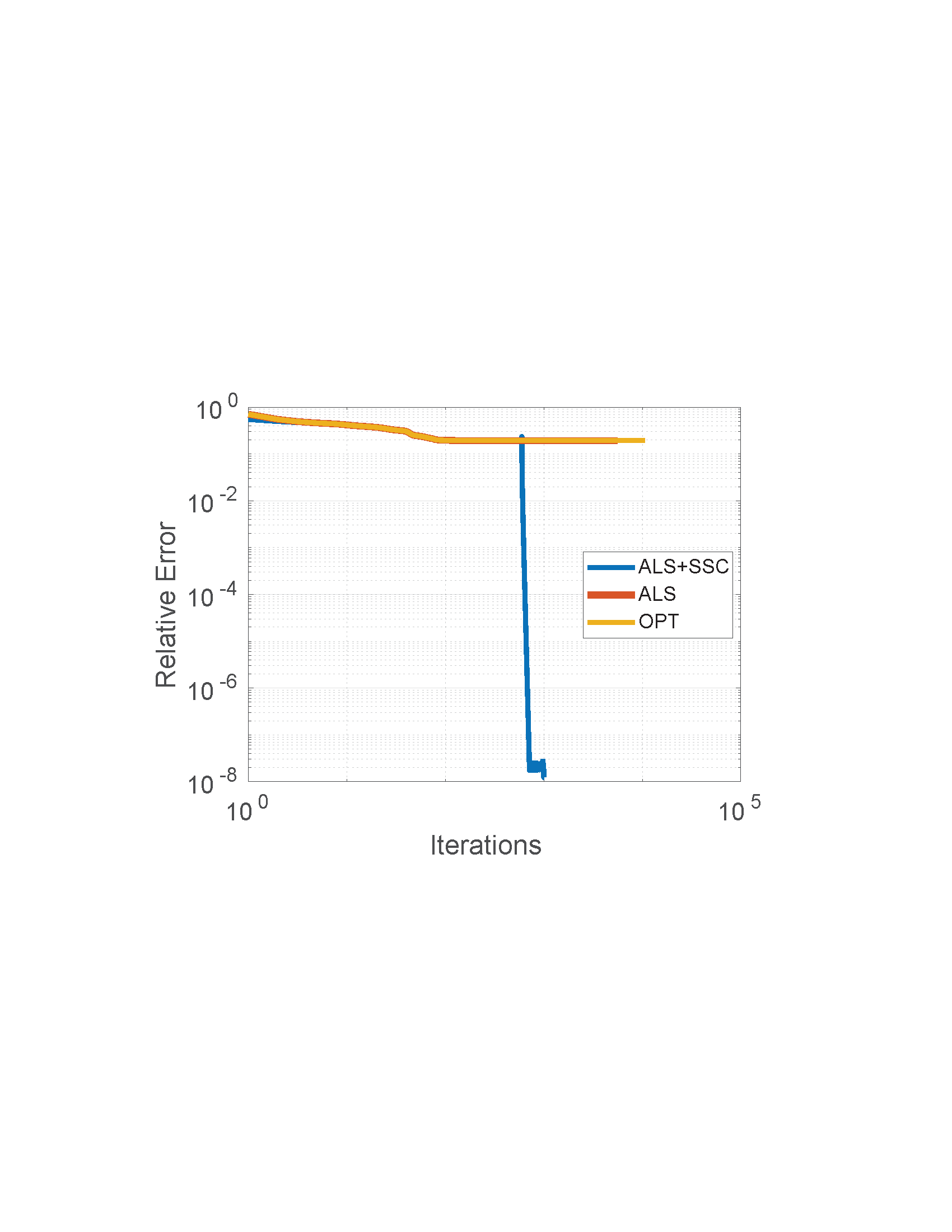}} 
 \subfigure[]{\includegraphics[width=.43\linewidth, trim = 3.0cm 8cm 3cm 8.5cm,clip=true]{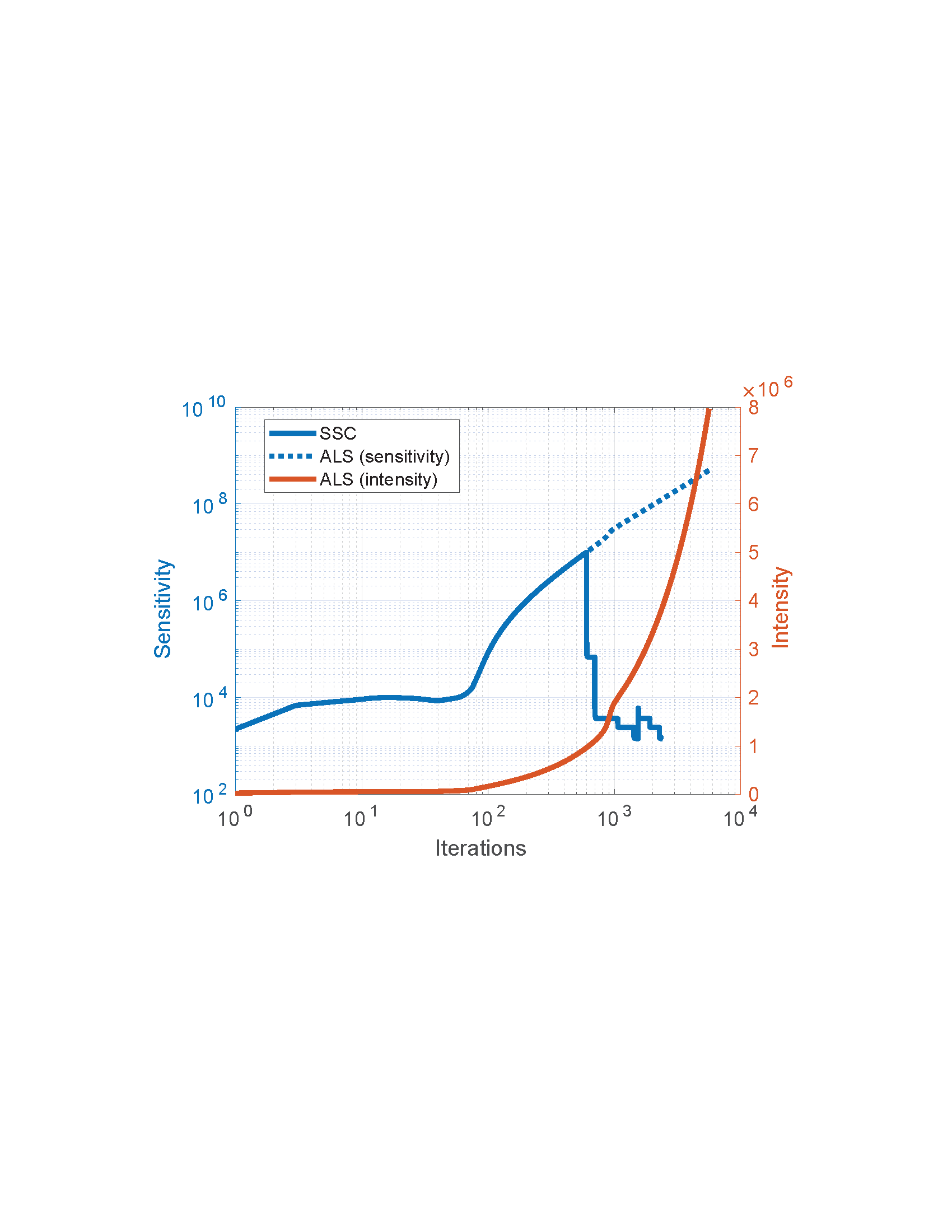}}
 \caption{Illustration for performances obtained by algorithms used in Example~\ref{ex_7x7x7x7_rank5}. 
 Tensors are of size $7 \times 7 \times 7 \times 7$ with bond $(5-5-5-5)$.
 After SS correction, we obtain a new estimated tensor with low sensitivity (see blue curve in (b)). TC decomposition (ALS+SS) quickly converged after the sensitivity correction (blue curve in (a)). }\label{fig_ex7777}
\end{figure}

\begin{figure}\centering
\subfigure[]{
 \includegraphics[width=.42\linewidth, trim = 3.0cm 8cm 4cm 8.6cm,clip=true]{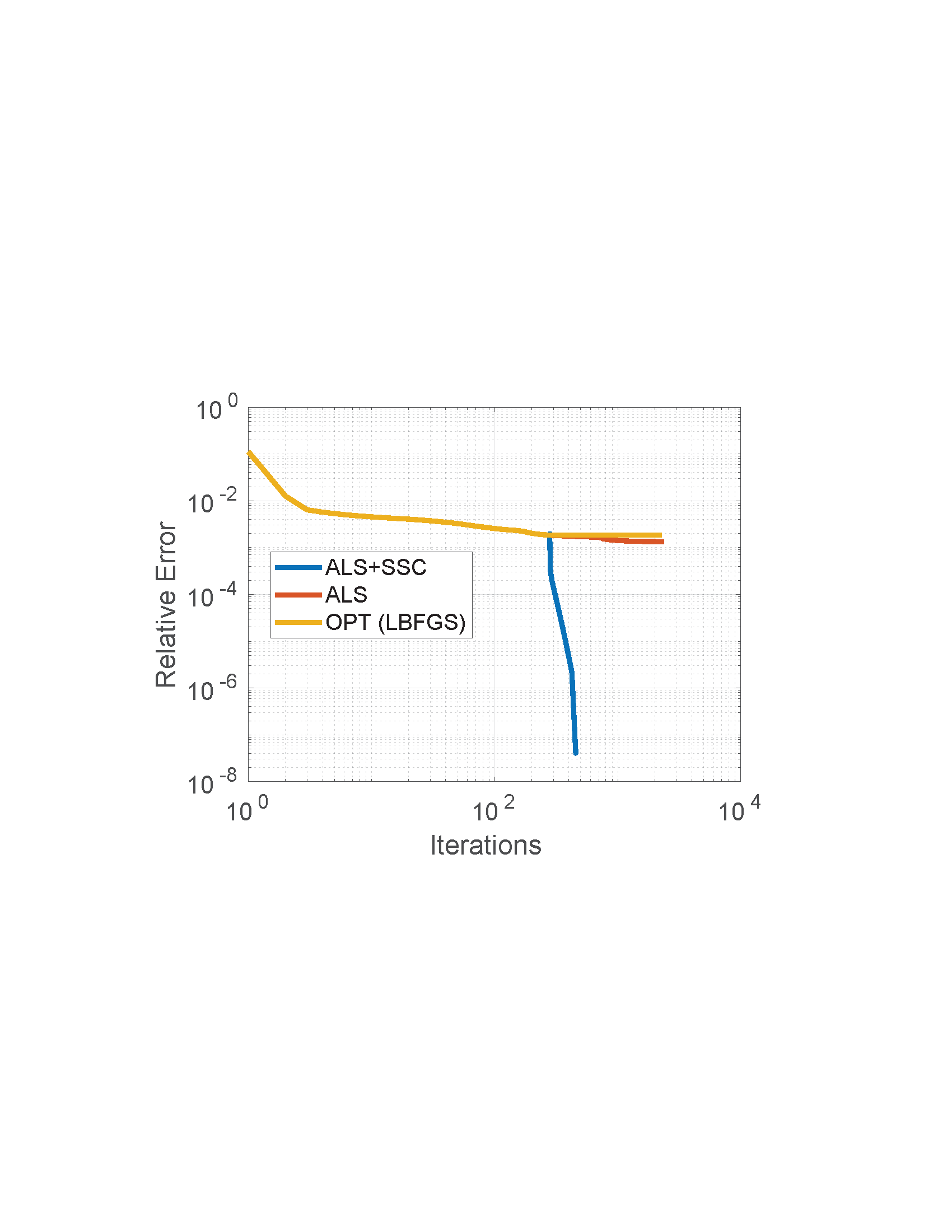}\label{fig_tc_collinear_N4_R25_err}} 
 \subfigure[]{\includegraphics[width=.45\linewidth, trim = 0cm 0cm 0cm 0cm,clip=true]{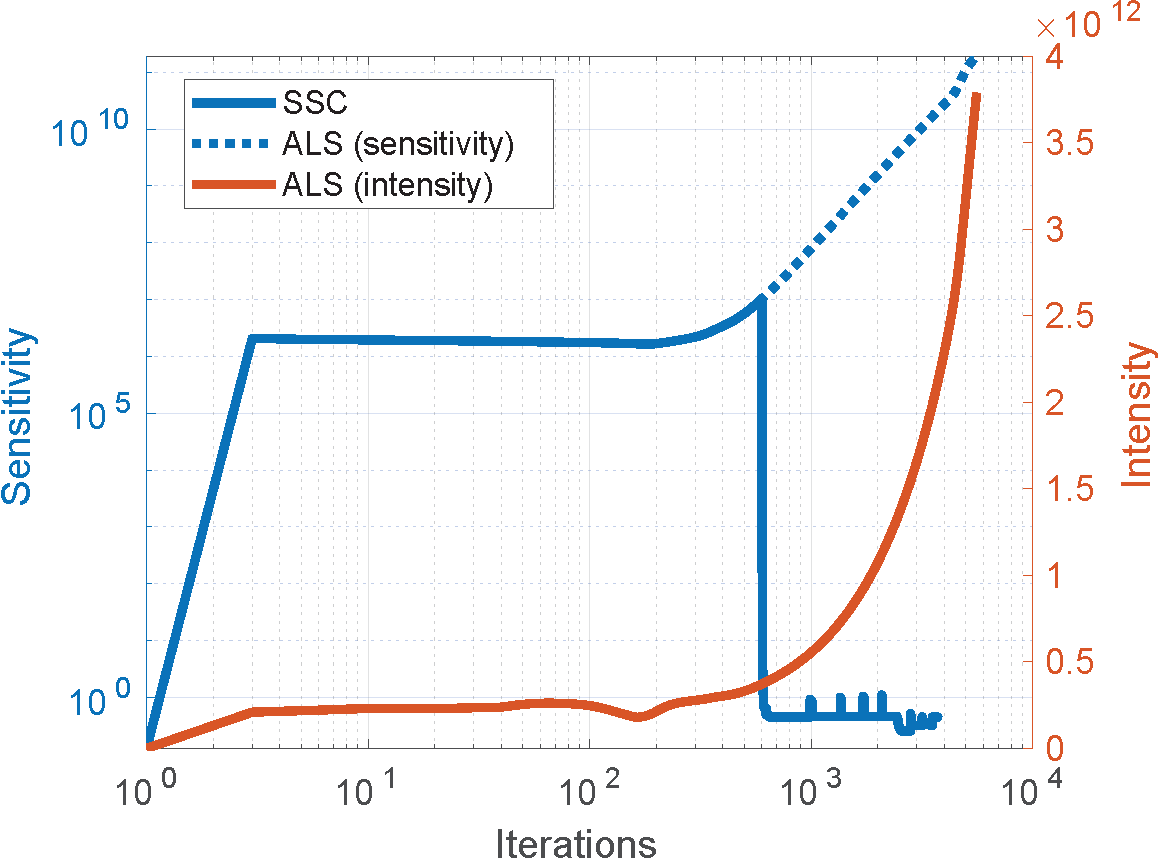}\label{fig_tc_collinear_N4_R25_ss}}
 \caption{Performances obtained by algorithms for decomposition of tensors of size $27 \times 27 \times 27 \times 27$ in Example~\ref{ex_27x27x27x27_rank25}.
 ALS+SS is the ALS plus sensitivity correction. Intensity of the estimated tensors using ALS exceeds $3.5 \, 10^{12}$, and the sensitivity of the tensor is greater than $2 \, 10^{10}$, while the algorithm converges slowly after 1000 iterations.
 }\label{fig_ex27272727}
\end{figure}

\begin{example}[TC decompositions for order-4 tensors with highly collinear components]\label{ex_27x27x27x27_rank25}

We demonstrate TC of higher order tensors, $\tY =  \circlellbracket \tU1, \tU_2, \tU_3, \tU_4\circlerrbracket$ which are similar to those in Example~\ref{ex_27x27x27_rank25} in the main text, but of order-4 with size $27 \times 27 \times 27 \times 27$ and bond dimensions (ranks) $(5-5-5-5)$. The factor matrices ${\bU}_n$ comprise three blocks, ${\bU}_n(:,1:9)$, ${\bU}_n(:,10:18)$ and ${\bU}_n(:,19:25)$, columns in each block are highly collinear. Figure~\ref{fig_tc_collinear_N4_R25_err} shows that the relative approximation errors using ALS are far from zero. Similar to the previous examples, with sensitivity correction, we can correct the estimated tensors with high sensitivity to a new estimation with sensitivity less than 10 (Figure~\ref{fig_tc_collinear_N4_R25_ss}). The decomposition converged quickly after the correction as seen in Figure~\ref{fig_tc_collinear_N4_R25_err}.
\end{example}


\begin{example}[More TC decompositions for order-4 tensors]\label{ex_tcorder4_}

In this example, we decompose order-4 tensors of size $I \times I \times I \times I$ with bond dimensions $(I-I-I-I)$, where $I = 10, 15$. 
The success rates of the three considered algorithms, ALS, OPT, and ALS plus sensitivity correction and control (SSC), are compared in  Figure~\ref{fig_ex_tcorde4_}. Both ALS and OPT failed to decompose a tensor of order-4 with large core tensors. However, with SSC, ALS attained a success rate of 99-100\%.
\end{example}

\begin{figure}\centering
\subfigure[Tensor dimension $I = 10$]{
 \includegraphics[width=.41\linewidth, trim = 0cm 0cm 0cm 0cm,clip=true]{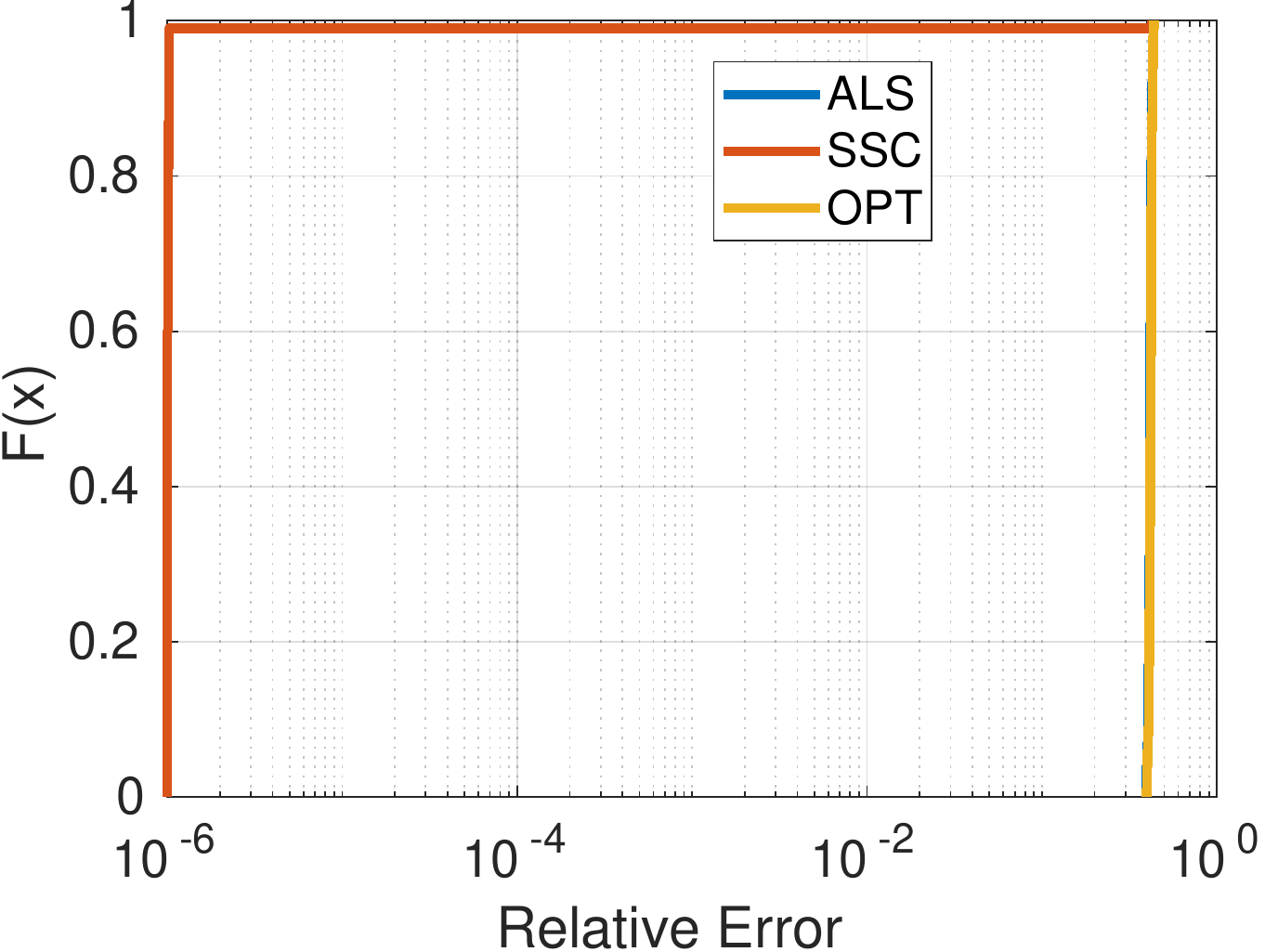}} 
 \hfill
 \subfigure[Tensor dimension $I = 15$]{\includegraphics[width=.41\linewidth, trim = 0cm 0cm 0cm 0cm,clip=true]{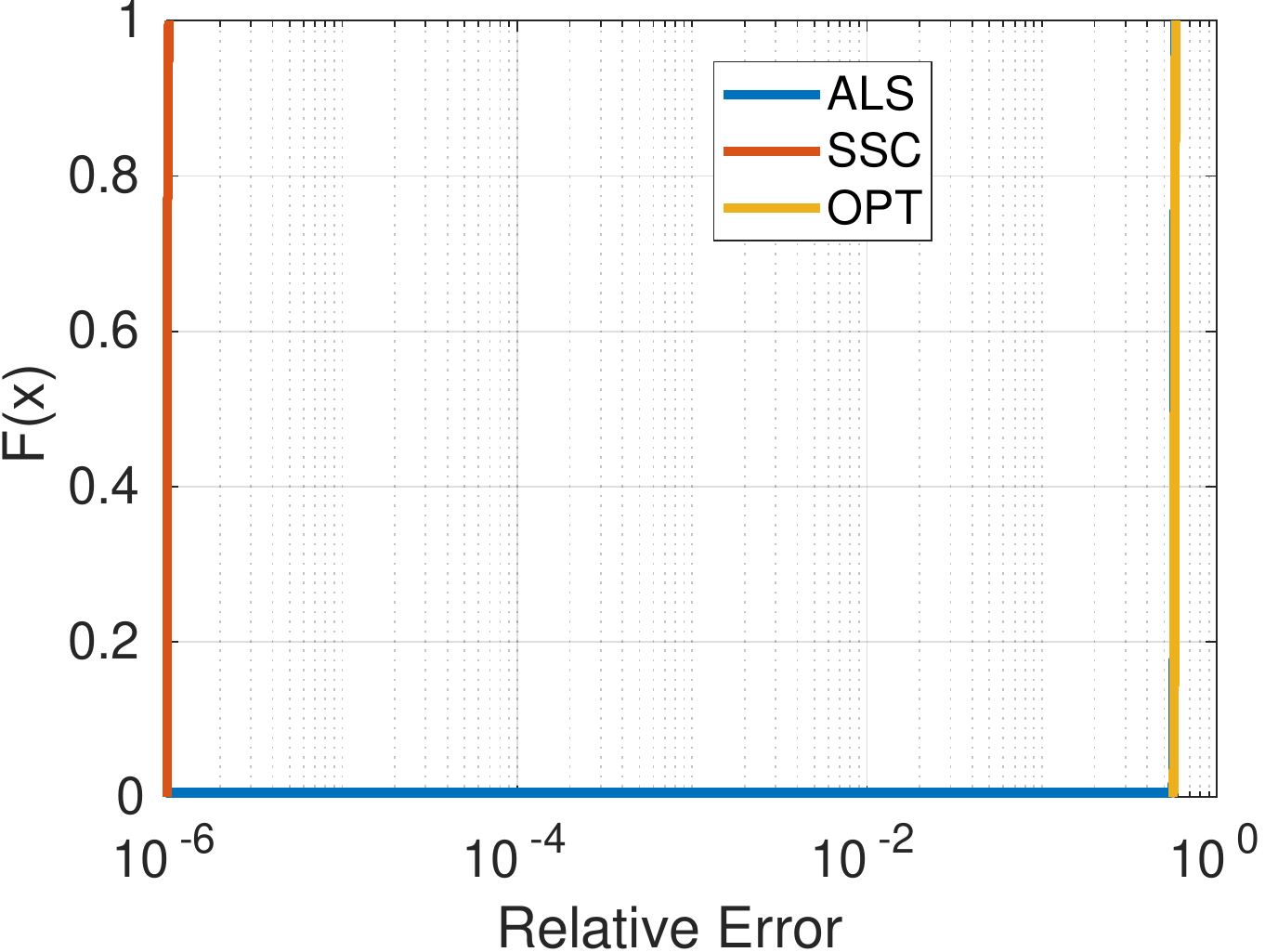}}
 \caption{Success rate of TC algorithms in decomposition of tensors of size $I \times I \times I \times I $ with bond $(I-I-I-I)$ in  Example~\ref{ex_tcorder4_}. Both ALS and OPT fail in decomposition of tensor of order-4 with big core tensors.}\label{fig_ex_tcorde4_}
\end{figure}


\begin{example}[More TC decompositions for higher order tensors]\label{ex_tcorder5_7}

In this example, we decompose order-5 tensors of size $I \times I \cdots \times I$ with bond dimensions $(R-R-\cdots-R)$, where $I = 7$ and $R = 13$, and order-7 tensors with dimensions $I = 3$ and bond $R = 8$.
Convergence of algorithms versus iterations is shown Figure~\ref{fig_tcodrder5} and Figure~\ref{fig_tcodrder7}. Like other TC decomposition for tensors of order-3,4, besides ALS and OPT, we tried several other TC algorithms, but none succeeded.

The success rates of the three considered algorithms, ALS, OPT, and ALS plus sensitivity correction and control (SSC), are compared in  Figure~\ref{fig_tc3_N5_I7R13_cdferr} and Figure~\ref{fig_tc3_N7_I3R8_cdferr}. Both ALS and OPT failed to decompose tensor of order-4 with large core tensors. However, with SSC, ALS attained a success rate of 99-100\%.
\end{example}

\begin{figure}\centering
\subfigure[Convergence of ALS, OPT and ALS+SSC]{
 \includegraphics[width=.31\linewidth, trim = 0.0cm 0cm 0cm 0cm,clip=true]{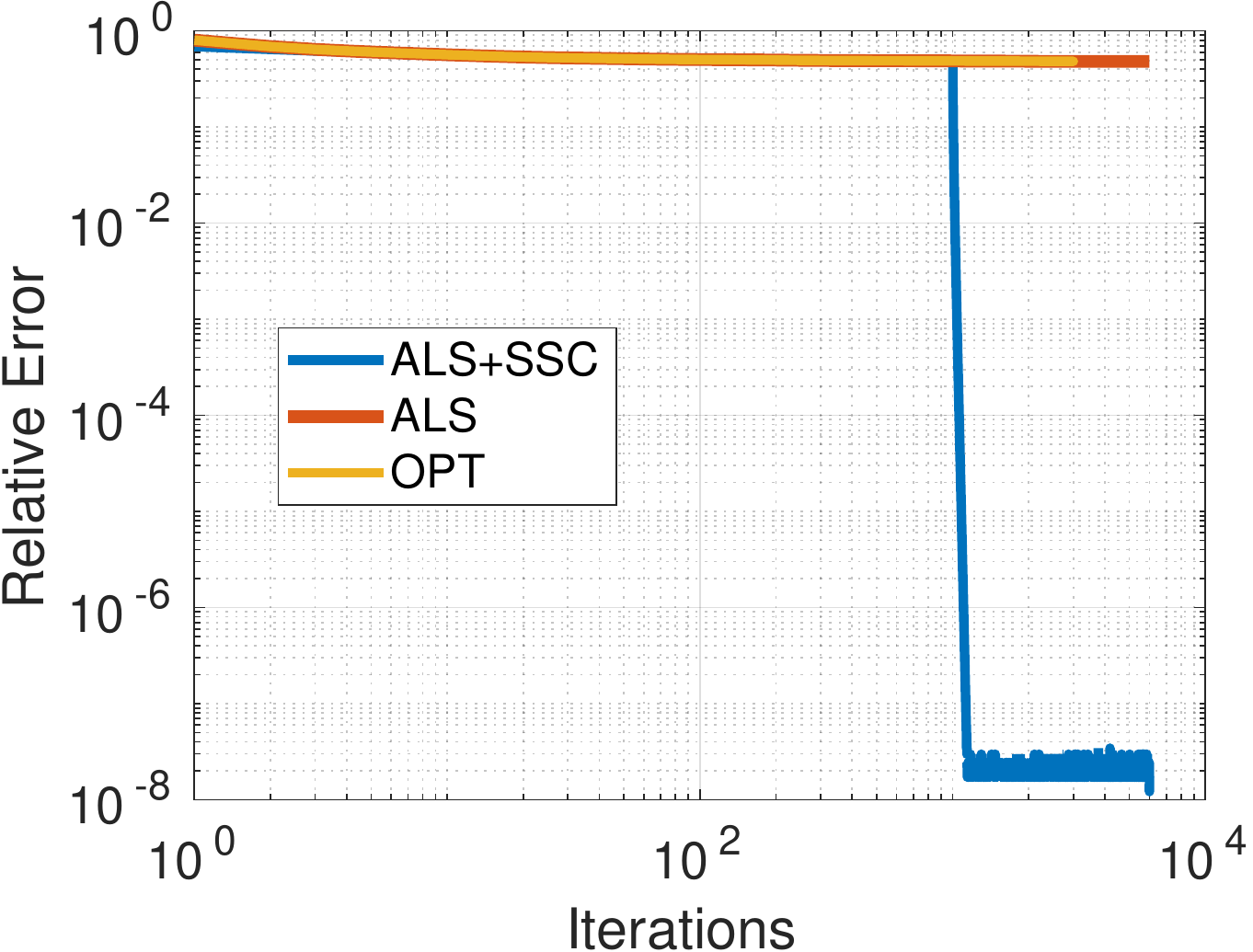}} 
 \hfill
 \subfigure[Sensitivity of ALS]{\includegraphics[width=.33\linewidth, trim = 0.0cm 0cm 0cm 0cm,clip=true]{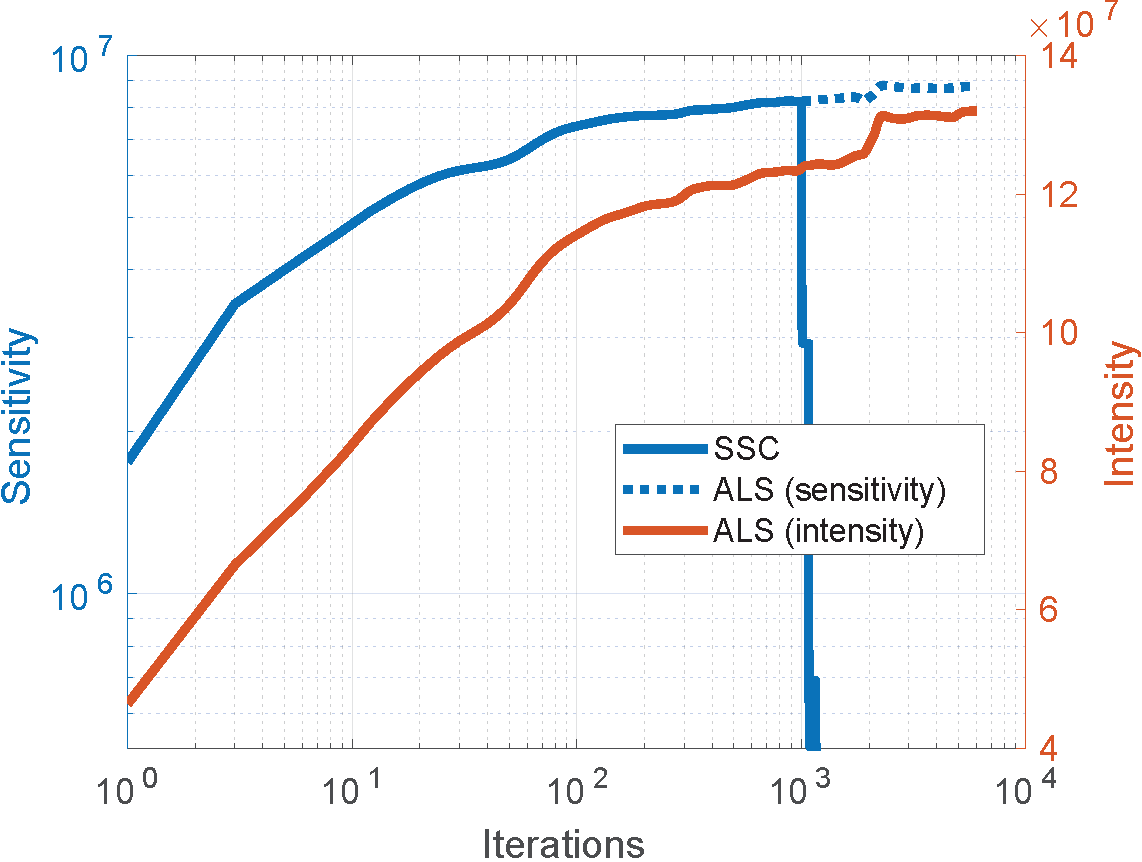}}
 \subfigure[Success rate]{\includegraphics[width=.31\linewidth, trim = 0.0cm 0cm 0cm 0cm,clip=true]{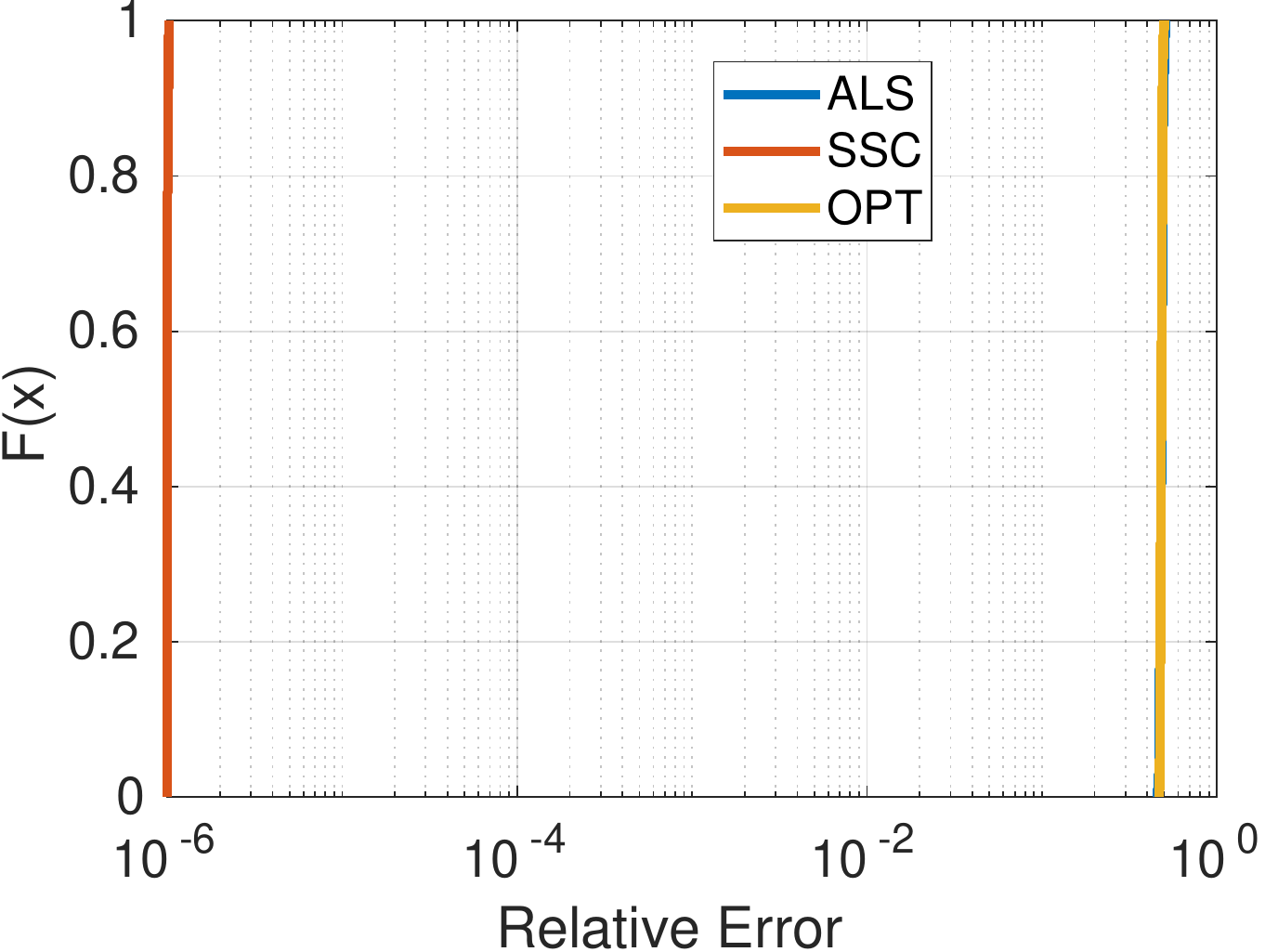}\label{fig_tc3_N5_I7R13_cdferr}}
  \caption{Illustration for performances for decomposition of tensors of order-5 and size $7 \times 7 \times 7 \times 7 \times 7$ in Example~\ref{ex_tcorder5_7}.}\label{fig_tcodrder5}
\end{figure}

\begin{figure}\centering
\subfigure[Convergence of ALS, OPT and ALS+SSC]{
 \includegraphics[width=.31\linewidth, trim = 0cm 0cm 0cm 0cm,clip=true]{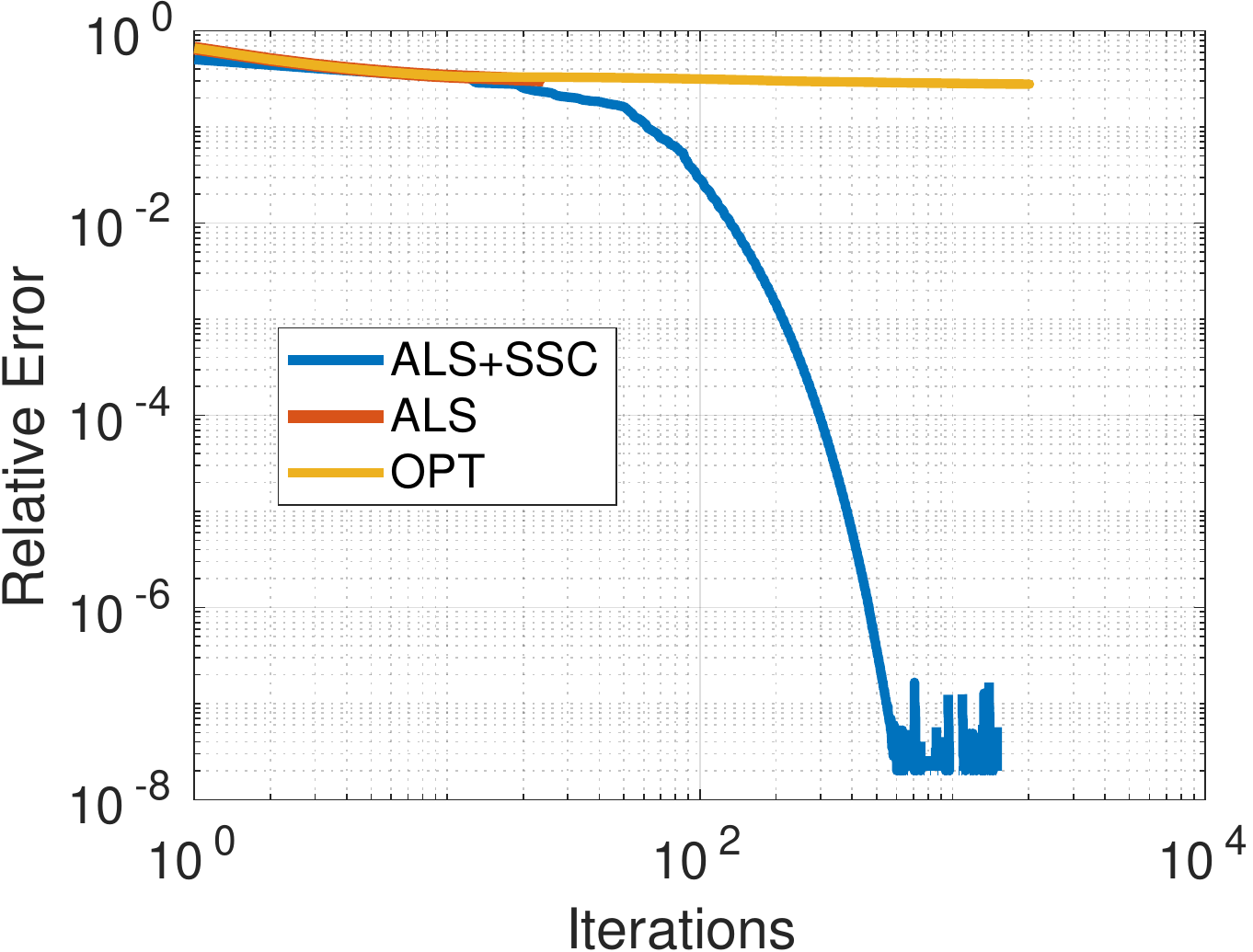}} 
 \hfill
  \subfigure[Sensitivity]{\includegraphics[width=.33\linewidth, trim = 0.0cm 0cm 0cm 0cm,clip=true]{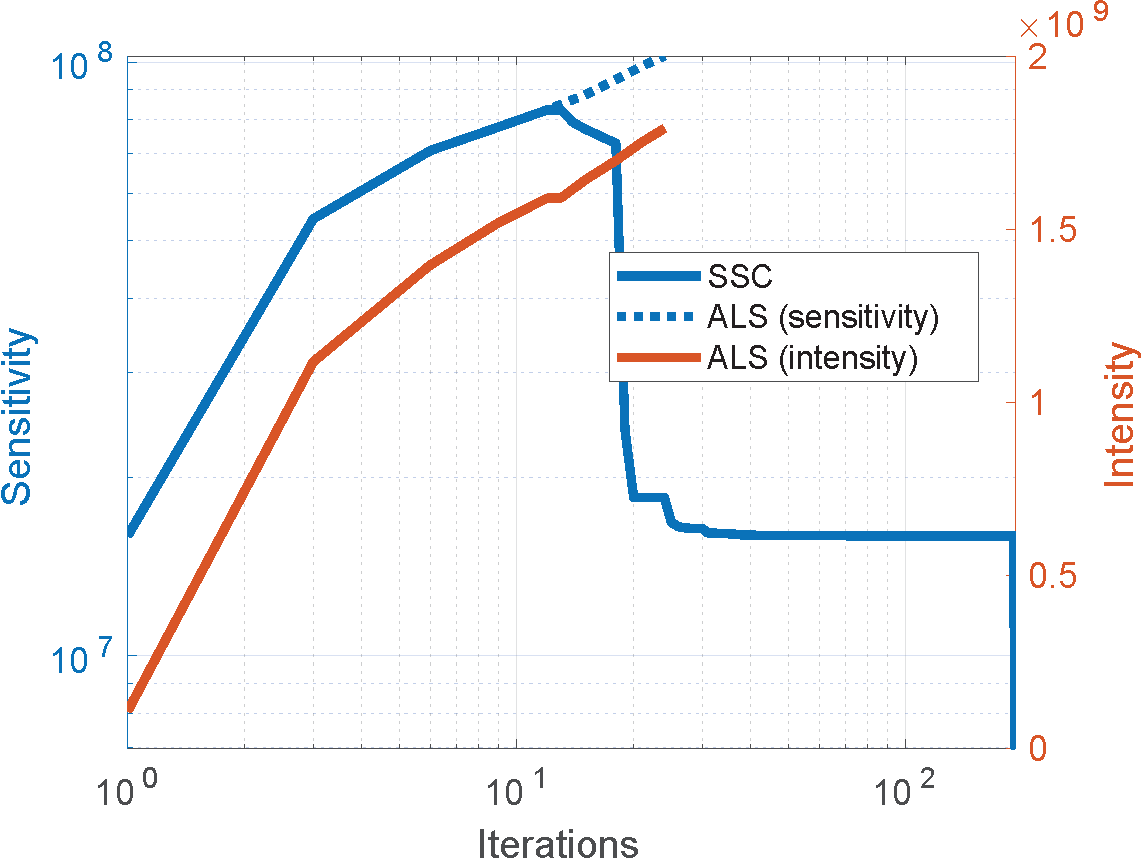}}
 \hfill
 \subfigure[Success Rate]{\includegraphics[width=.33\linewidth, trim = 0.0cm 0.0cm 0cm 0cm,clip=true]{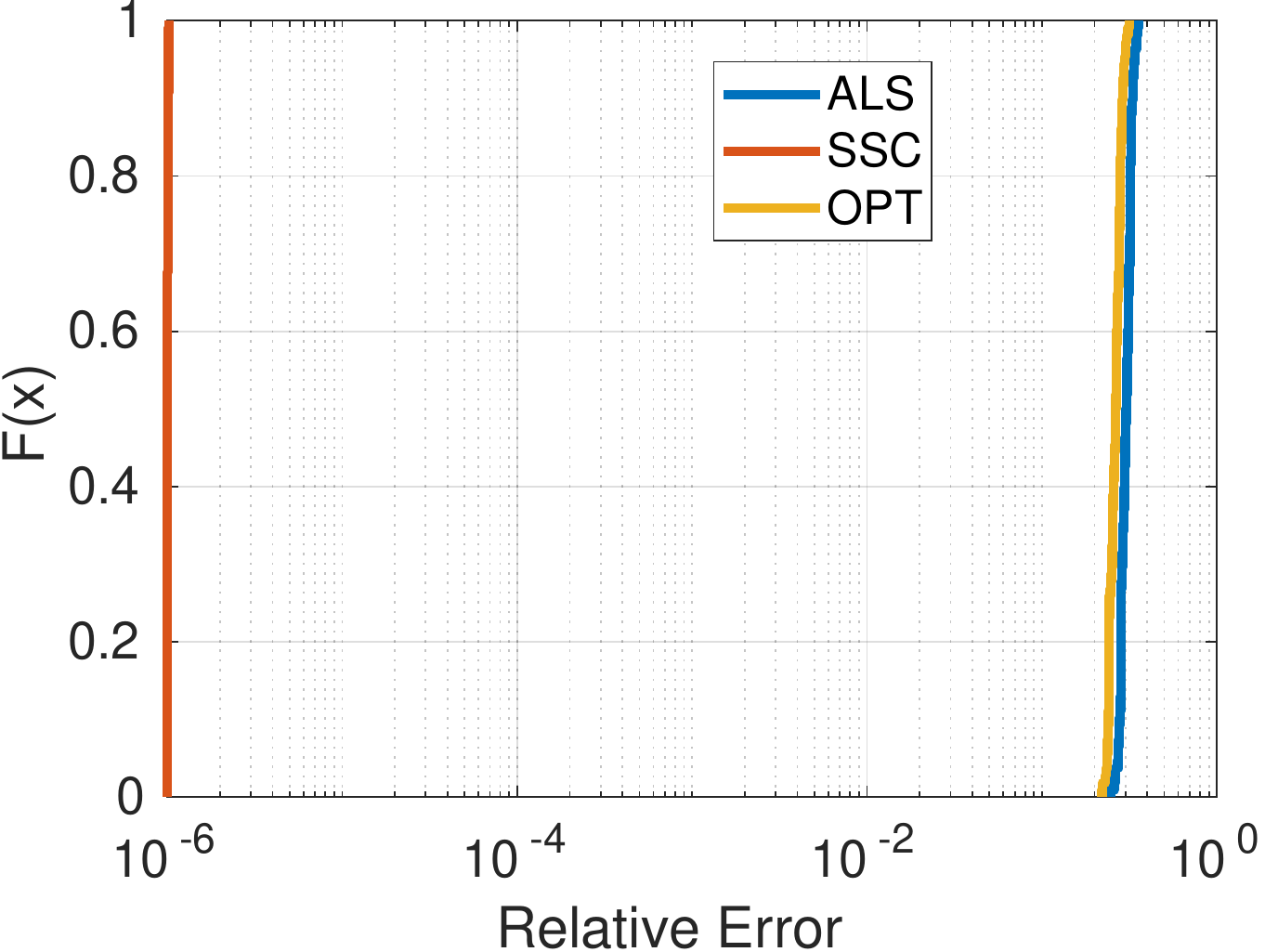}\label{fig_tc3_N7_I3R8_cdferr}}
 \caption{Illustration for performances for decomposition of tensors of order-7 and size $3 \times 3 \times \cdots \times 3$ with bond dimensions $R = 8$ in Example~\ref{ex_tcorder5_7}}.\label{fig_tcodrder7}
\end{figure}


\subsection{Fitting Images by TC-decomposition}

\begin{example}[TC for image approximation]\label{ex_image_tc}

\begin{figure}
\includegraphics[width=.158\linewidth, trim = 0.0cm 0cm 0cm 0cm,clip=true]{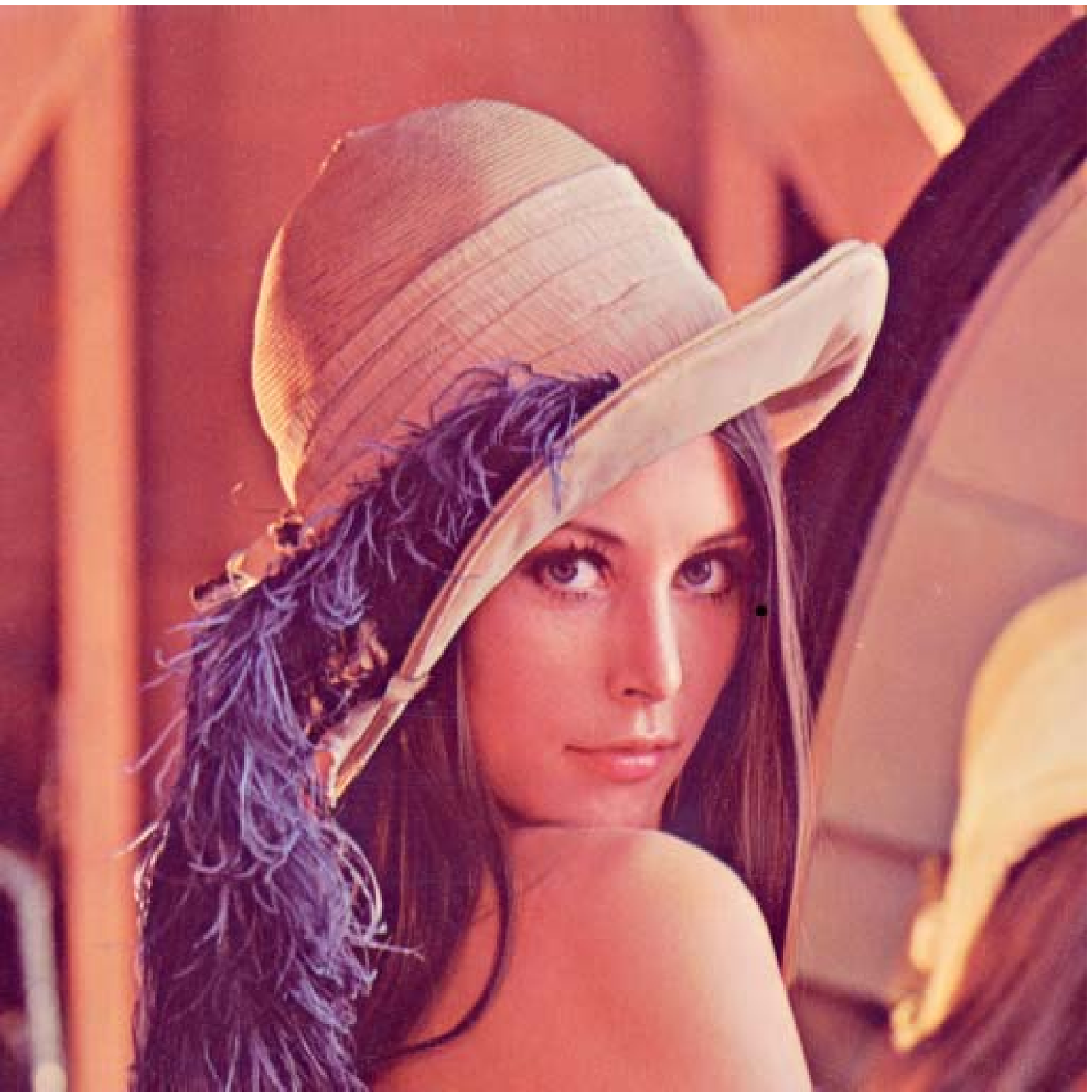}
\includegraphics[width=.158\linewidth, trim = 0.0cm 0cm 0cm 0cm,clip=true]{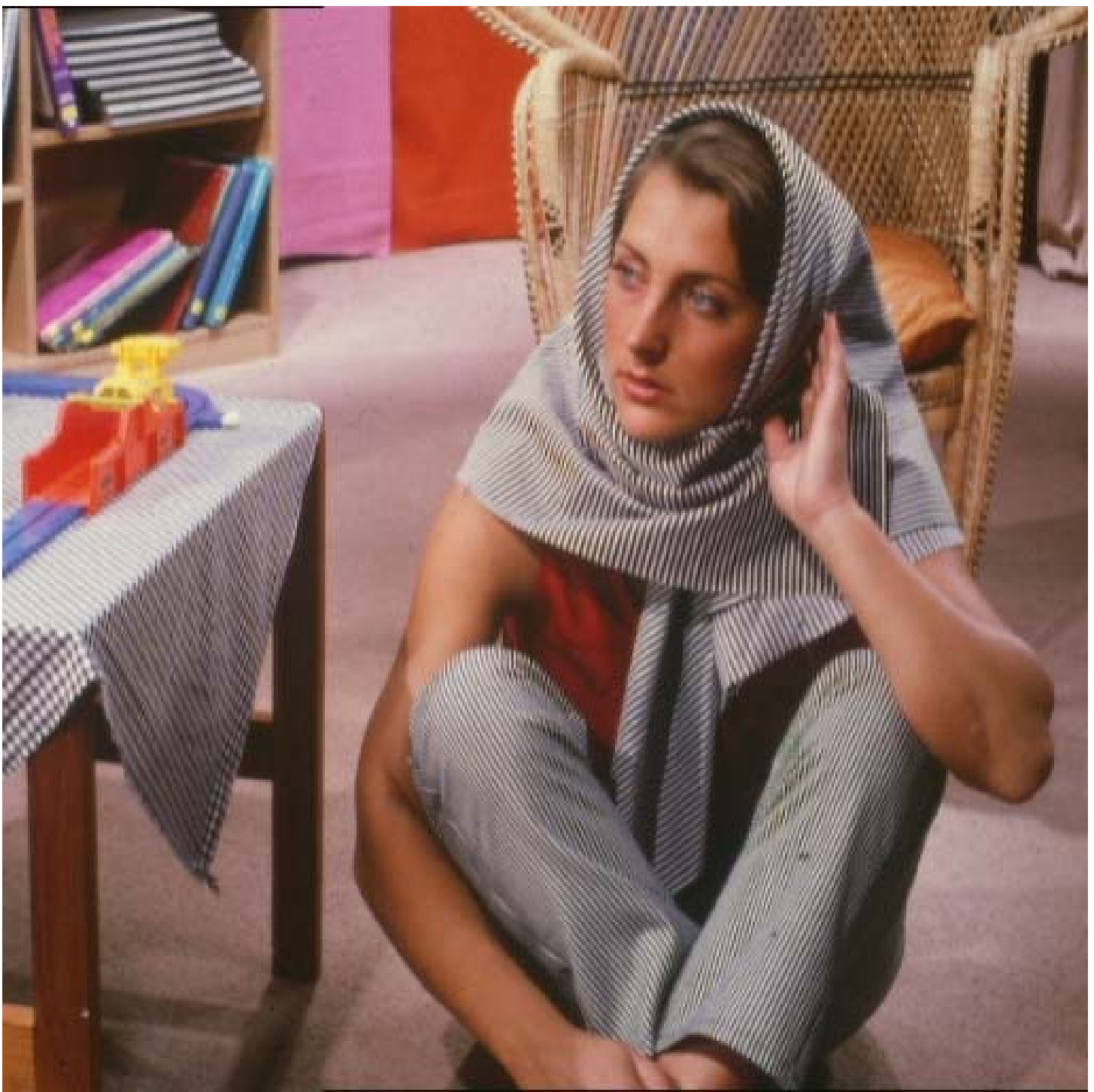}
\includegraphics[width=.158\linewidth,height=.158\linewidth, trim = 0.0cm 0cm 0cm 0cm,clip=true]{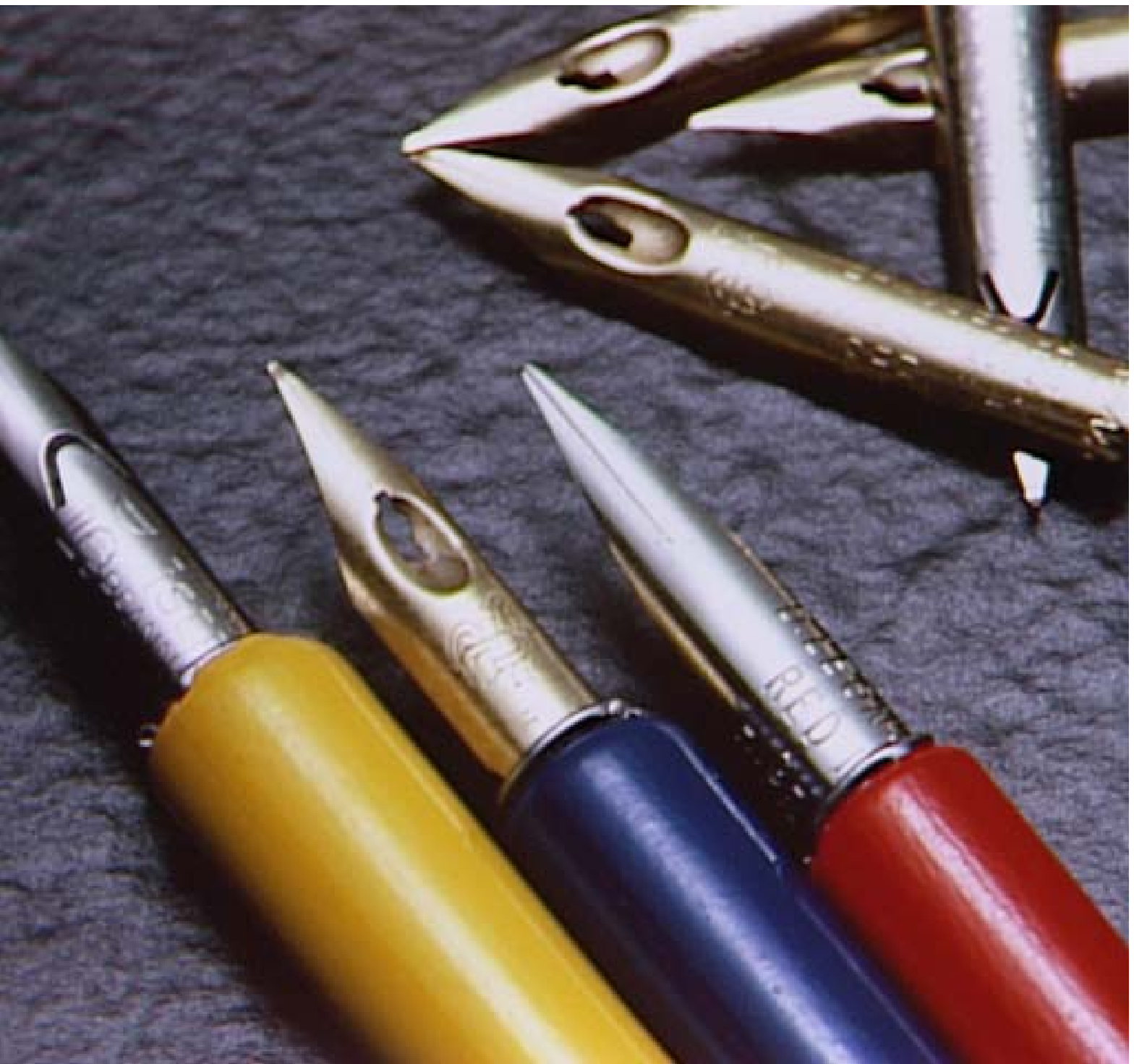}
\includegraphics[width=.158\linewidth, trim = 0.0cm 0cm 0cm 0cm,clip=true]{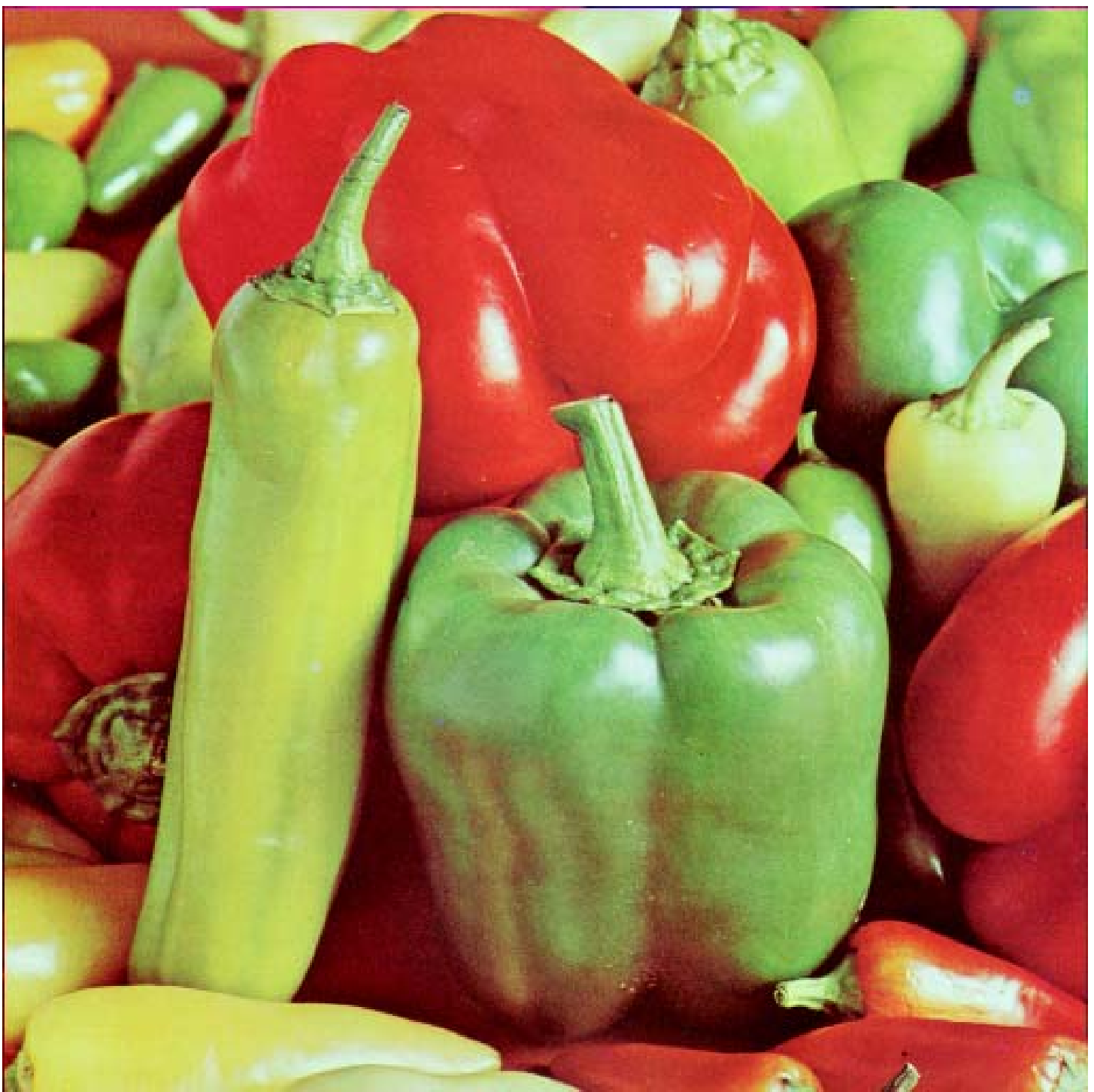}
\includegraphics[width=.158\linewidth, trim = 0.0cm 0cm 0cm 0cm,clip=true]{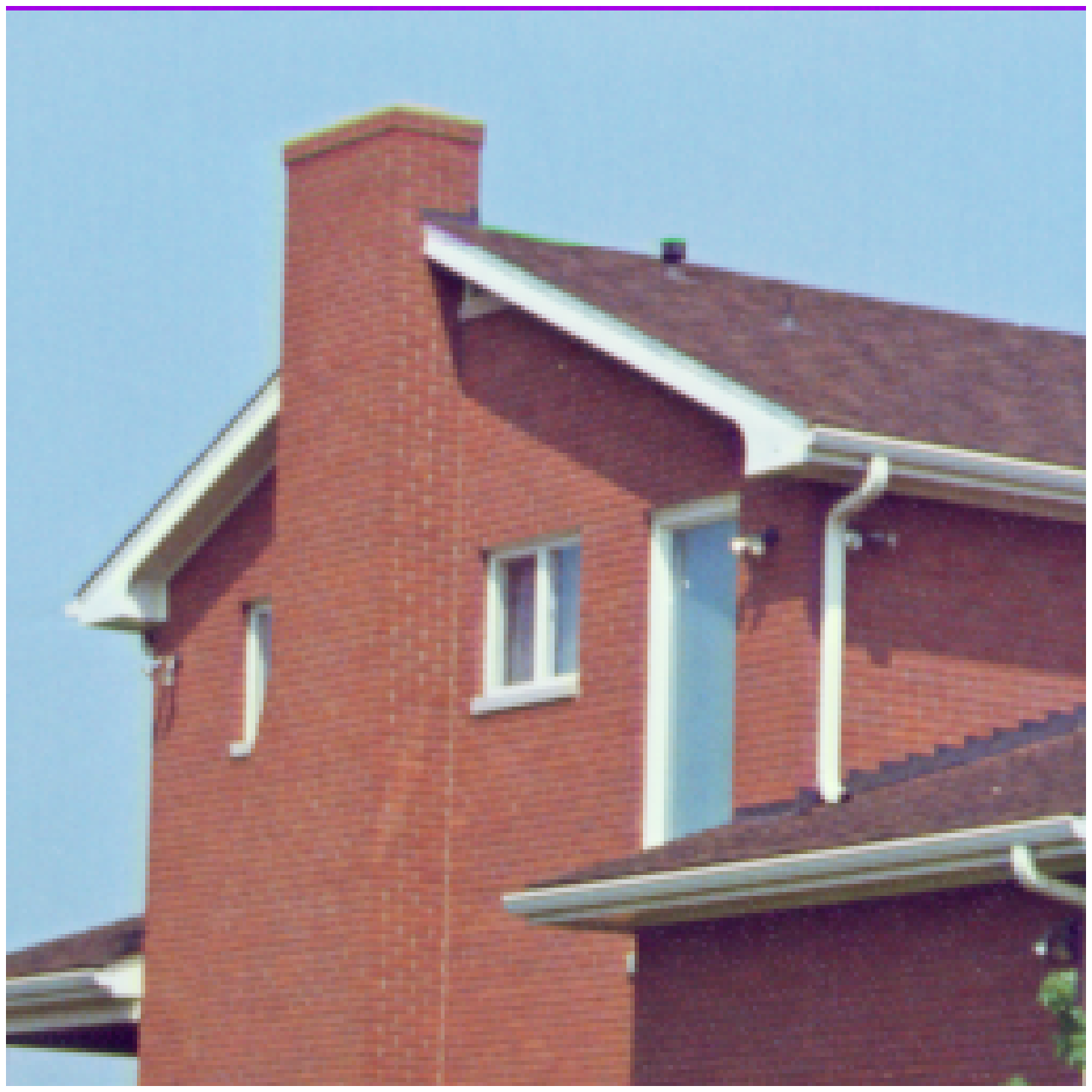}
\includegraphics[width=.158\linewidth, trim = 0.0cm 0cm 0cm 0cm,clip=true]{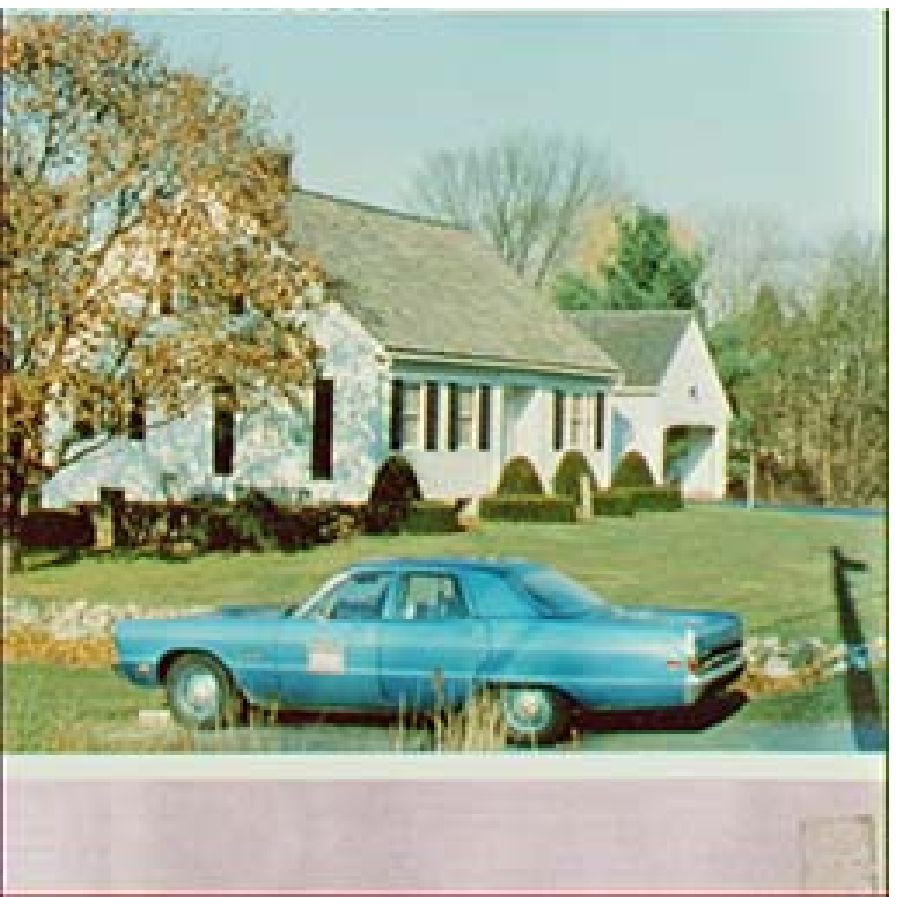}
\caption{Images used in Example~\ref{ex_image_tc}.}
\label{fig_6images}
\end{figure}

We fit six images of size $128 \times 128 \times 3$ shown in in Figure~\ref{fig_6images}, by TC models with various bond dimensions $(R_1, R_2, R_3 = R_1)$. 
$R_1 = 2, \ldots, 40$ and $R_2\ge 2$ such that the number of  parameters of TC models should not exceed the number of data elements, i.e., $49152$. There are in total 388-392 TC decompositions for each image. 
\be
\#TC(R_1,R_2,R_1) = 256 R_1 R_2 + 3 R_1^2 \le 49152
\ee

For the same approximation bound, we compare three models obtained using ALS, ALS with sensitivity correction, and algorithm with Sensitivity Control
$$\|\tY - \tX \|_F \le \epsilon \, \|\tY\|_F $$
Figure~\ref{fig:images_plots} compares the approximation errors for different bond dimensions. 

For the ''Peppers'' image, the SSC significantly improves the approximation error for the same TC model. In other words, SSC  allows us to choose a smaller TC model with the same approximation error bound than using ALS (or any other algorithms for TC). For example, at the approximation error bound of $0.02$, SSC gives the TC model with bond $(6-16-6)$, which comprises 24684 parameters, and attains a relative approximation error of $0.01927 < 0.02$. 
For the same bond dimensions, ALS converges to a model with an approximation error of $0.0331 > 0.02$. In order to attain the same accuracy, ALS procedures a TC model with bond dimensions of $(6-21-6)$ or $(7-18-7)$, which have 32364 or 32403 parameters, i.e., demanding 7719 more parameters than the model estimated by SSC.

The improvements are even significant for TC decomposition with high accuracy, i.e., low approximation error. Not only for the `'Pepper'' image, but Figure~\ref{fig:images_plots} also shows that SSC gains performance of ALS for approximation of the other images.



\begin{figure}
\subfigure[Mandrill]{\includegraphics[width=.48\linewidth, trim = 0.0cm 0cm 0cm 0cm,clip=true]{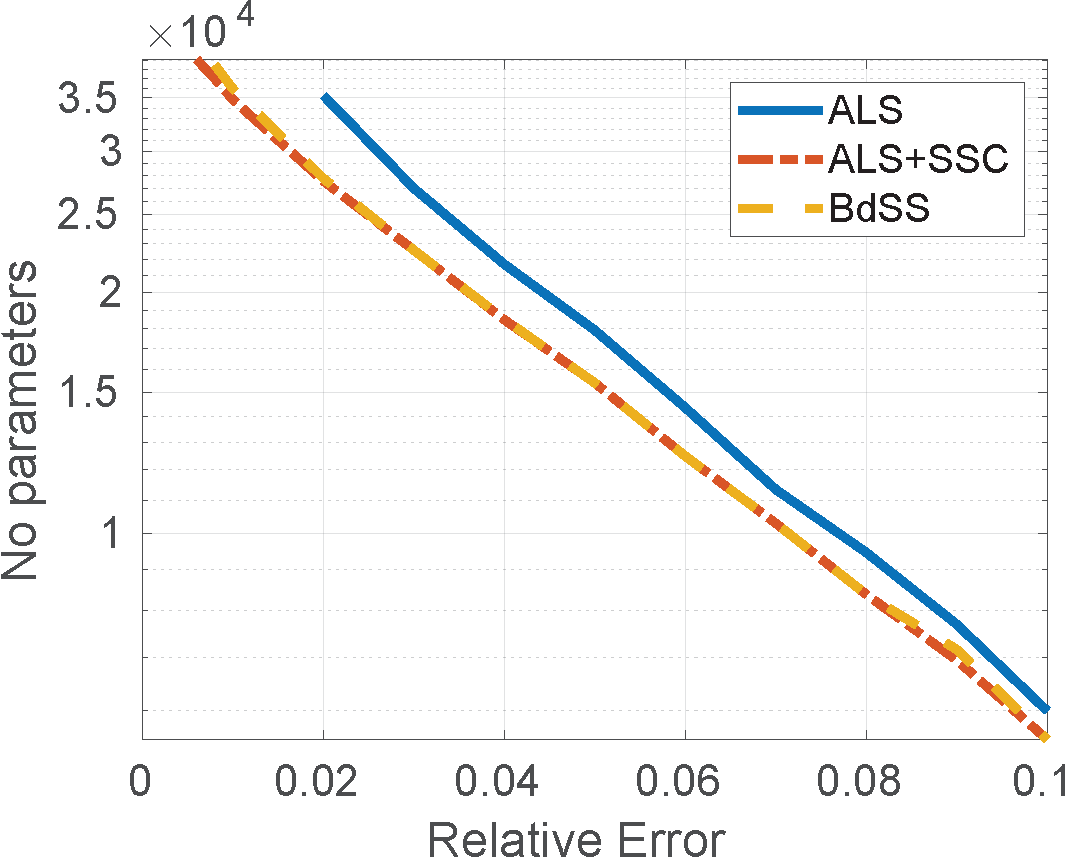}}
\subfigure[Peppers]{\includegraphics[width=.48\linewidth, trim = 0.0cm 0cm 0cm 0cm,clip=true]{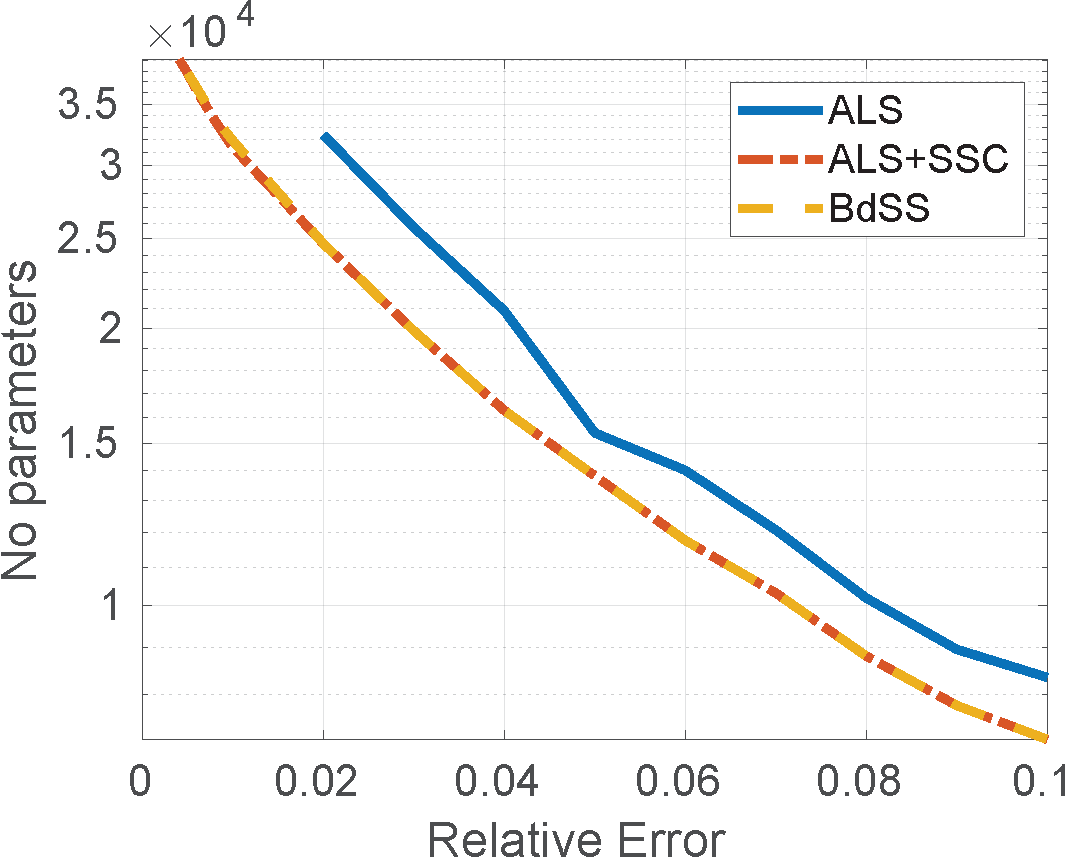}}
\\
\subfigure[Lena]{\includegraphics[width=.48\linewidth, trim = 0.0cm 0cm 0cm 0cm,clip=true]{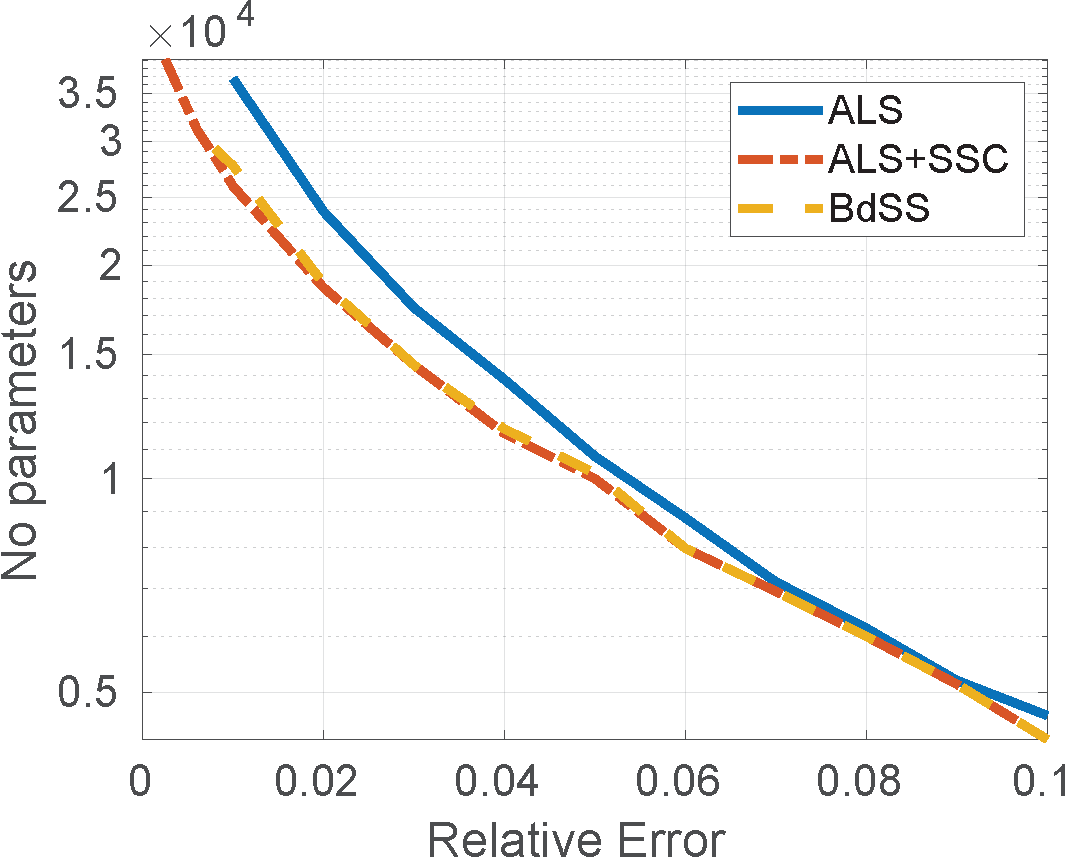}}
\subfigure[Barbara]{\includegraphics[width=.48\linewidth, trim = 0.0cm 0cm 0cm 0cm,clip=true]{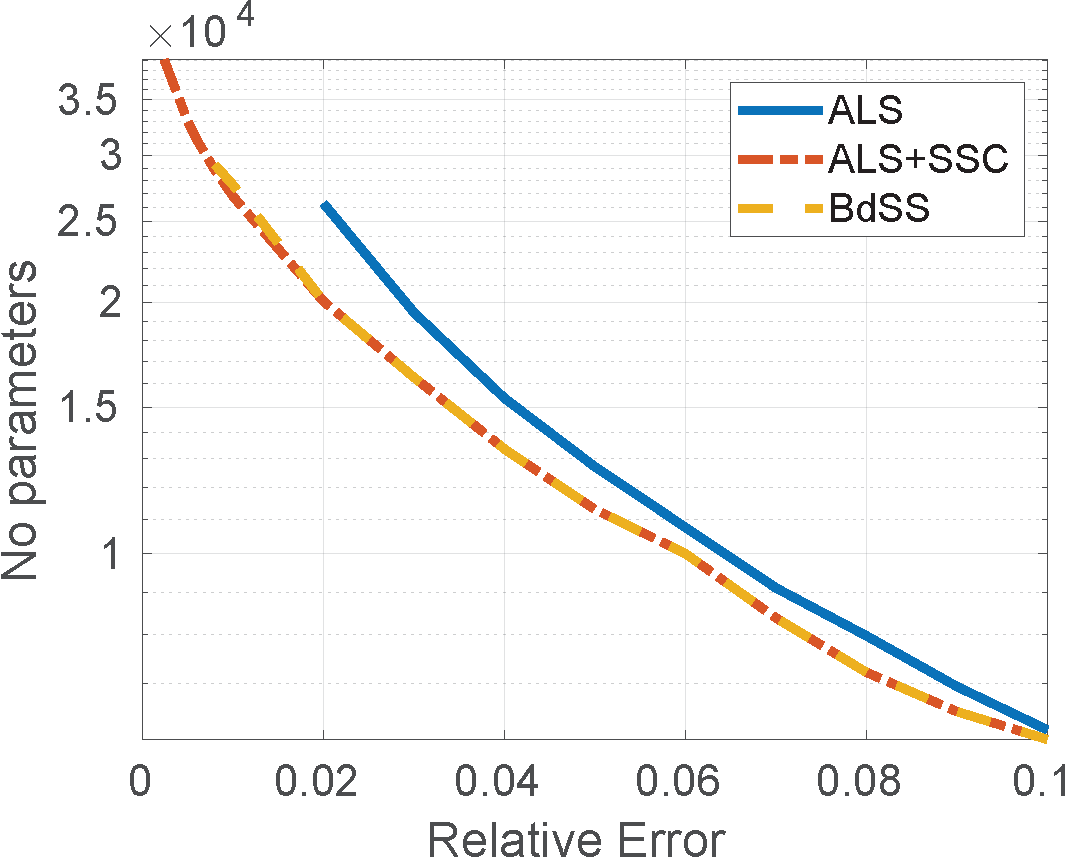}}
\\
\subfigure[Tiffany]{\includegraphics[width=.48\linewidth, trim = 0.0cm 0cm 0cm 0cm,clip=true]{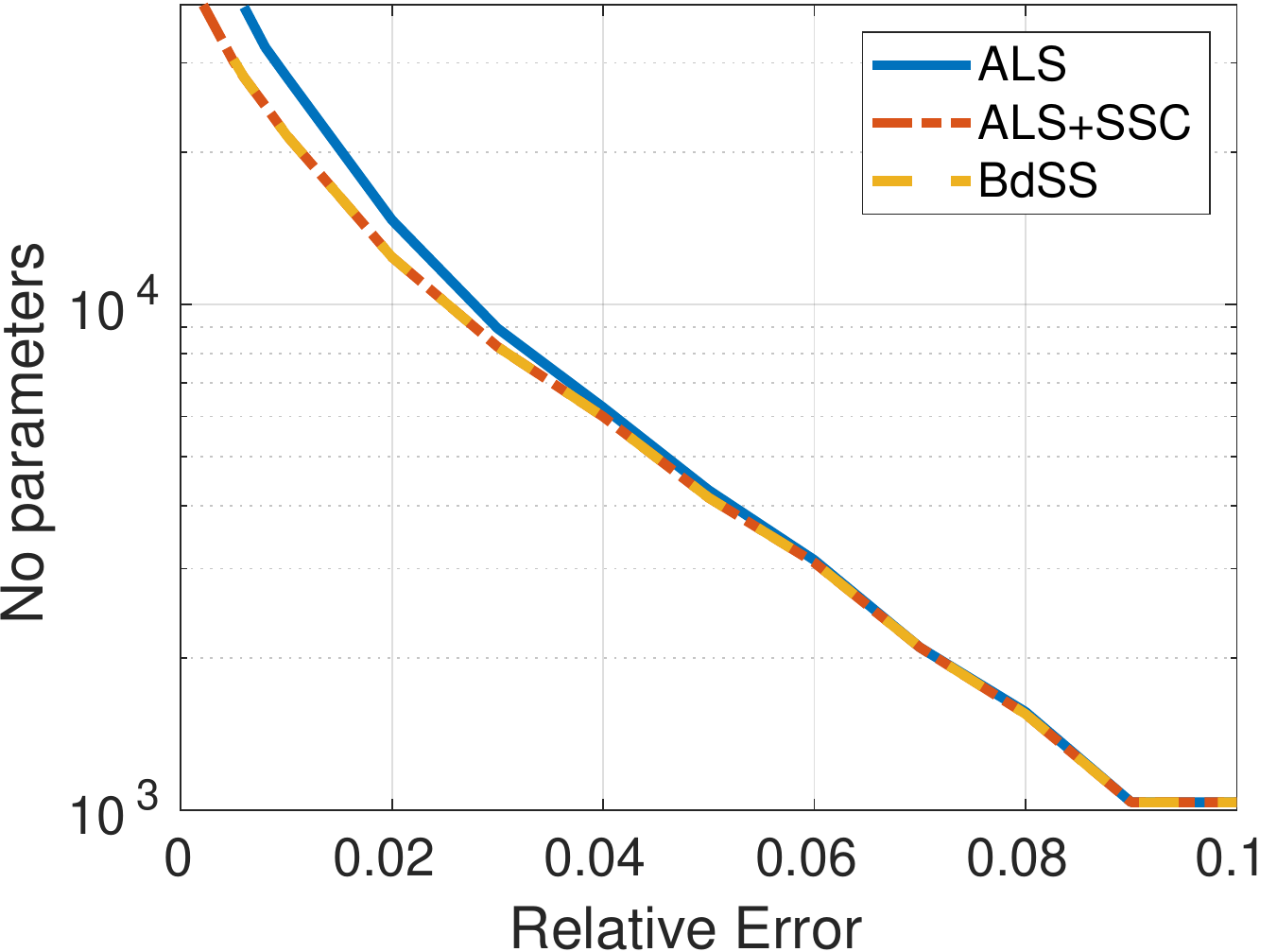}}
\subfigure[House]{\includegraphics[width=.48\linewidth, trim = 0.0cm 0cm 0cm 0cm,clip=true]{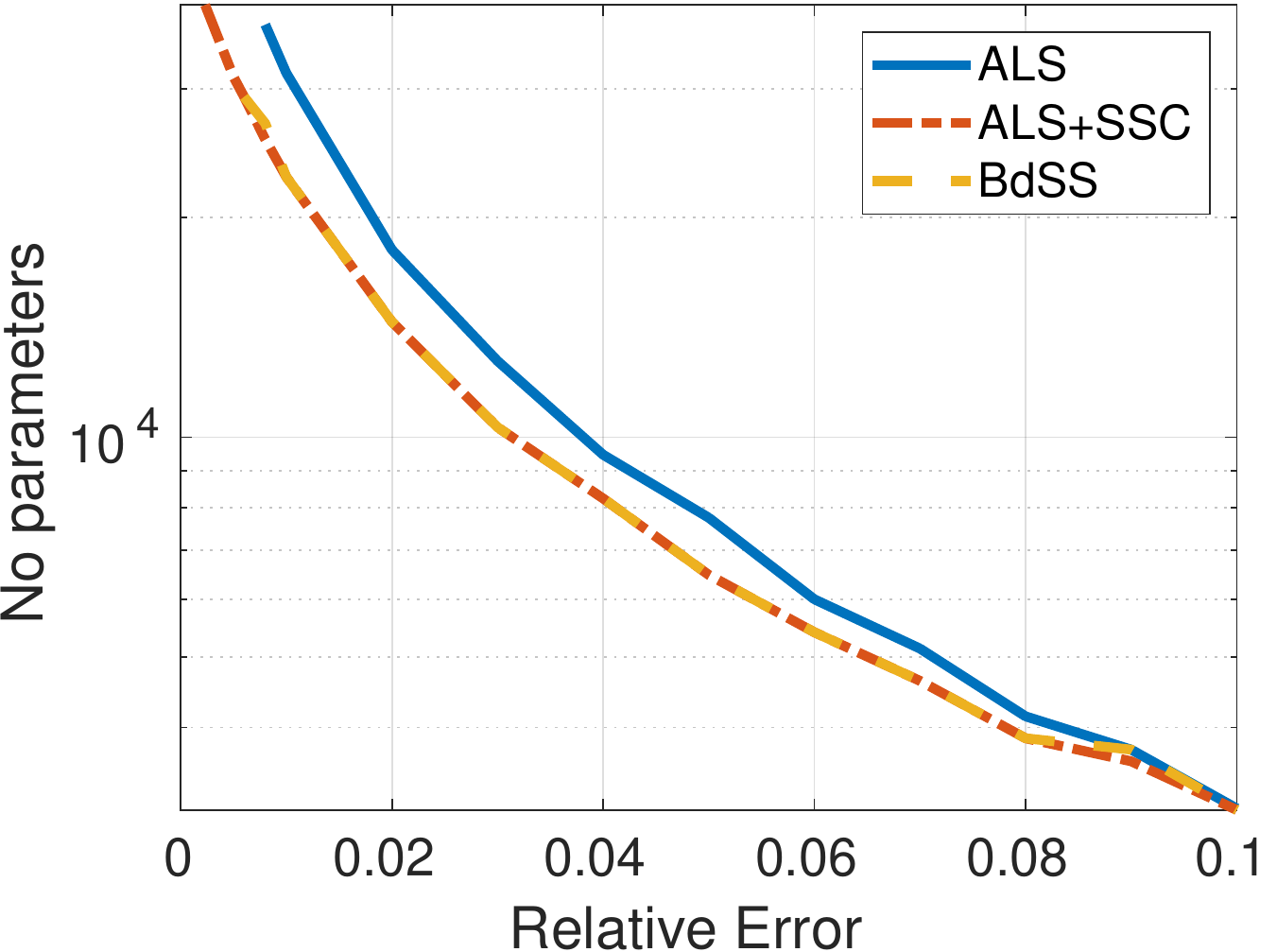}}
\caption{Comparison of TC models which approximate images in Example~\ref{ex_image_tc}.  }\label{fig:images_plots}
\end{figure}

 \end{example}
 

\subsection{CNN compression}

\setcounter{example}{3}
\begin{example}[extended from Example~\ref{ex_resnet92}]
\label{ex_resnet92_ext}
Figure~\ref{fig_ex3_err_it} compares convergence of the decomposition of convolutional kernels using ALS and SSC. 
A similar comparison is presented in Figure~\ref{fig:resnet18_}.
SSC was applied after 3000 ALS updates, and the decomposition resumed with 3000 ALS updates. Decompositions using only ALS could not achieve the approximation errors obtained by SSC.
\end{example}

\setcounter{example}{3}
\begin{example}[extended from Example~\ref{ex_resnet18_cifar10_v2}]
\label{ex_resnet18_cifar10_v2_ext}
SSC achieves smaller approximation errors, and yields estimated models with smaller sensitivity. This helps the new neural networks to fine-tune easier. 
In Figure~\ref{fig_ex3_b}, we provide more comparisons between the learning curves of two versions of ResNet-18 after replacing one convolutional layer by a TC-layer, one with convolutional kernels estimated using SSC, and another one obtained by ALS. Similar curves are presented in Figure~\ref{fig:resnet18_ex3}(c).
In most examples, neural networks without SSC cannot attain the original accuracy of ResNet-18, e.g., convolutional layers 4, 5, 6, 9, 10, 12, 14, 15, 16, 17, or need much more number iterations than the networks using TC-SSC, e.g., layer 3. Except for the network with TC-layer applied to the convolutional layer 11, we observe a slight difference between the two learning curves.

In addition to single compression, we perform full model compression of ResNet-18. TC-layers replace all convolutional layers. 
Figure~\ref{fig::resnet18_cifar_fullcompression} shows that the new ResNet-18 with kernels estimated using ALS could not attain the original accuracy of 92.29\% for CIFAR-10. A similar ResNet-18 but weights in all TC-layers estimated with SSC can recover the initial accuracy after fine-tuning.
\end{example}

\begin{figure}[!ht]
\centering
\subfigure[Convolutional layer 3]{\includegraphics[width=.30\linewidth, trim = 0.0cm 0cm 0cm 0cm,clip=true]{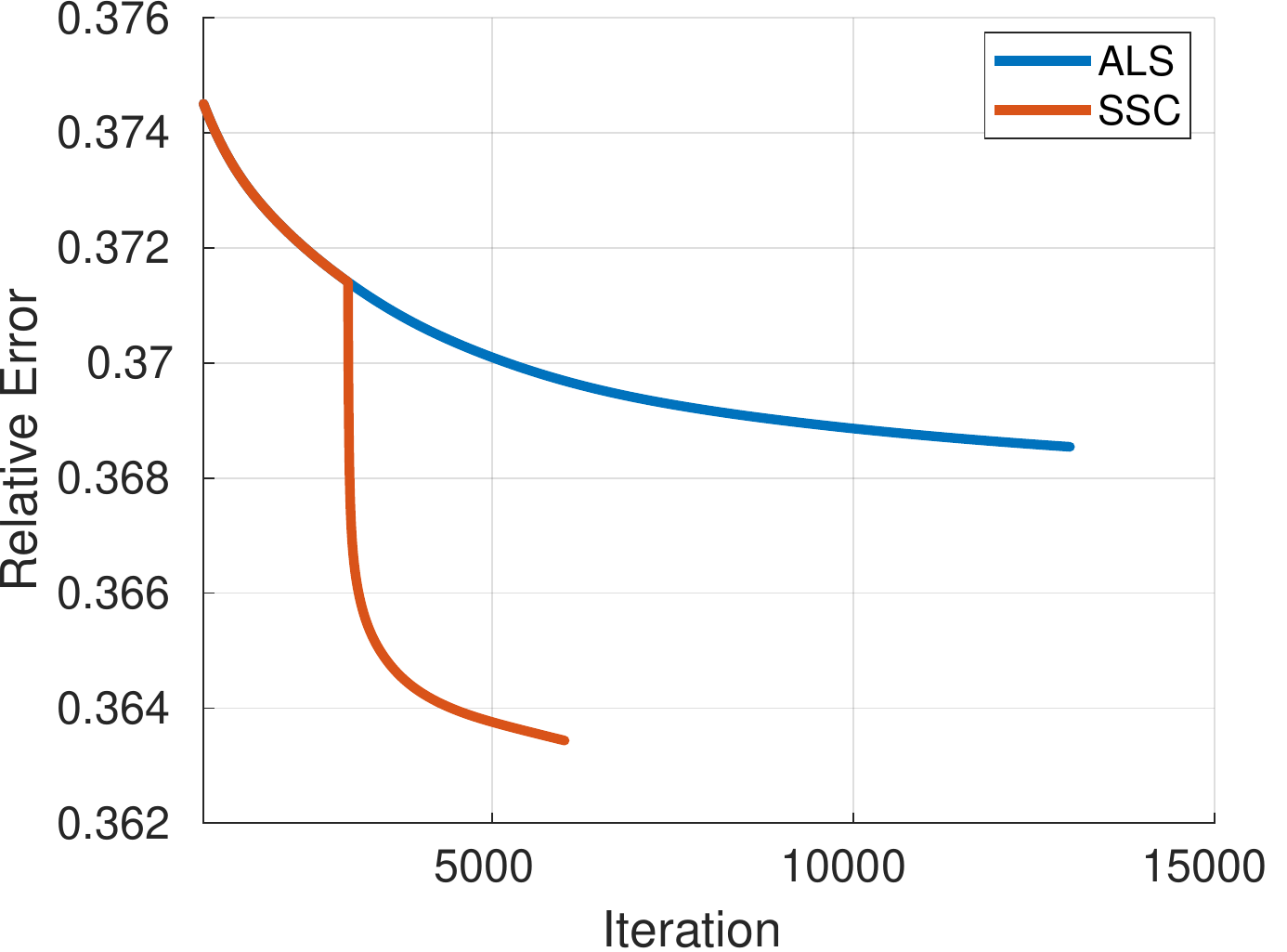}}
\subfigure[Convolutional layer 4]{\includegraphics[width=.30\linewidth, trim = 0.0cm 0cm 0cm 0cm,clip=true]{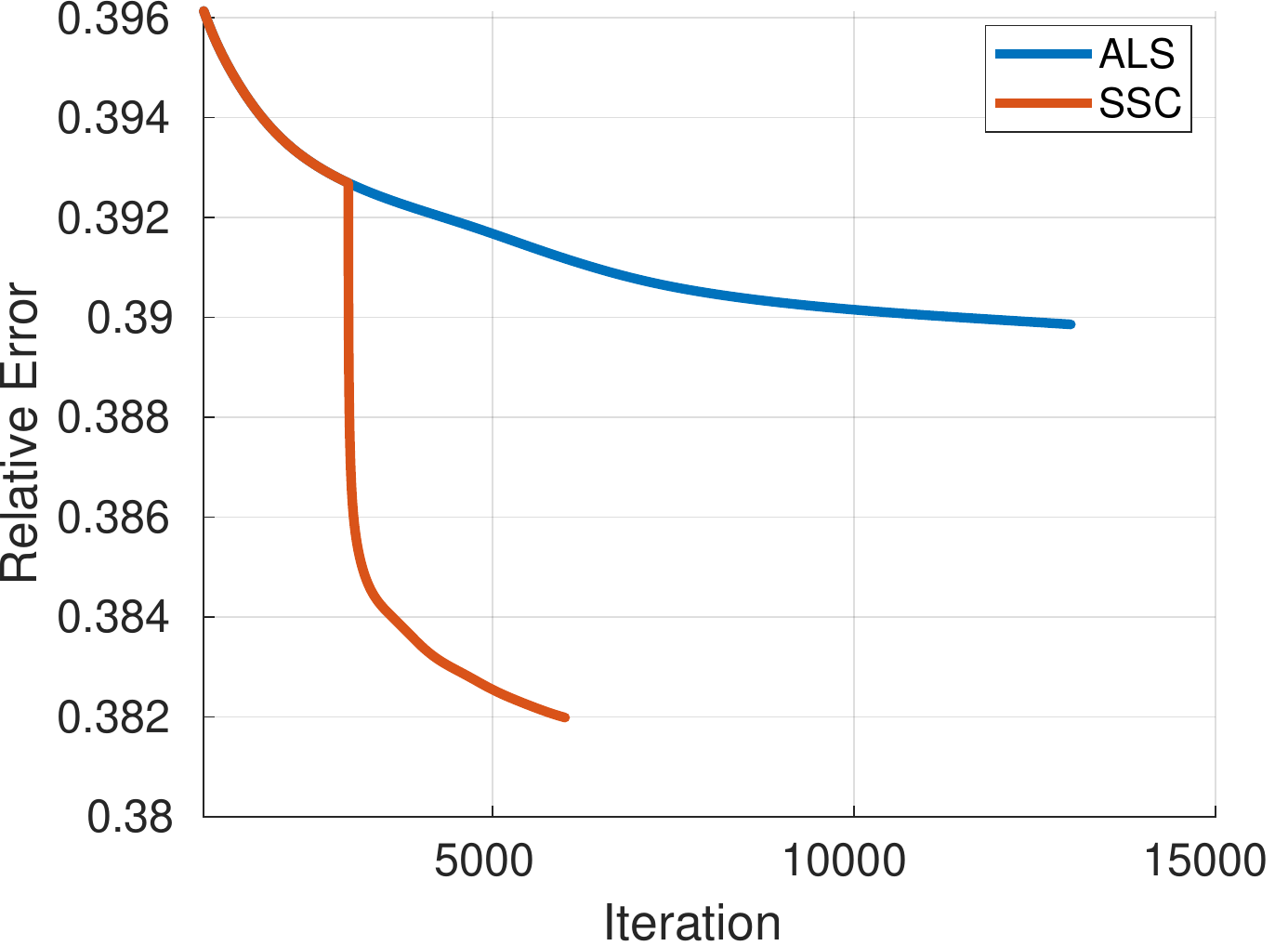}}
\subfigure[Convolutional layer 5]
{\includegraphics[width=.30\linewidth, trim = 0.0cm 0cm 0cm 0cm,clip=true]{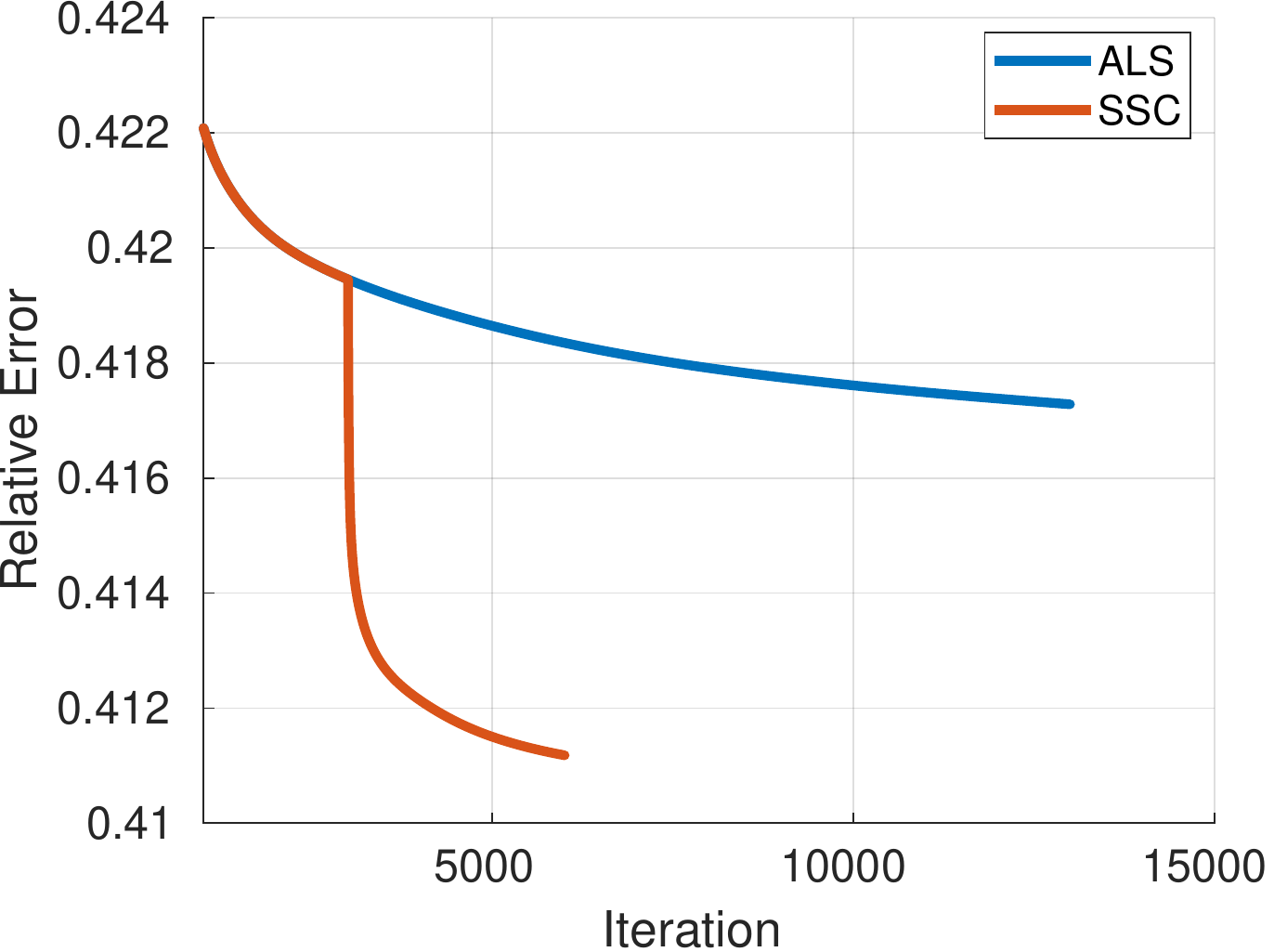}}
\subfigure[Convolutional layer 6]
{\includegraphics[width=.30\linewidth, trim = 0.0cm 0cm 0cm 0cm,clip=true]{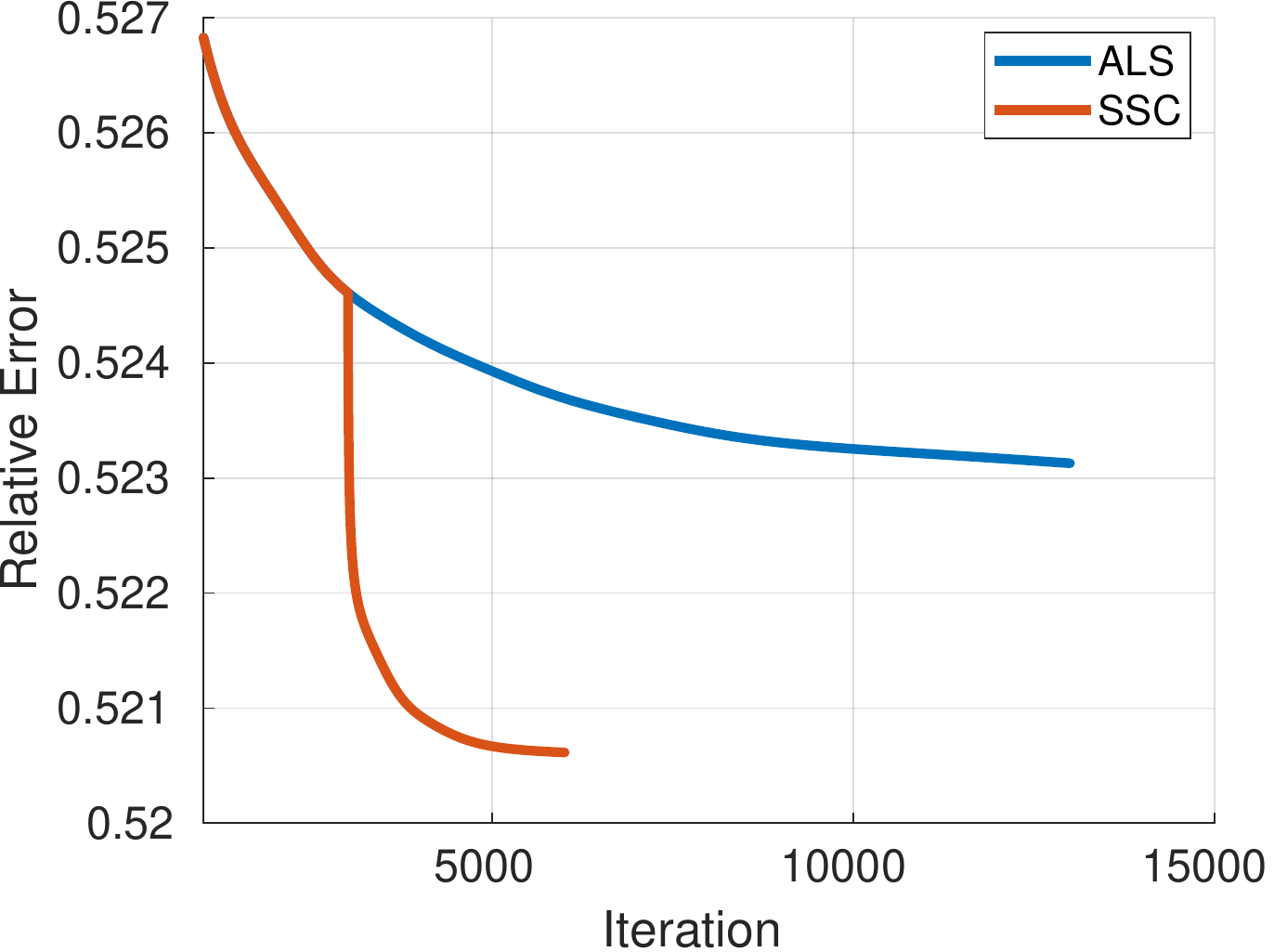}}
\subfigure[Convolutional layer 7]{\includegraphics[width=.30\linewidth, trim = 0.0cm 0cm 0cm 0cm,clip=true]{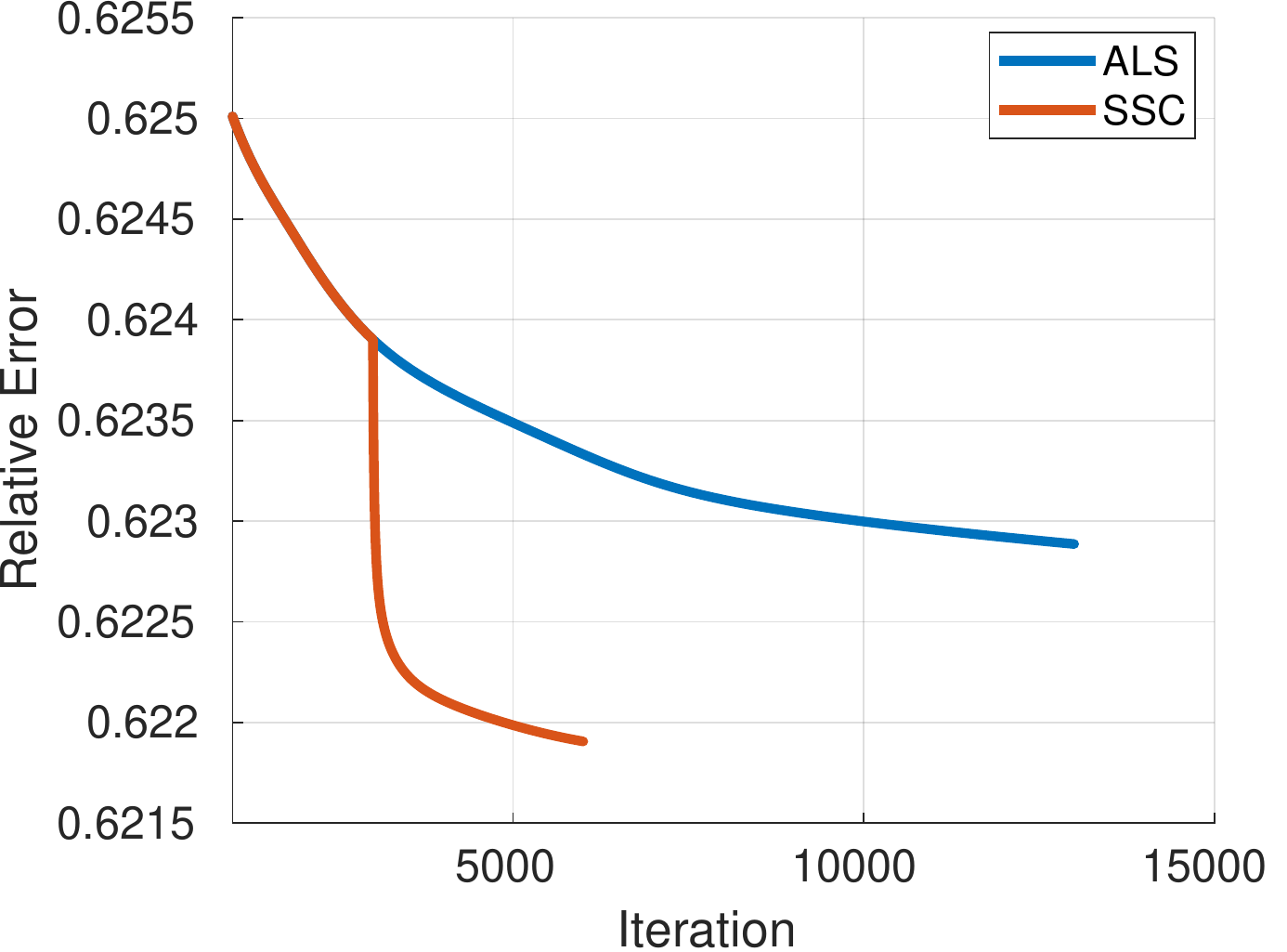}}
\subfigure[Convolutional layer 9]{\includegraphics[width=.30\linewidth, trim = 0.0cm 0cm 0cm 0cm,clip=true]{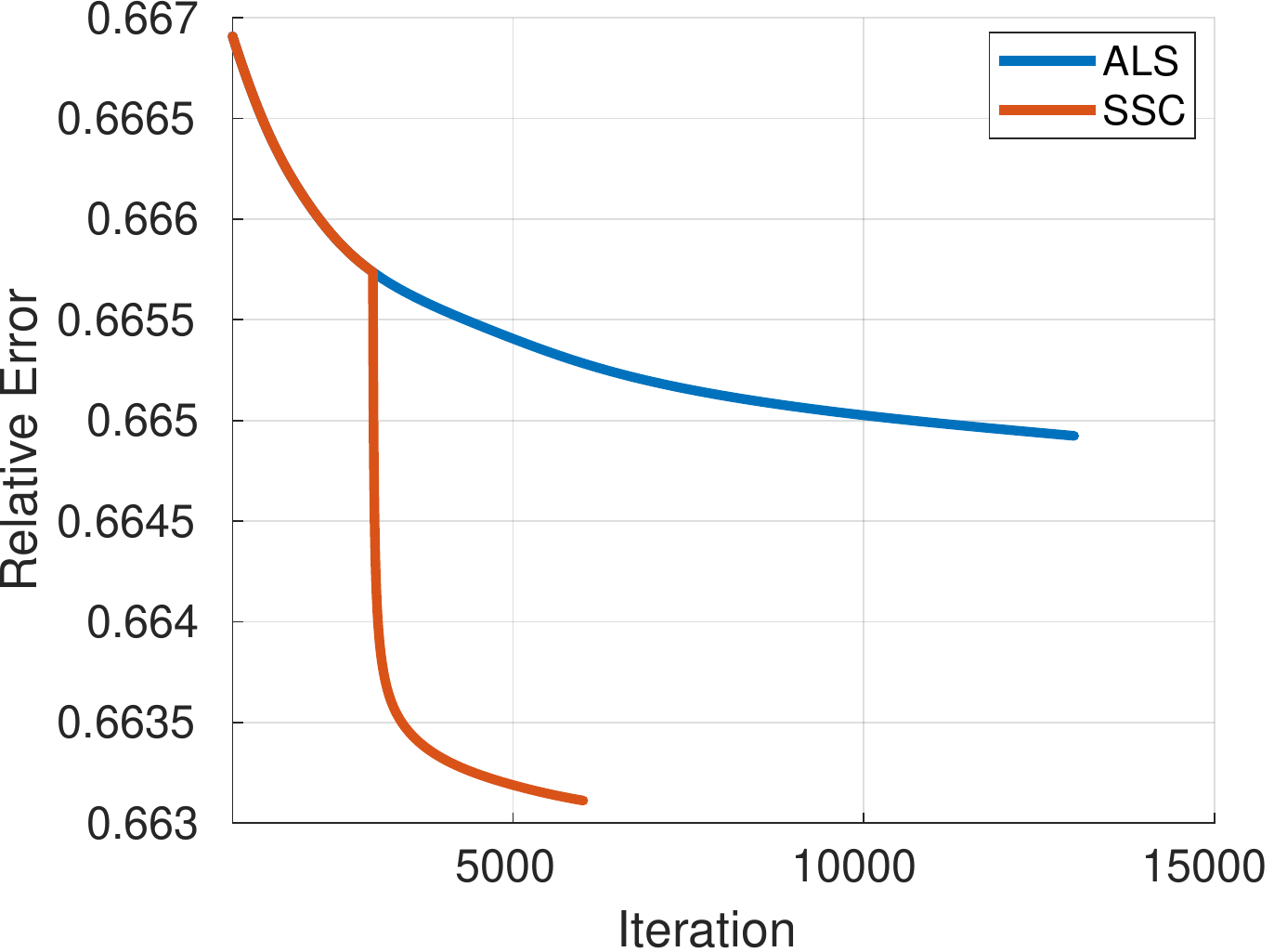}}
\subfigure[Convolutional layer 10]{\includegraphics[width=.30\linewidth, trim = 0.0cm 0cm 0cm 0cm,clip=true]{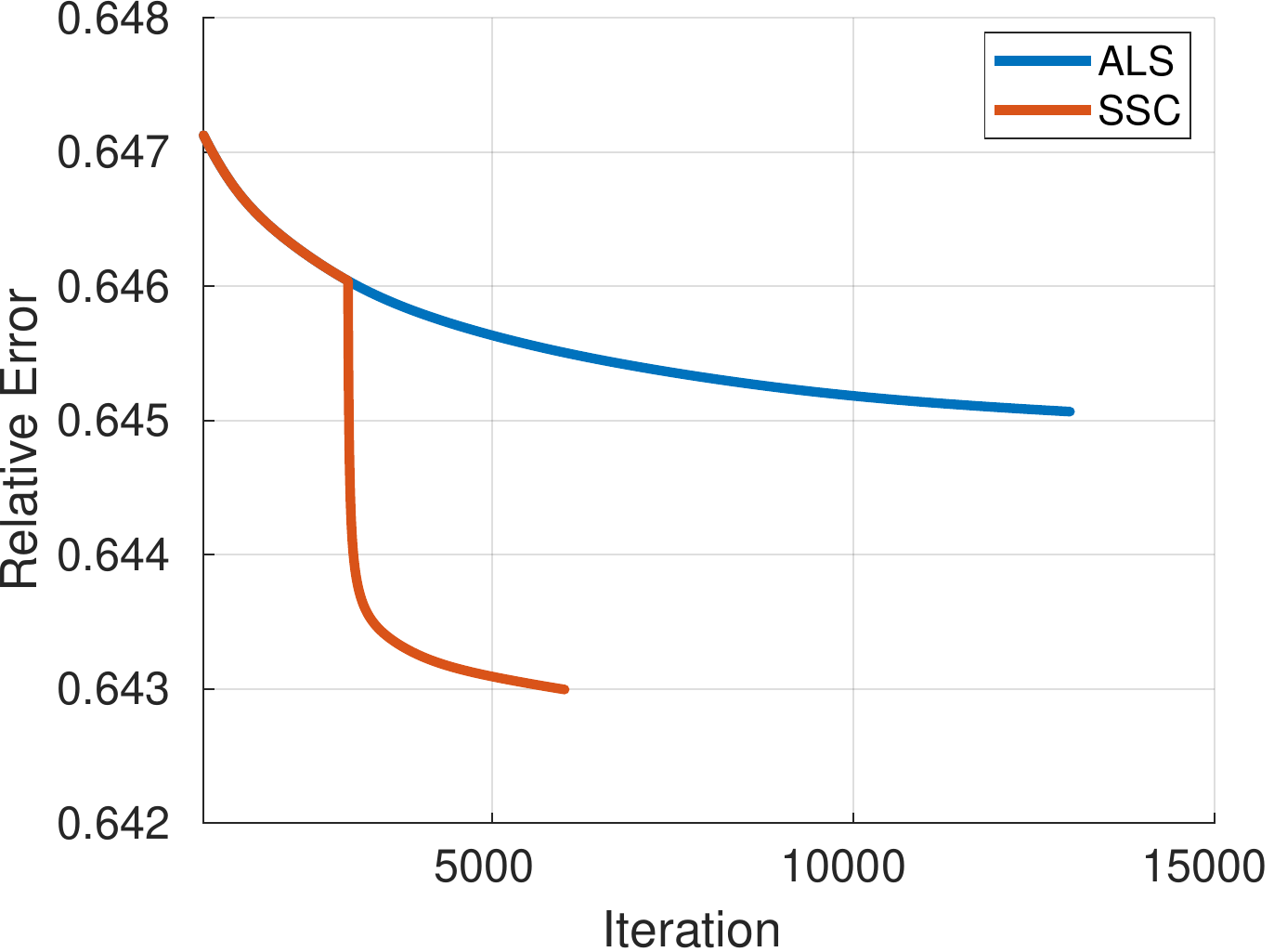}}
\subfigure[Convolutional layer 11]{\includegraphics[width=.30\linewidth, trim = 0.0cm 0cm 0cm 0cm,clip=true]{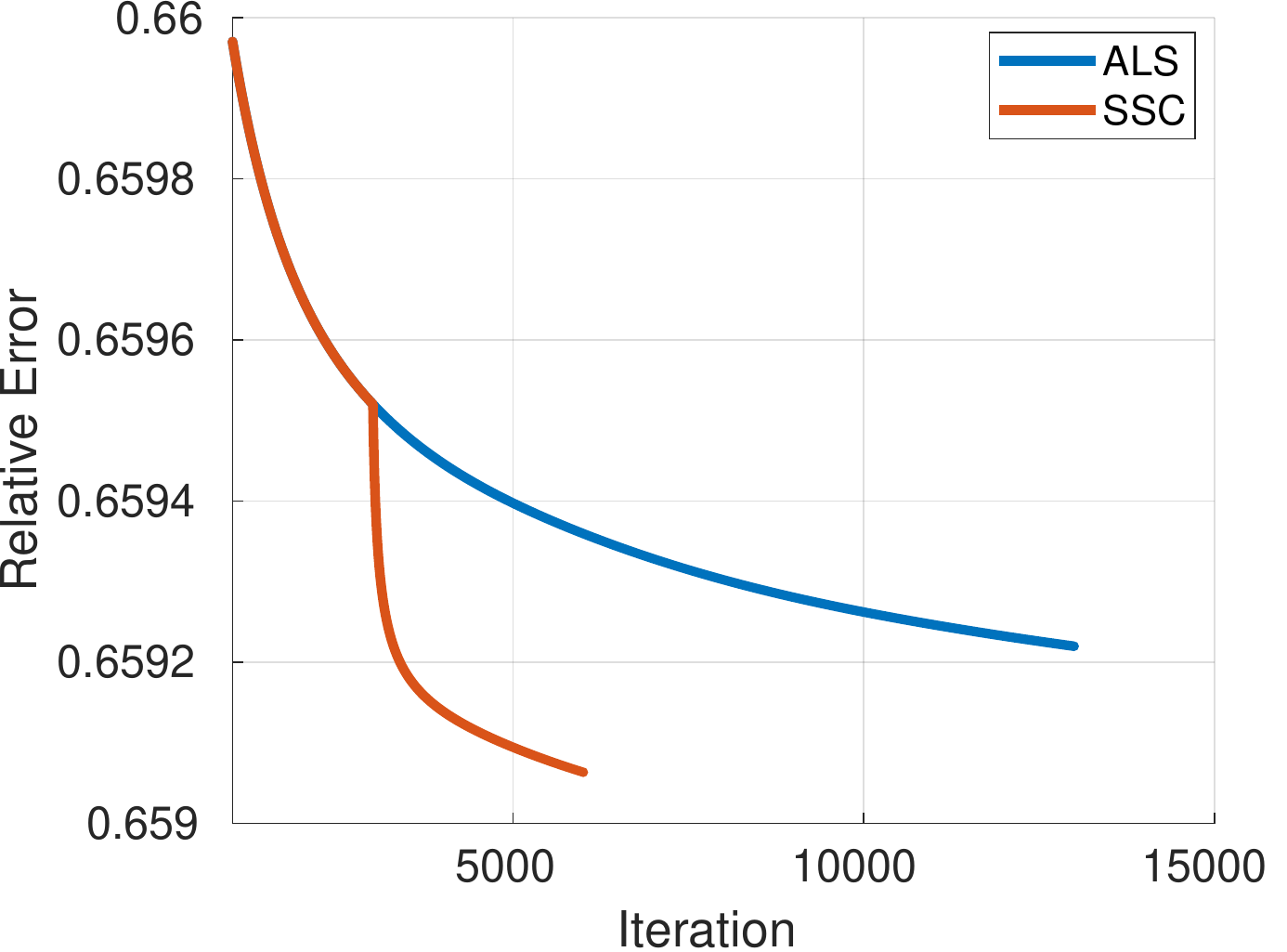}}
\subfigure[Convolutional layer 12]{\includegraphics[width=.30\linewidth, trim = 0.0cm 0cm 0cm 0cm,clip=true]{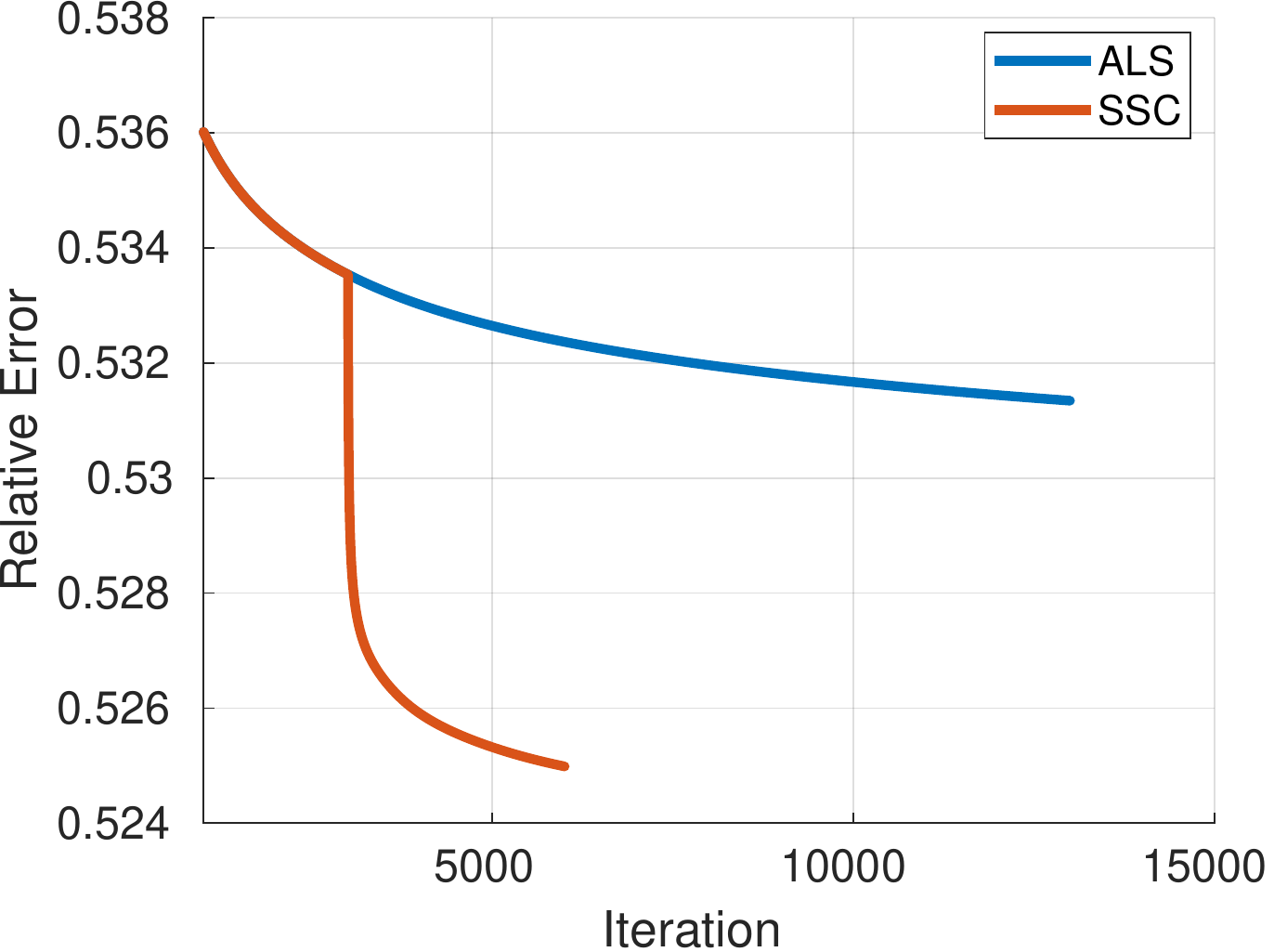}}
\subfigure[Convolutional layer 14]{\includegraphics[width=.30\linewidth, trim = 0.0cm 0cm 0cm 0cm,clip=true]{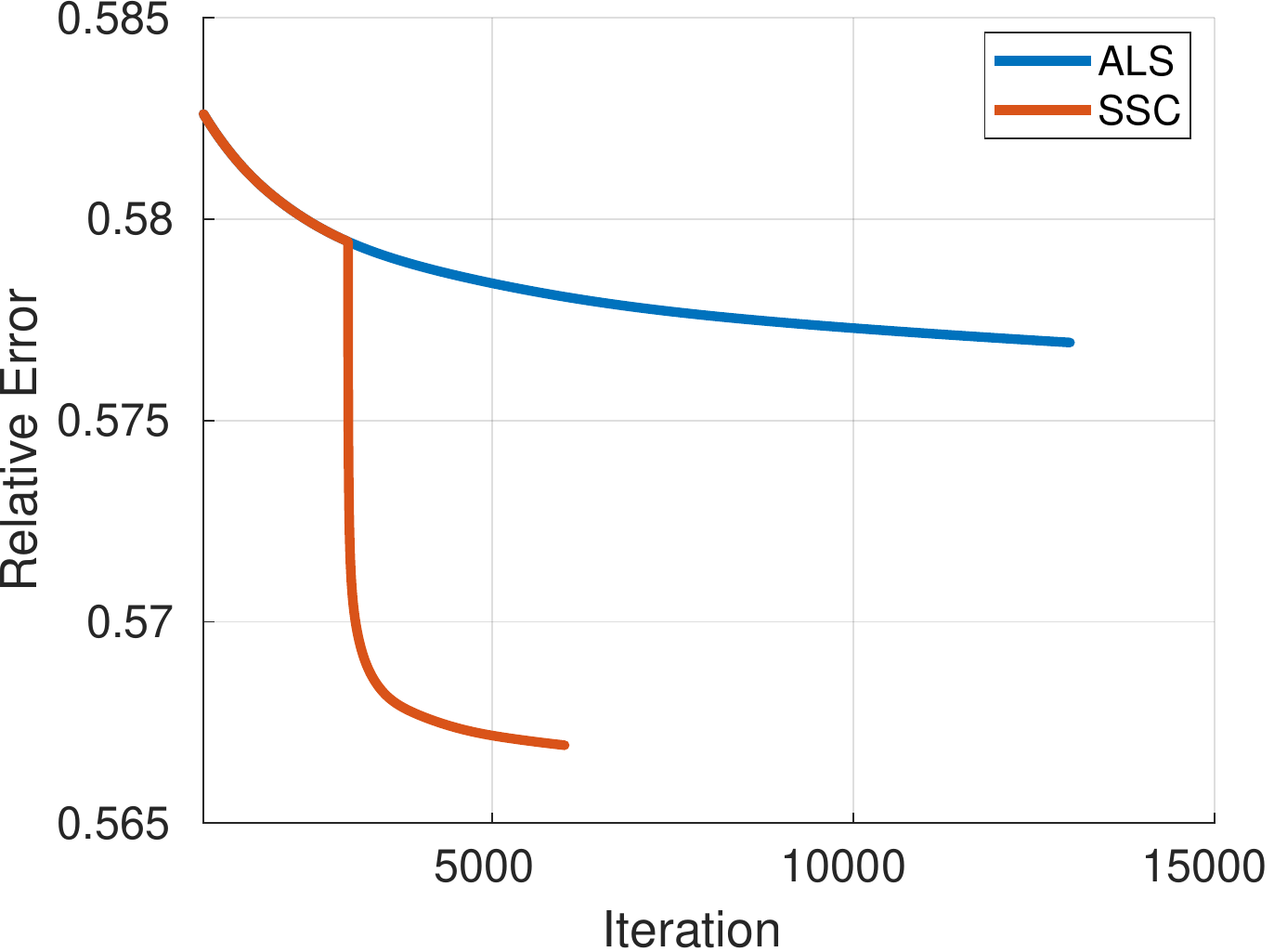}}
\subfigure[Convolutional layer 15]{\includegraphics[width=.30\linewidth, trim = 0.0cm 0cm 0cm 0cm,clip=true]{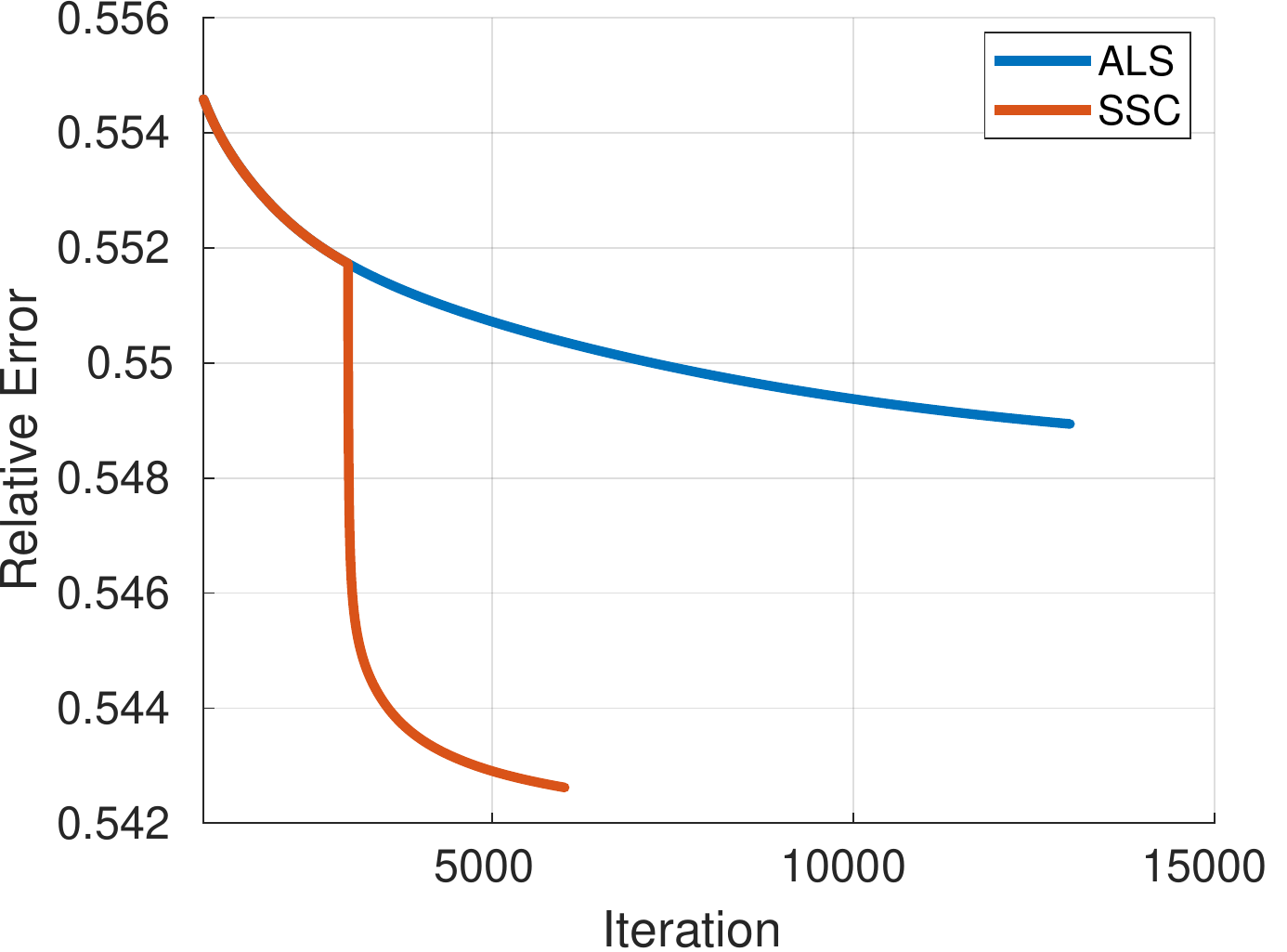}}
\subfigure[Convolutional layer 16]{\includegraphics[width=.30\linewidth, trim = 0.0cm 0cm 0cm 0cm,clip=true]{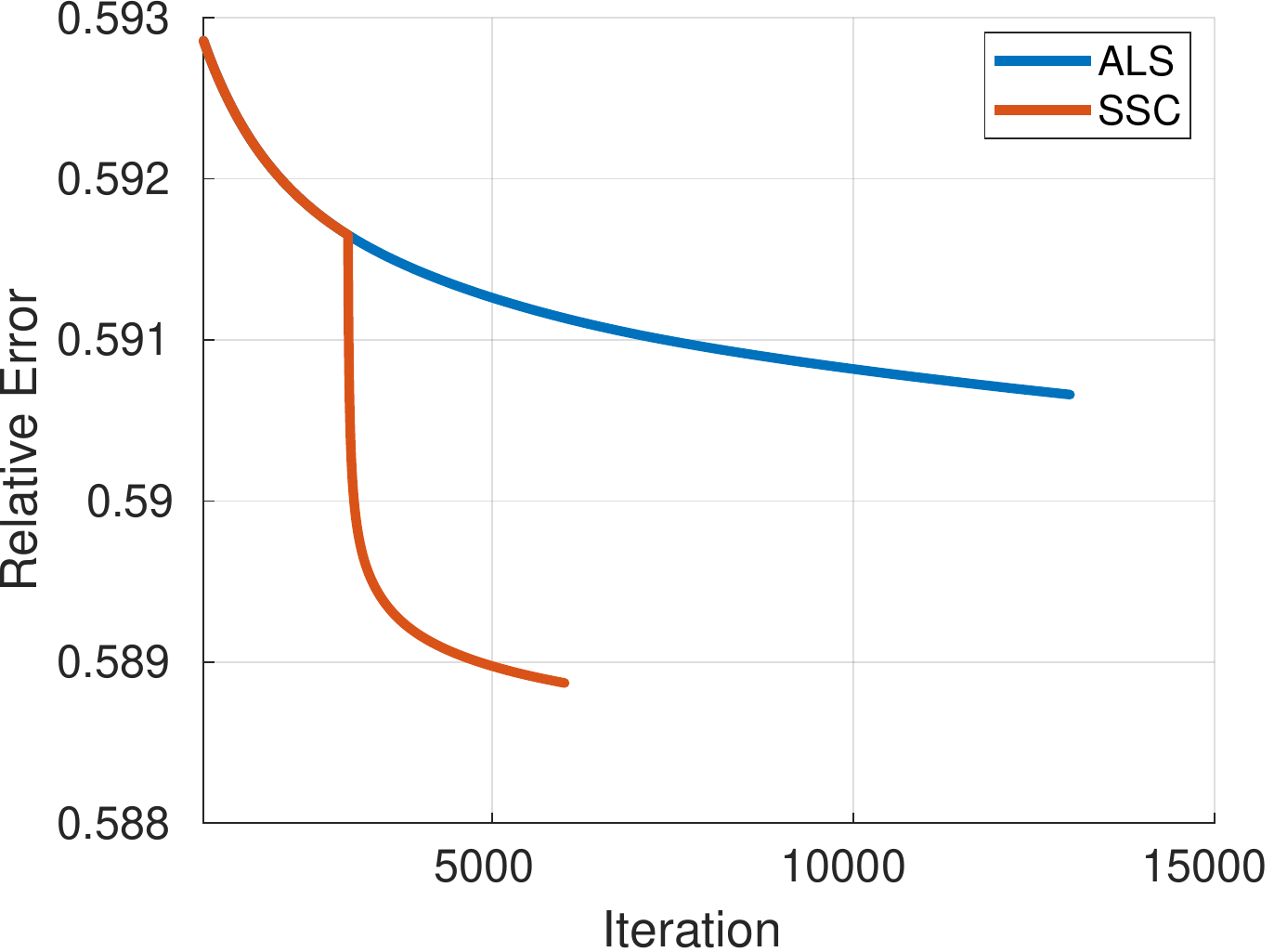}}
\subfigure[Convolutional layer 17]{\includegraphics[width=.30\linewidth, trim = 0.0cm 0cm 0cm 0cm,clip=true]{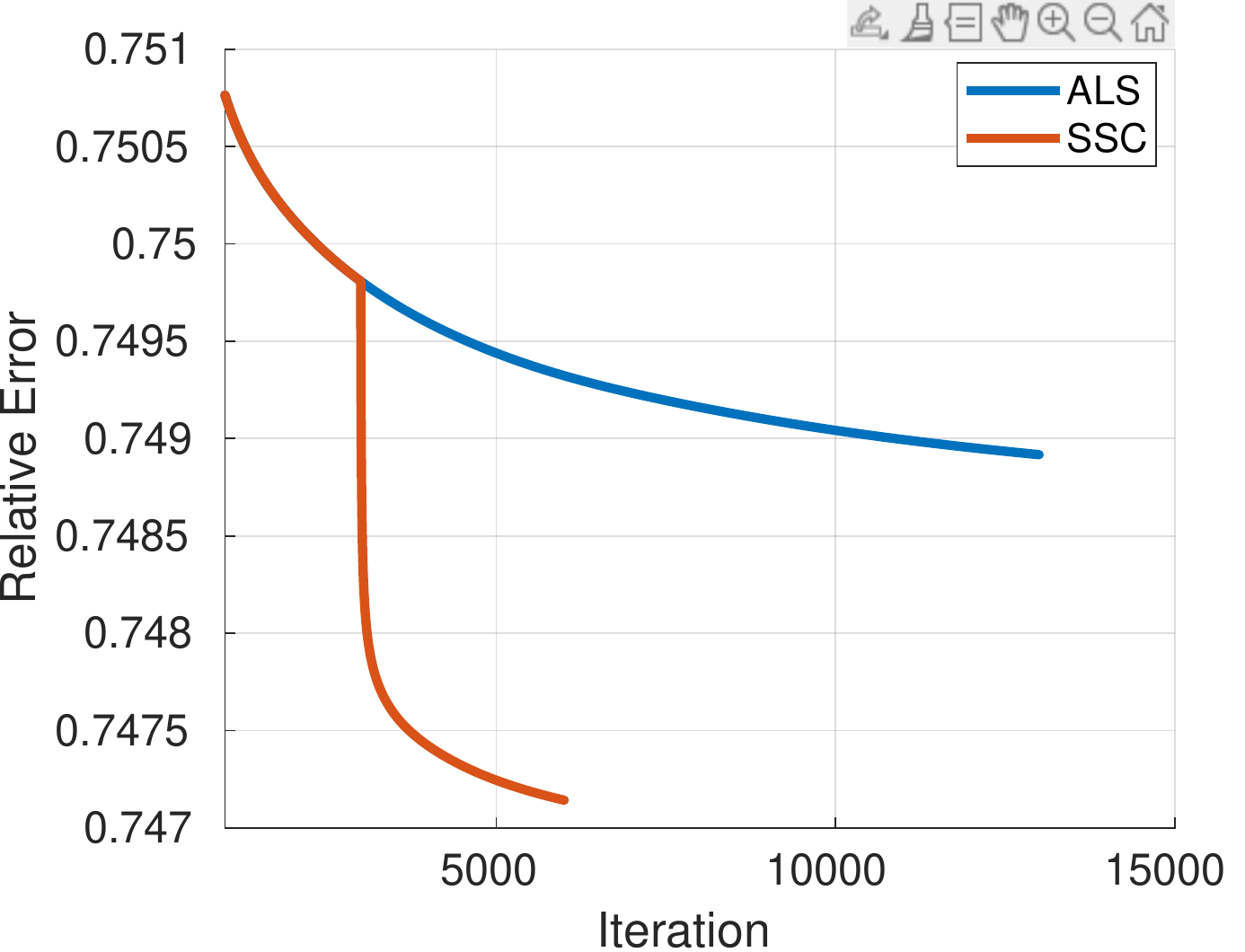}}
\subfigure[Convolutional layer 19]{\includegraphics[width=.30\linewidth, trim = 0.0cm 0cm 0cm 0cm,clip=true]{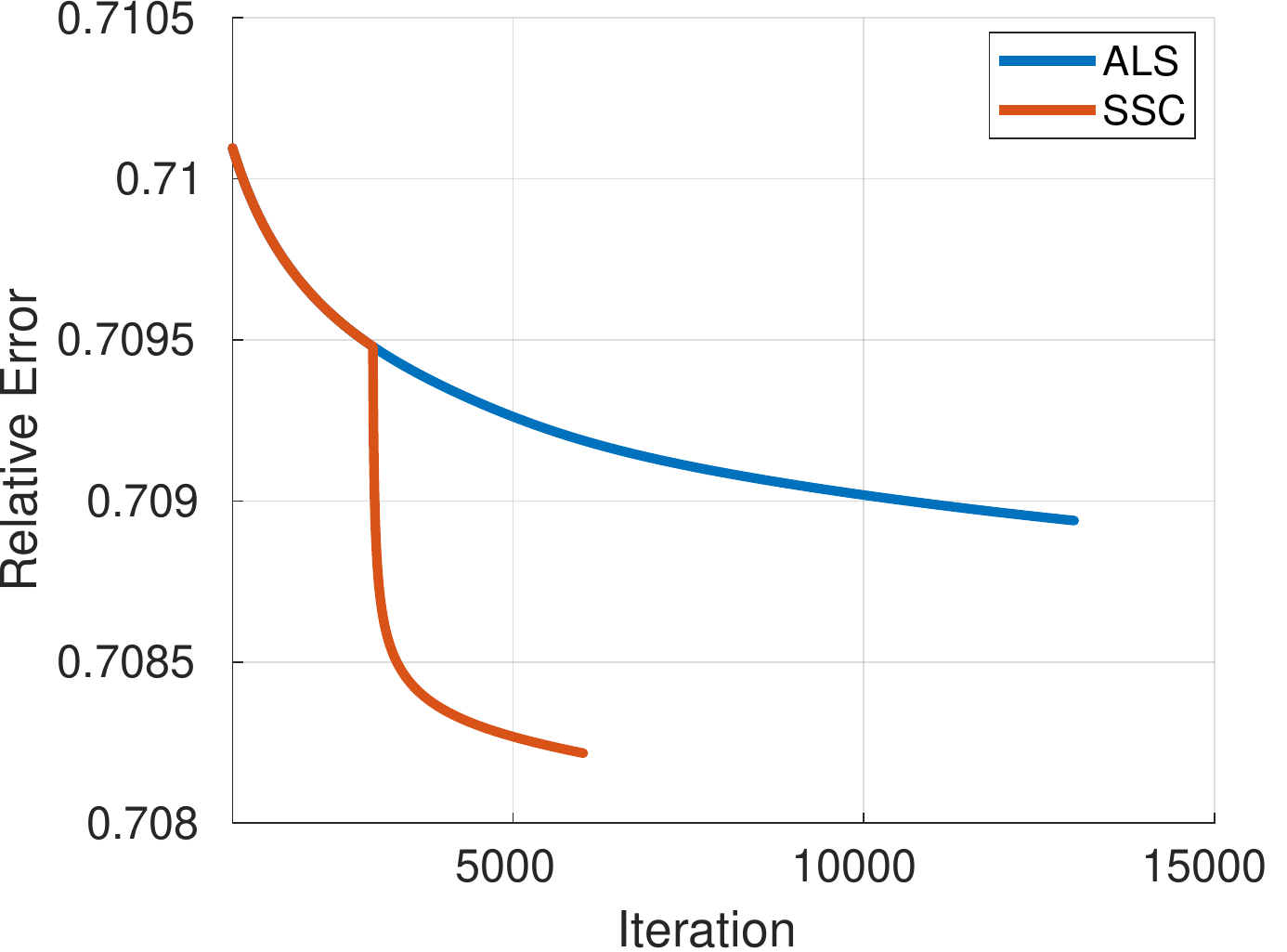}}
\subfigure[Convolutional layer 20]{\includegraphics[width=.30\linewidth, trim = 0.0cm 0cm 0cm 0cm,clip=true]{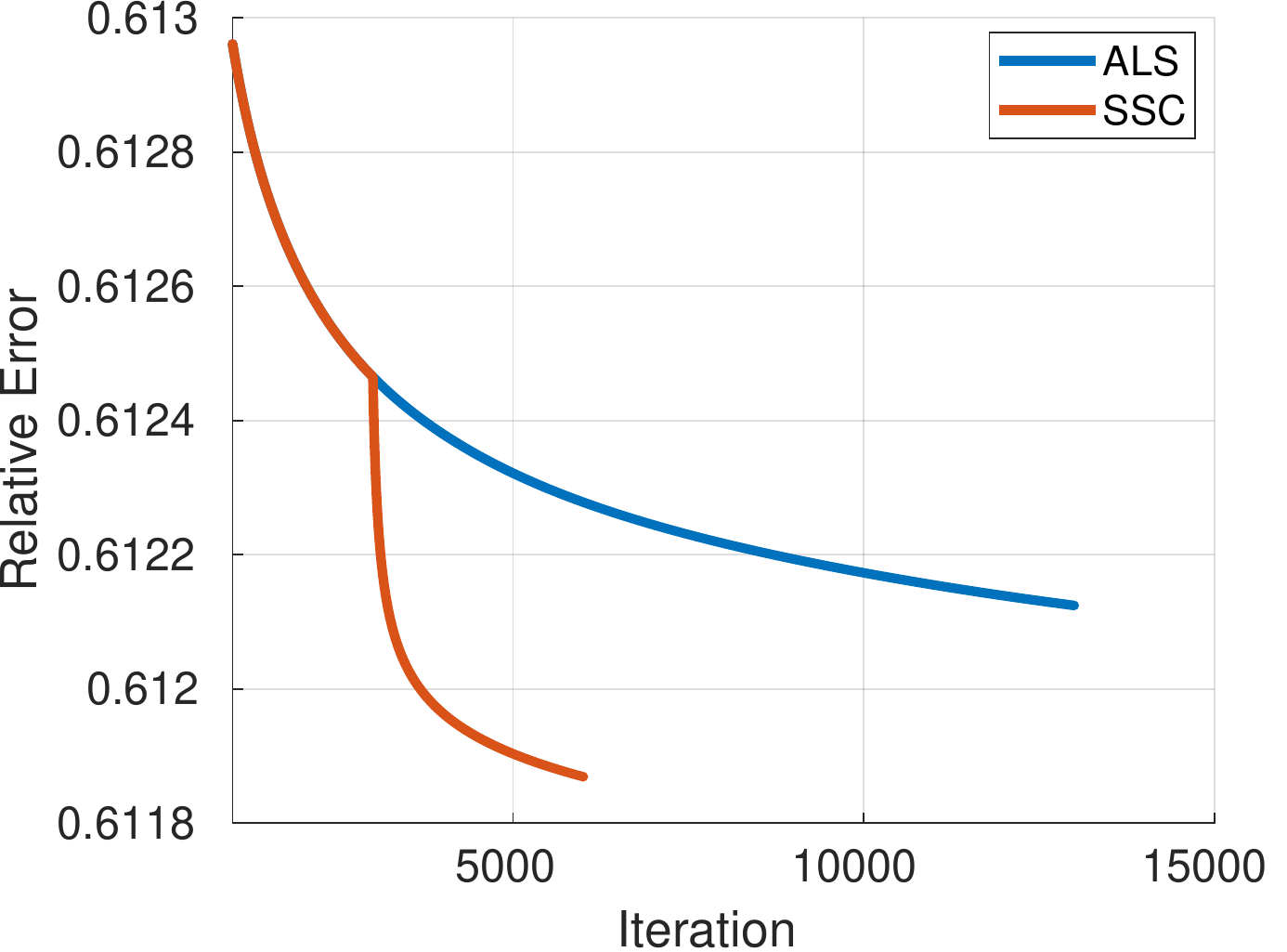}}
\caption{Relative approximation errors as function of iterations illustrate convergence of the decomposition of convoluional kernels using ALS and SSC.}\label{fig_ex3_err_it}
\end{figure}

\begin{figure}[!ht]
\centering
\subfigure[Convolutional layer 3]{\includegraphics[width=.32\linewidth, trim = 0.0cm 0cm 0cm 0cm,clip=true]{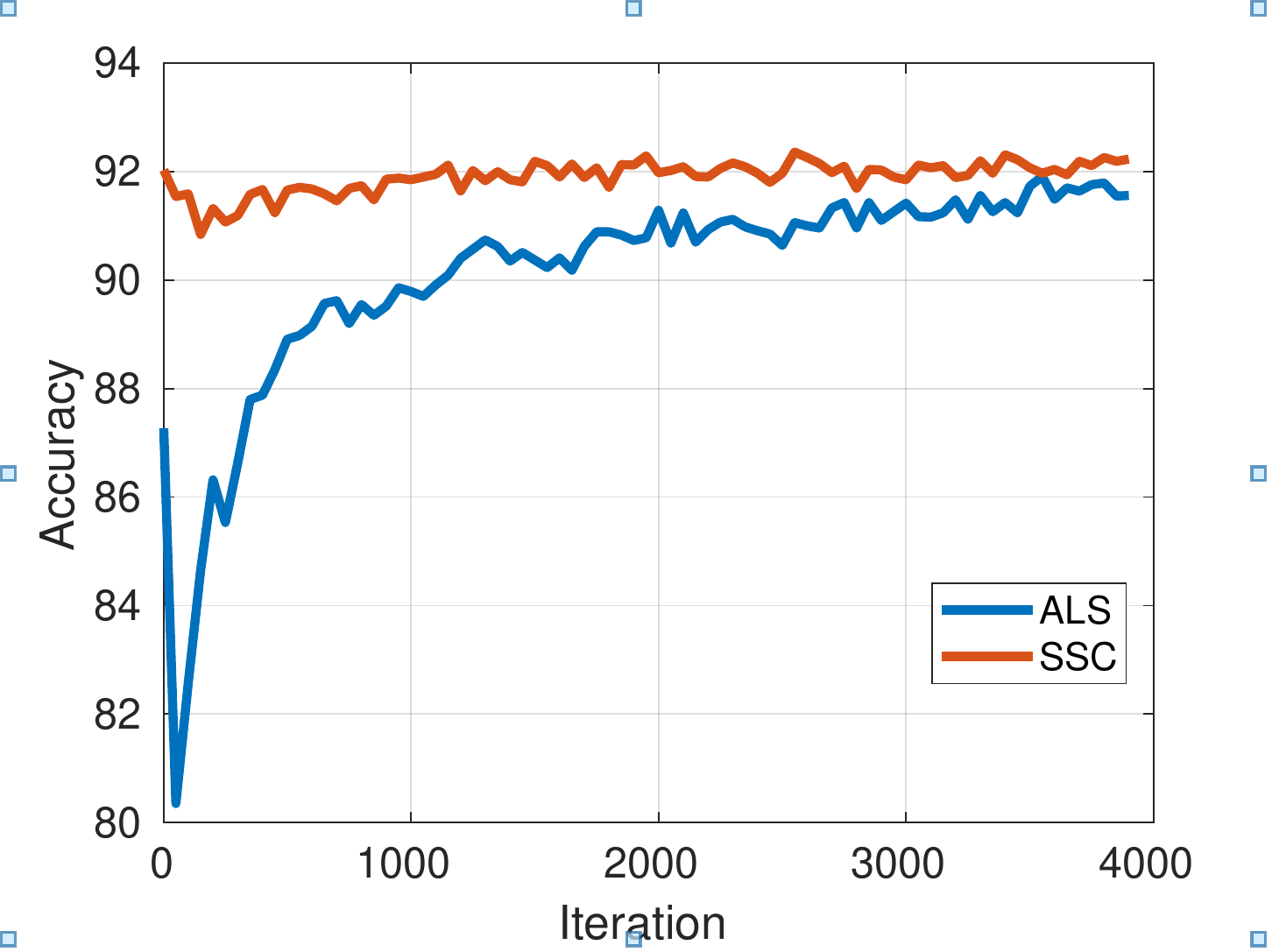}}
\subfigure[Convolutional layer 4]{\includegraphics[width=.32\linewidth, trim = 0.0cm 0cm 0cm 0cm,clip=true]{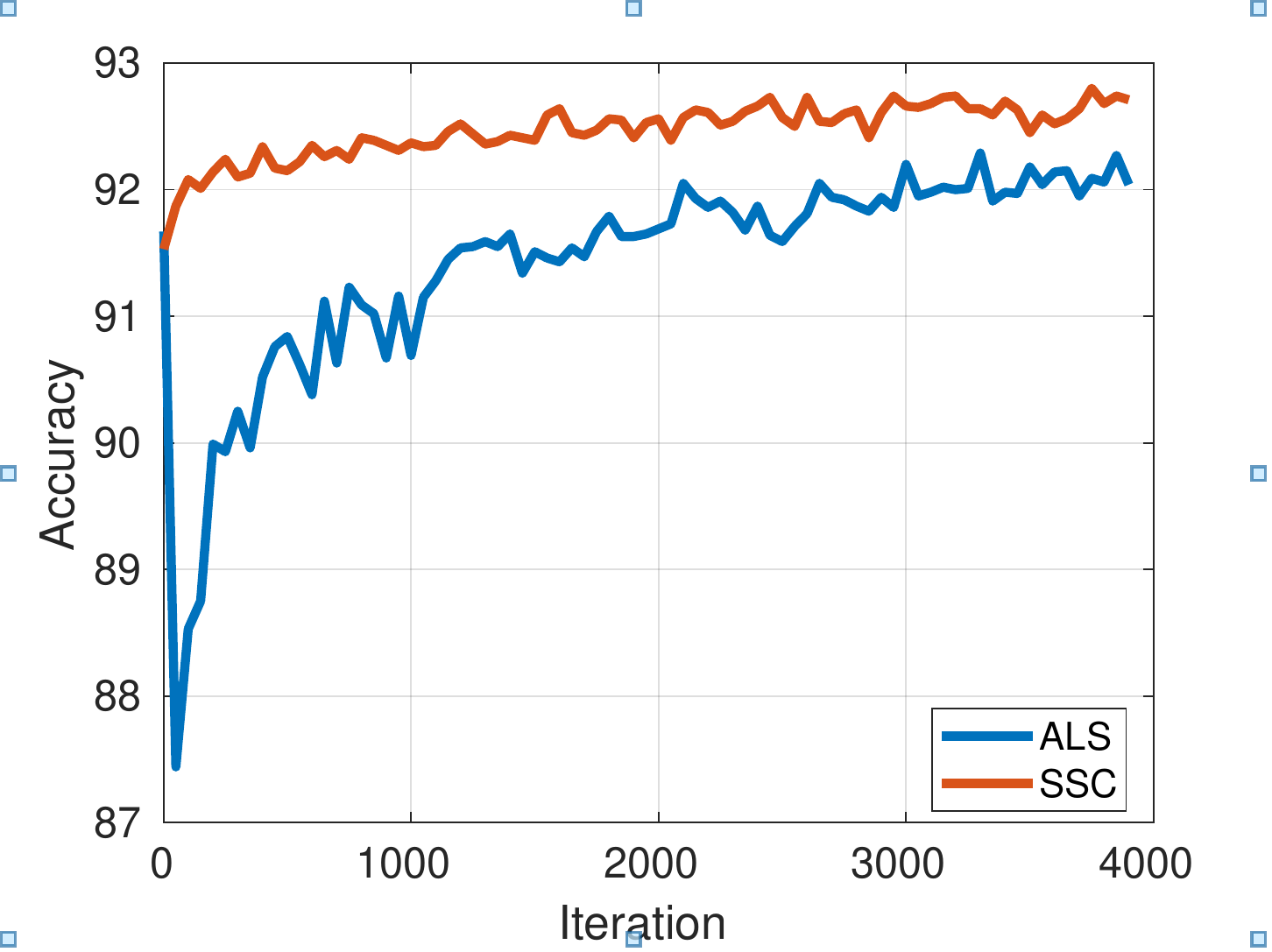}}
\subfigure[Convolutional layer 5]{\includegraphics[width=.32\linewidth, trim = 0.0cm 0cm 0cm 0cm,clip=true]{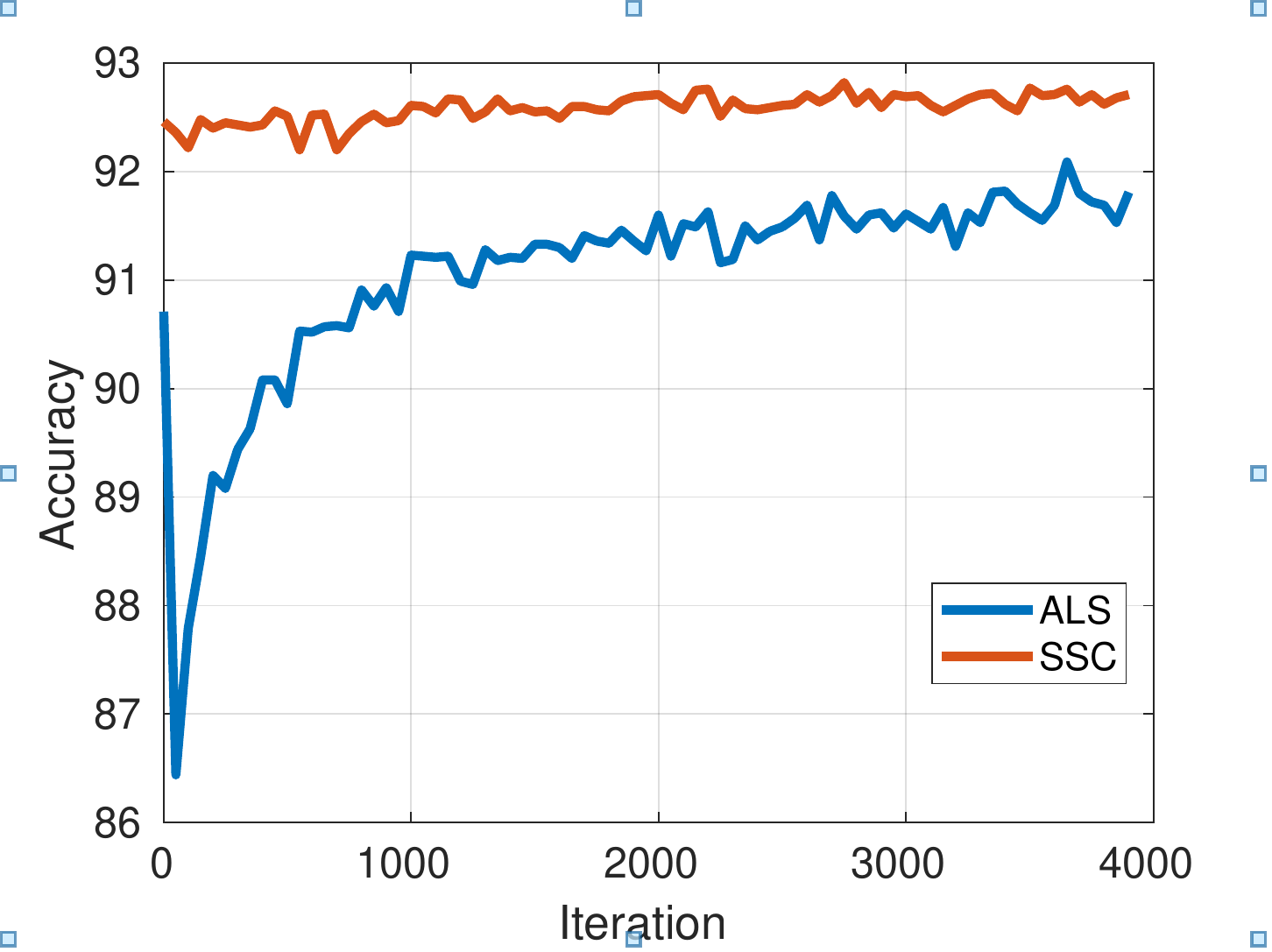}}
\subfigure[Convolutional layer 6]{\includegraphics[width=.32\linewidth, trim = 0.0cm 0cm 0cm 0cm,clip=true]{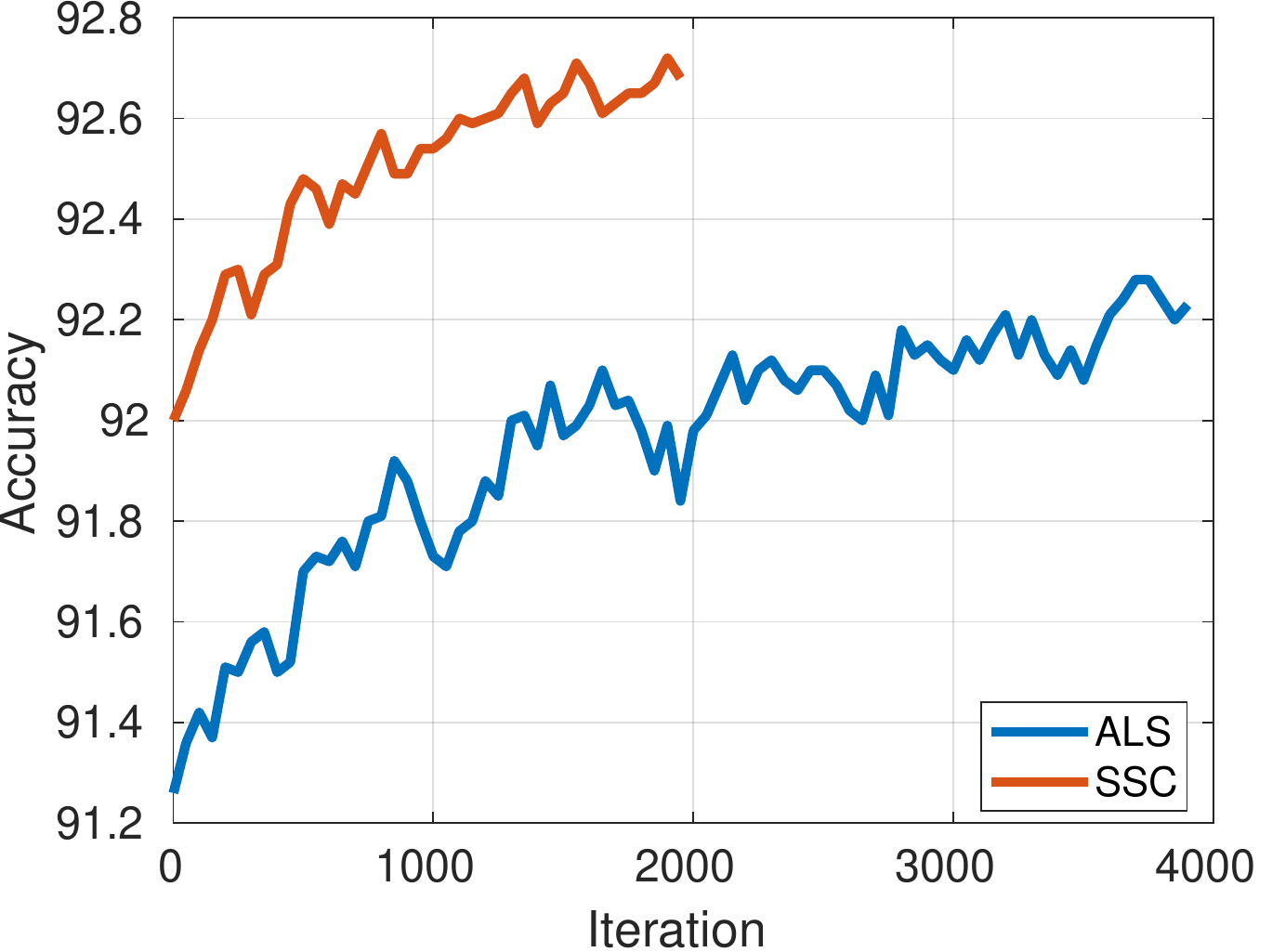}}
\subfigure[Convolutional layer 9]{\includegraphics[width=.32\linewidth, trim = 0.0cm 0cm 0cm 0cm,clip=true]{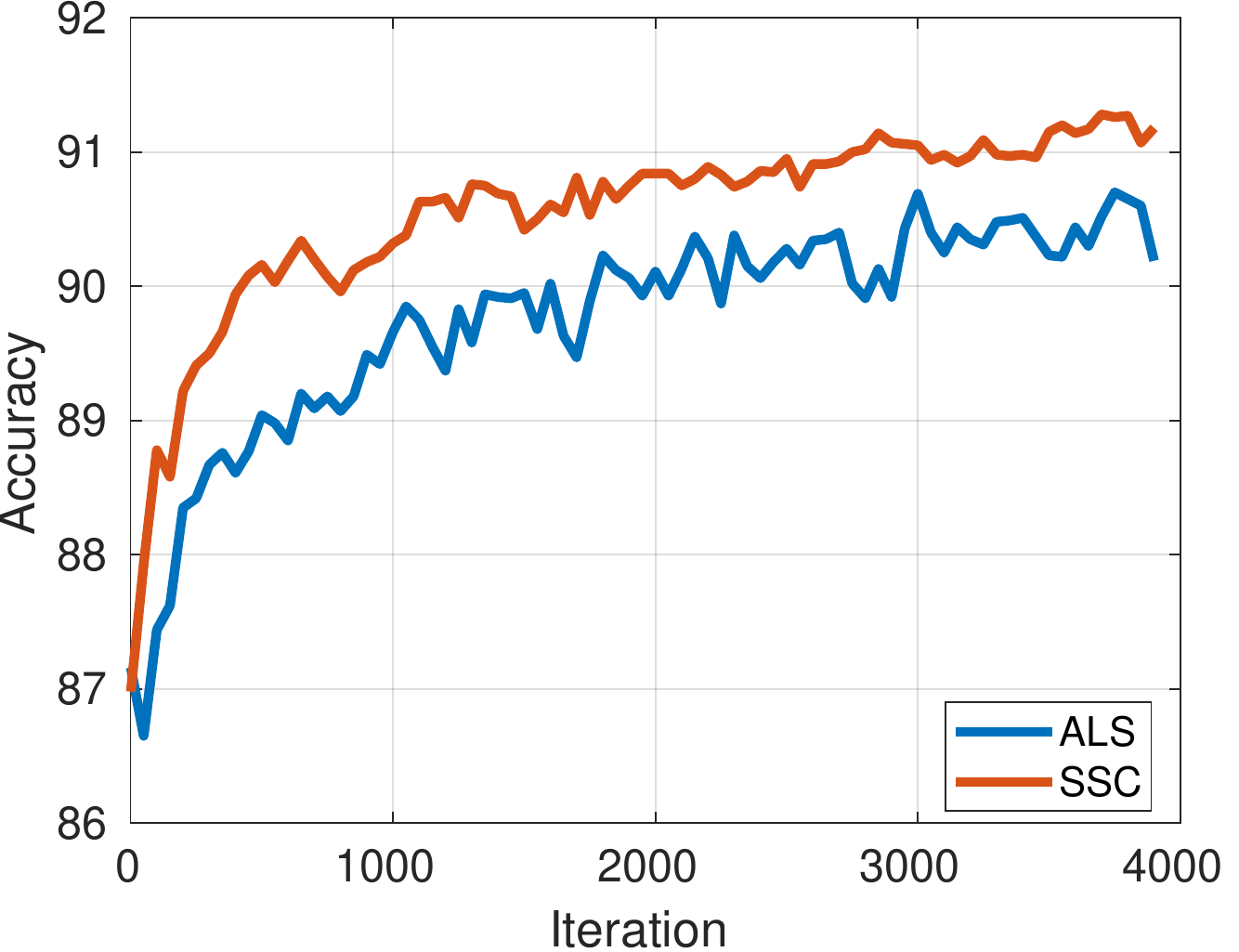}}
\subfigure[Convolutional layer 10]{\includegraphics[width=.32\linewidth, trim = 0.0cm 0cm 0cm 0cm,clip=true]{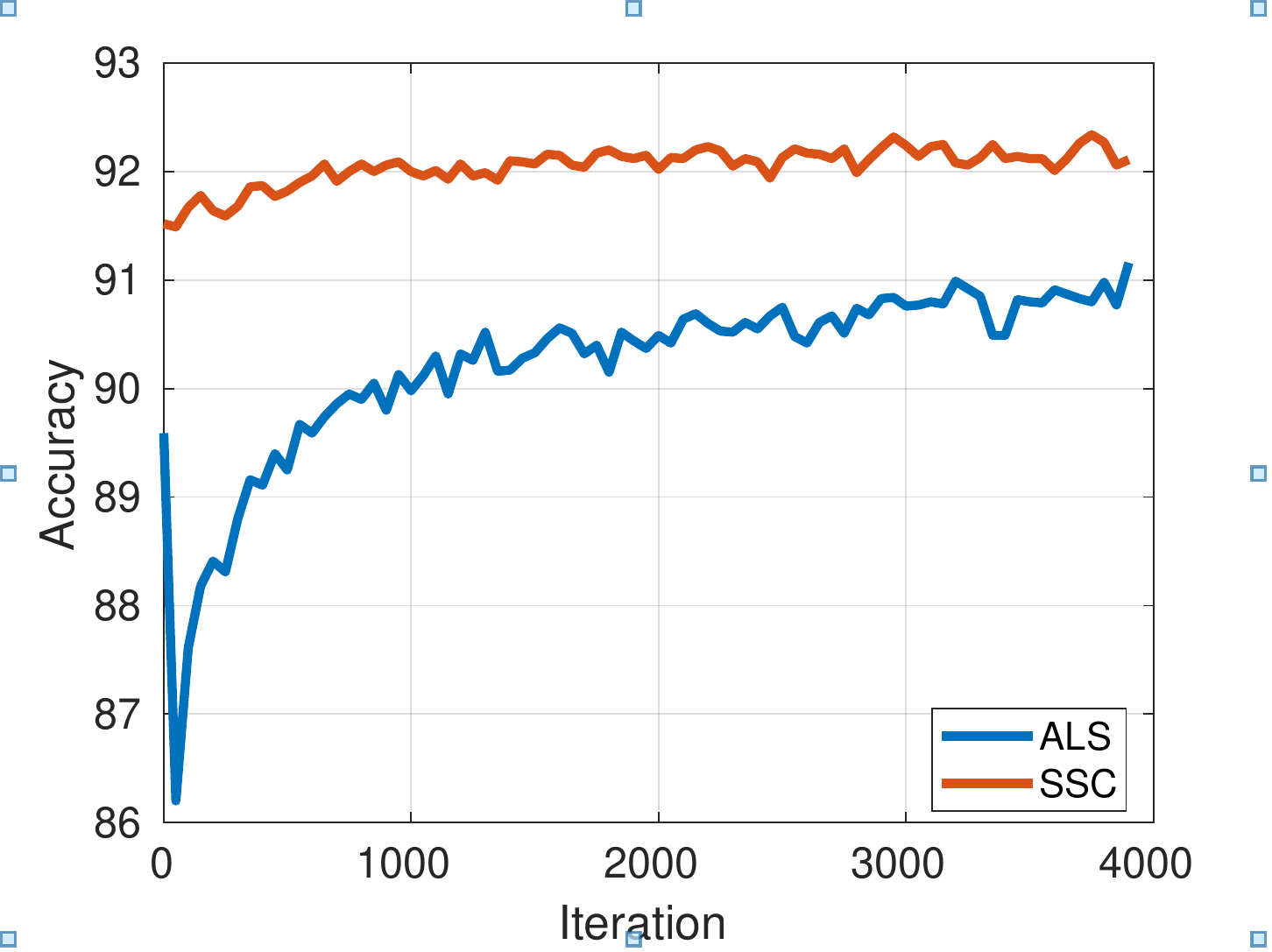}}
\subfigure[Convolutional layer 11]{\includegraphics[width=.32\linewidth, trim = 0.0cm 0cm 0cm 0cm,clip=true]{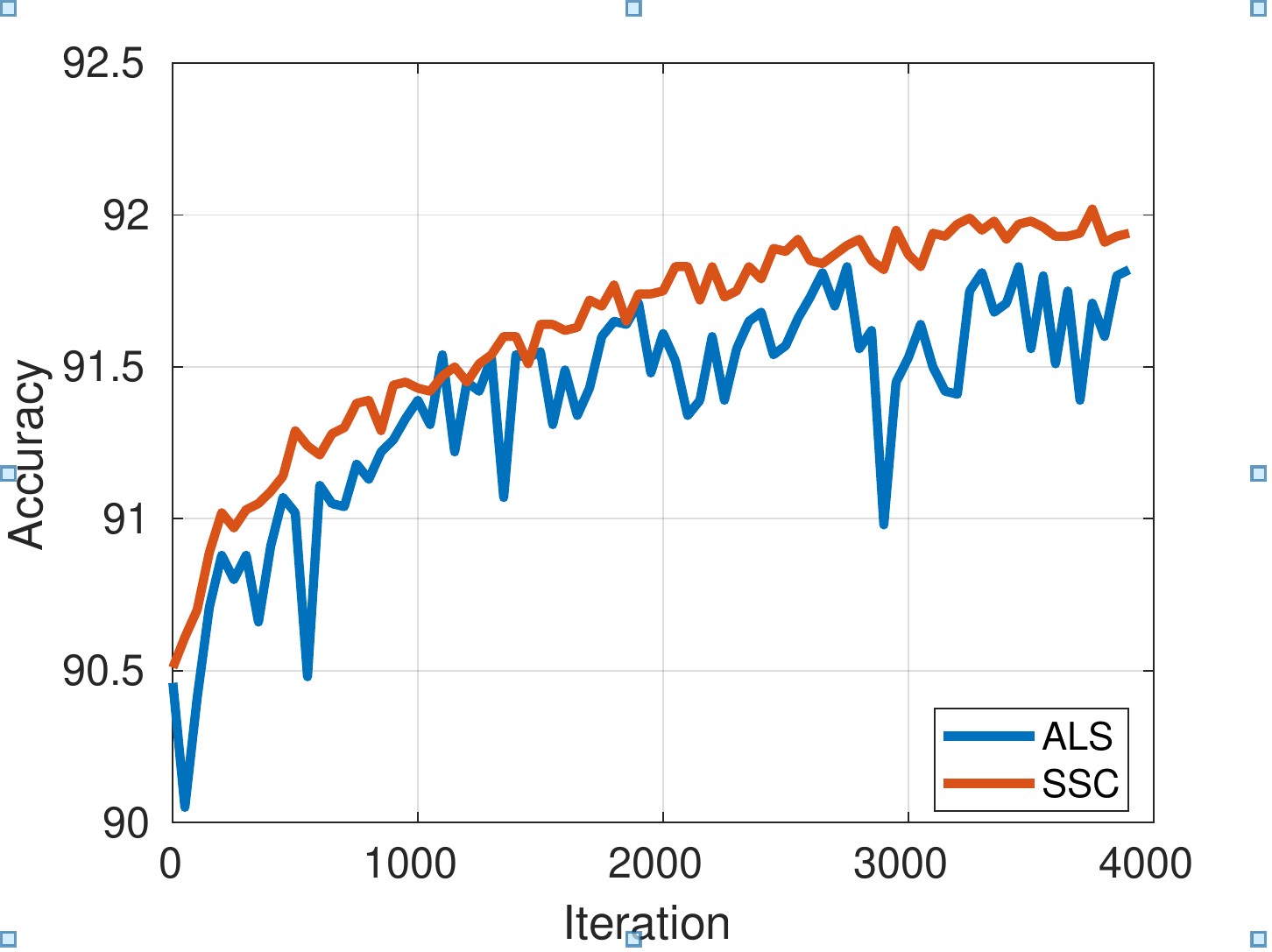}}
\subfigure[Convolutional layer 12]{\includegraphics[width=.32\linewidth, trim = 0.0cm 0cm 0cm 0cm,clip=true]{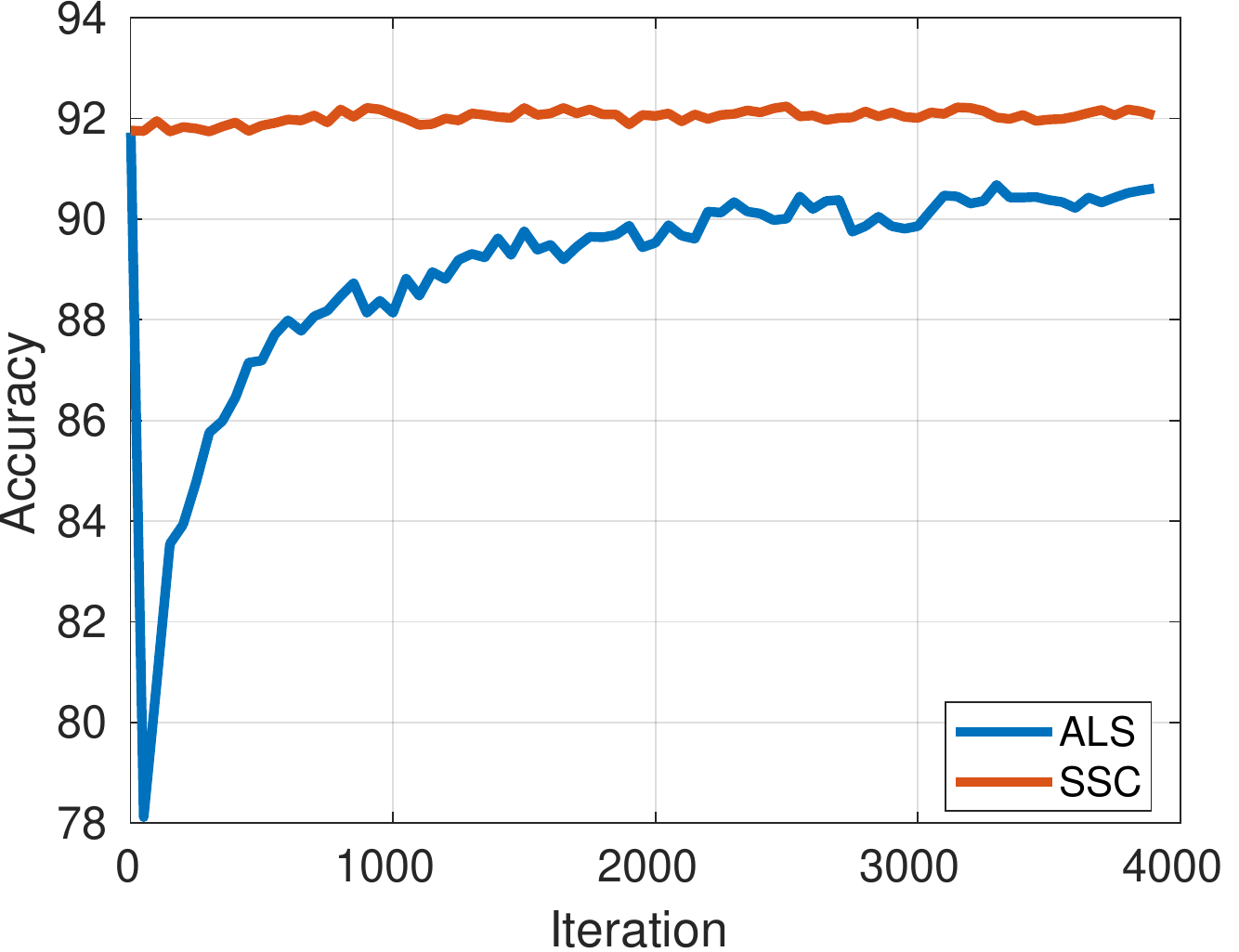}}
\subfigure[Convolutional layer 14]{\includegraphics[width=.32\linewidth, trim = 0.0cm 0cm 0cm 0cm,clip=true]{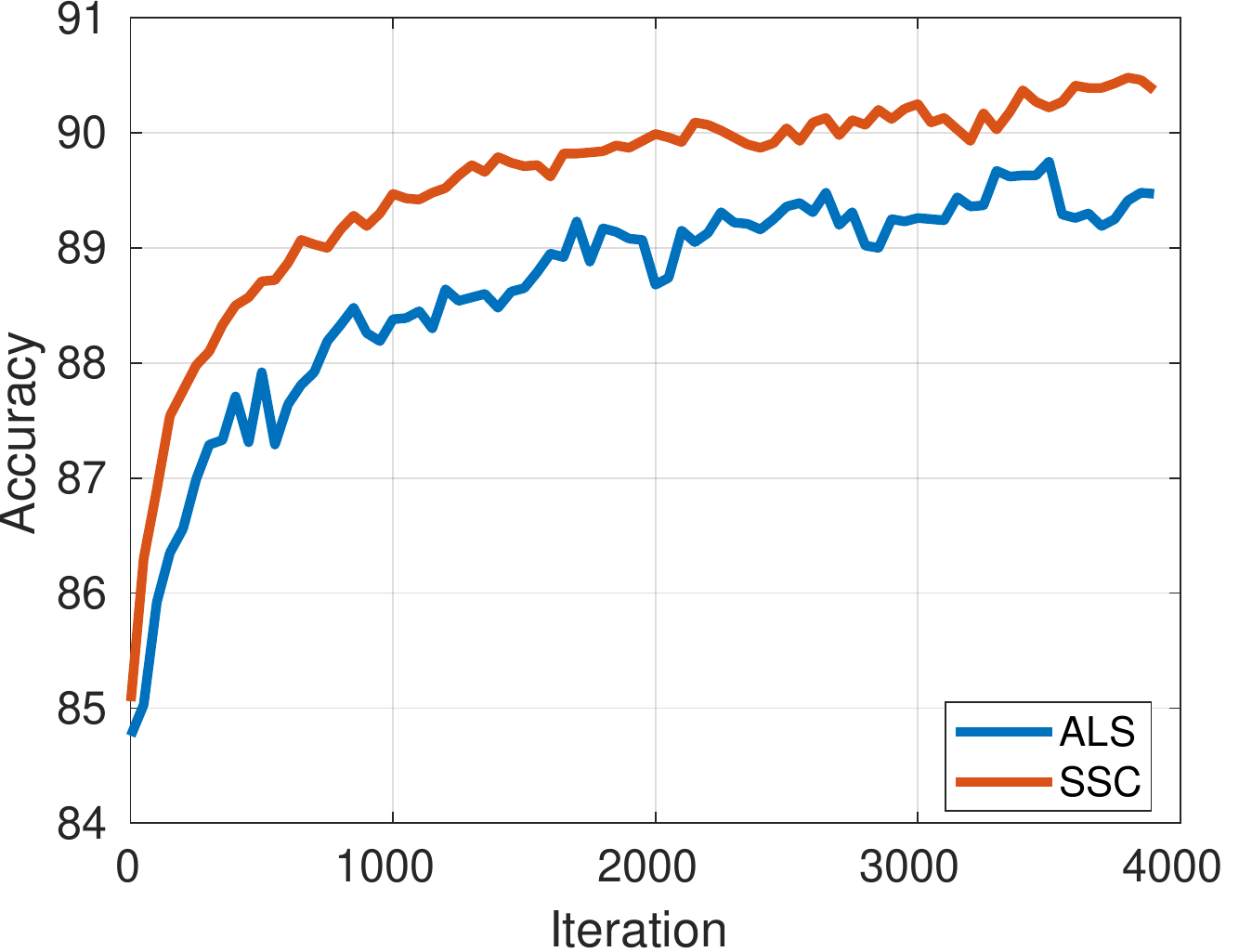}}
\subfigure[Convolutional layer 15]{\includegraphics[width=.32\linewidth, trim = 0.0cm 0cm 0cm 0cm,clip=true]{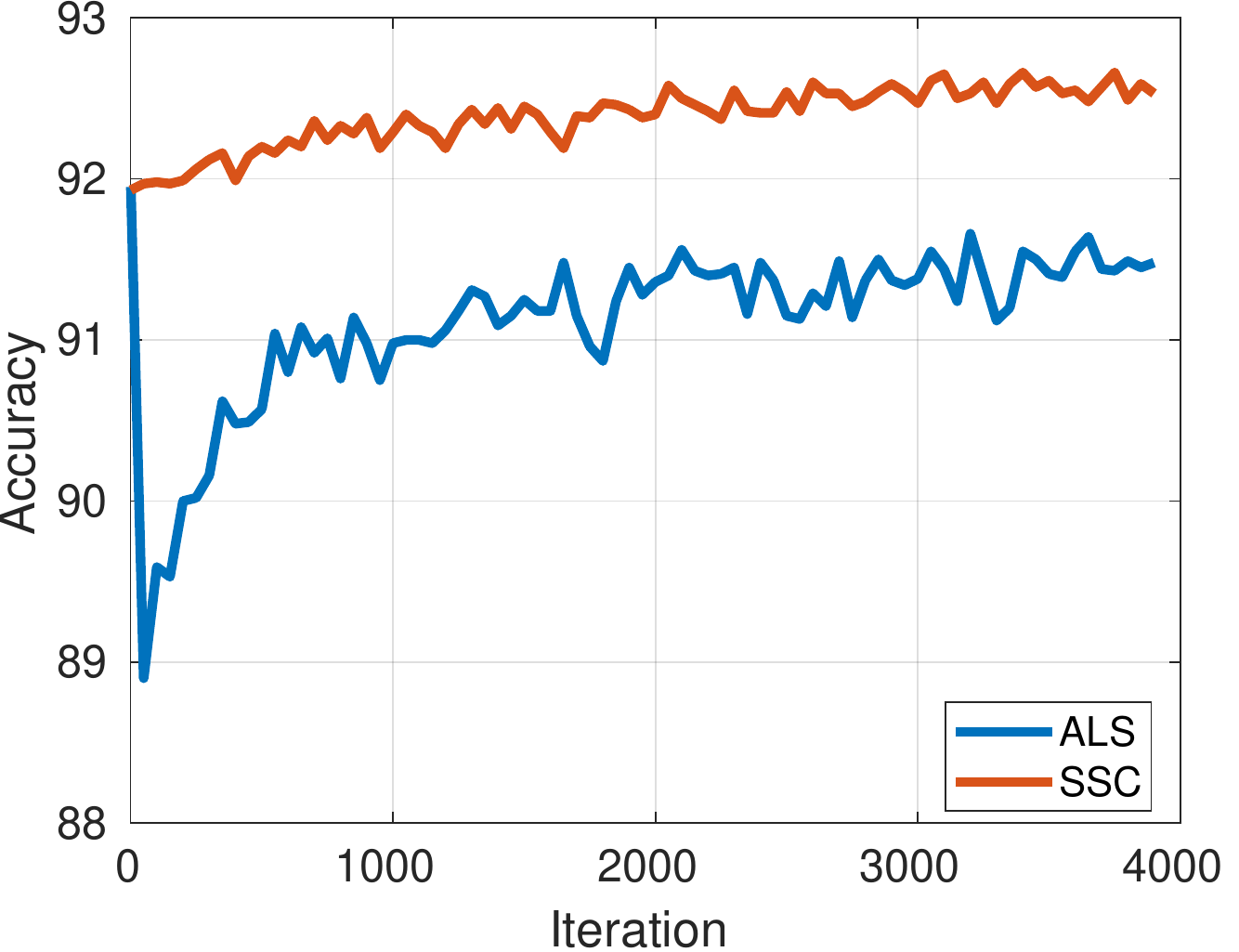}}
\subfigure[Convolutional layer 16]{\includegraphics[width=.32\linewidth, trim = 0.0cm 0cm 0cm 0cm,clip=true]{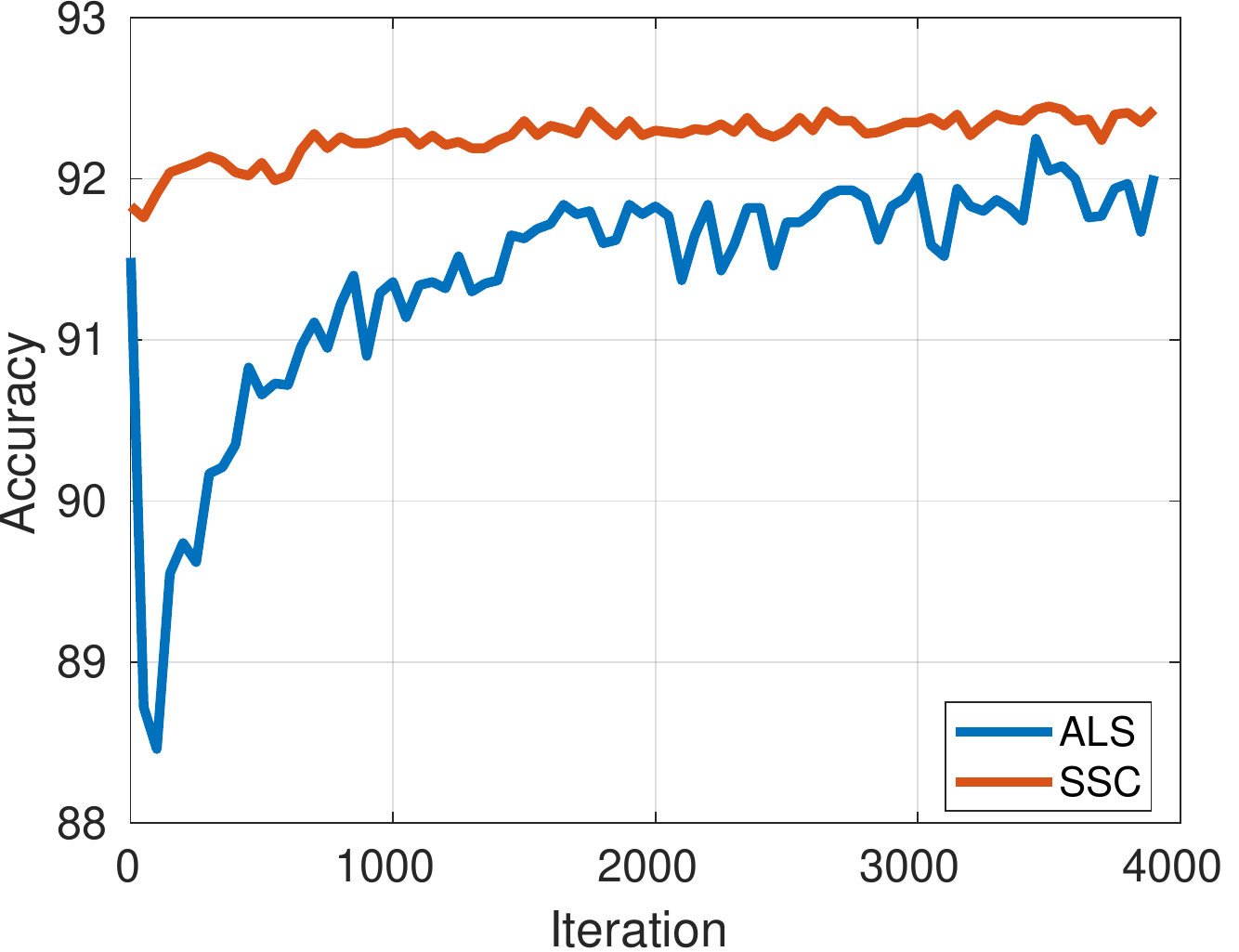}}
\subfigure[Convolutional layer 17]{\includegraphics[width=.32\linewidth, trim = 0.0cm 0cm 0cm 0cm,clip=true]{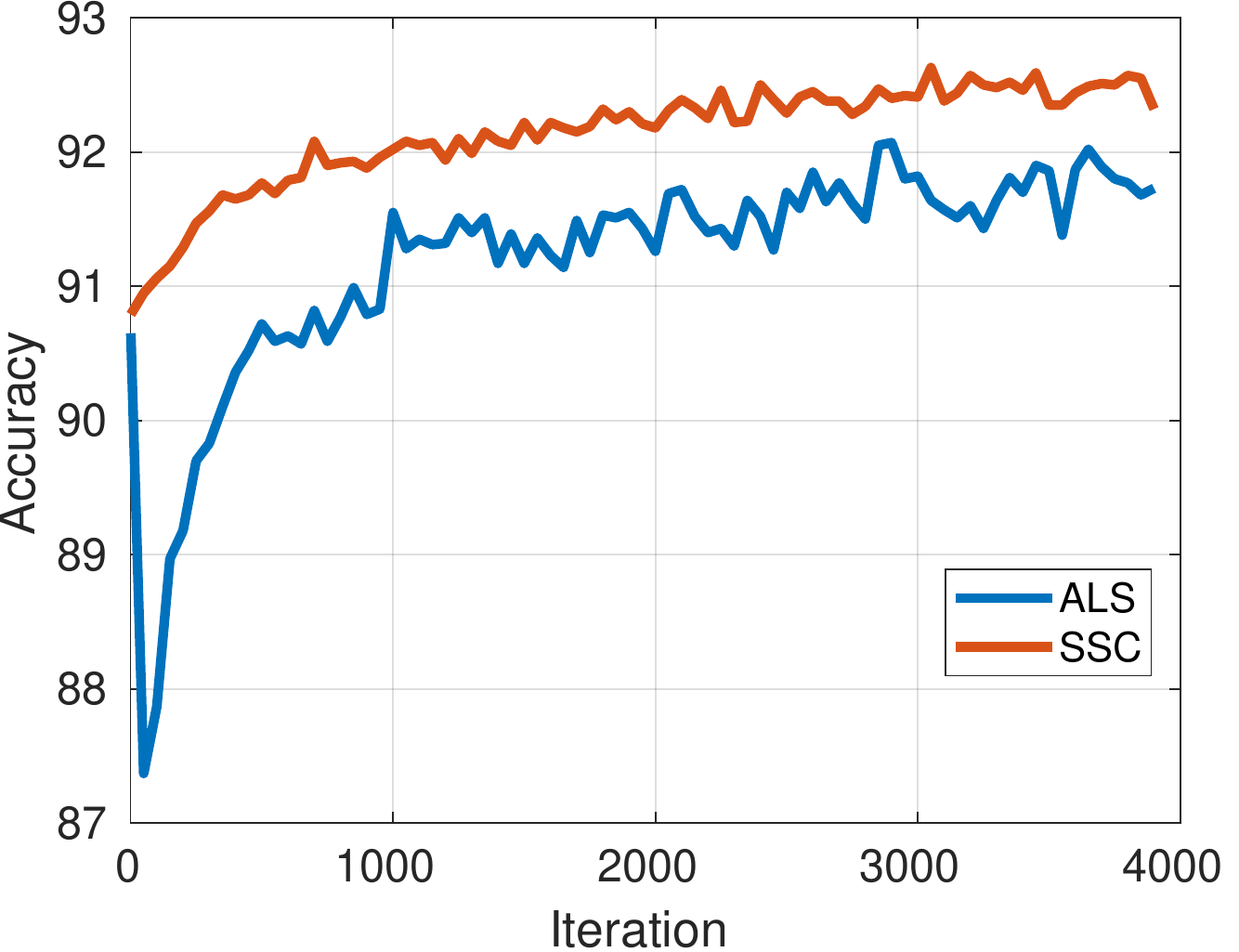}}
\caption{Learning curves for single layer compression of ResNet-18 finetuned on CIFAR-10 for Example~\ref{ex_resnet18_cifar10_v2}. Neural networks with TC layers whose convolutional coefficients are obtained by SSC are easier to train and converge faster than the same neural networks with weights estimated using ALS.}\label{fig_ex3_b}
\end{figure}

\begin{figure}[!ht]
\centering
{\includegraphics[width=.45\linewidth, trim = 0.0cm 0cm 0cm 0cm,clip=true]{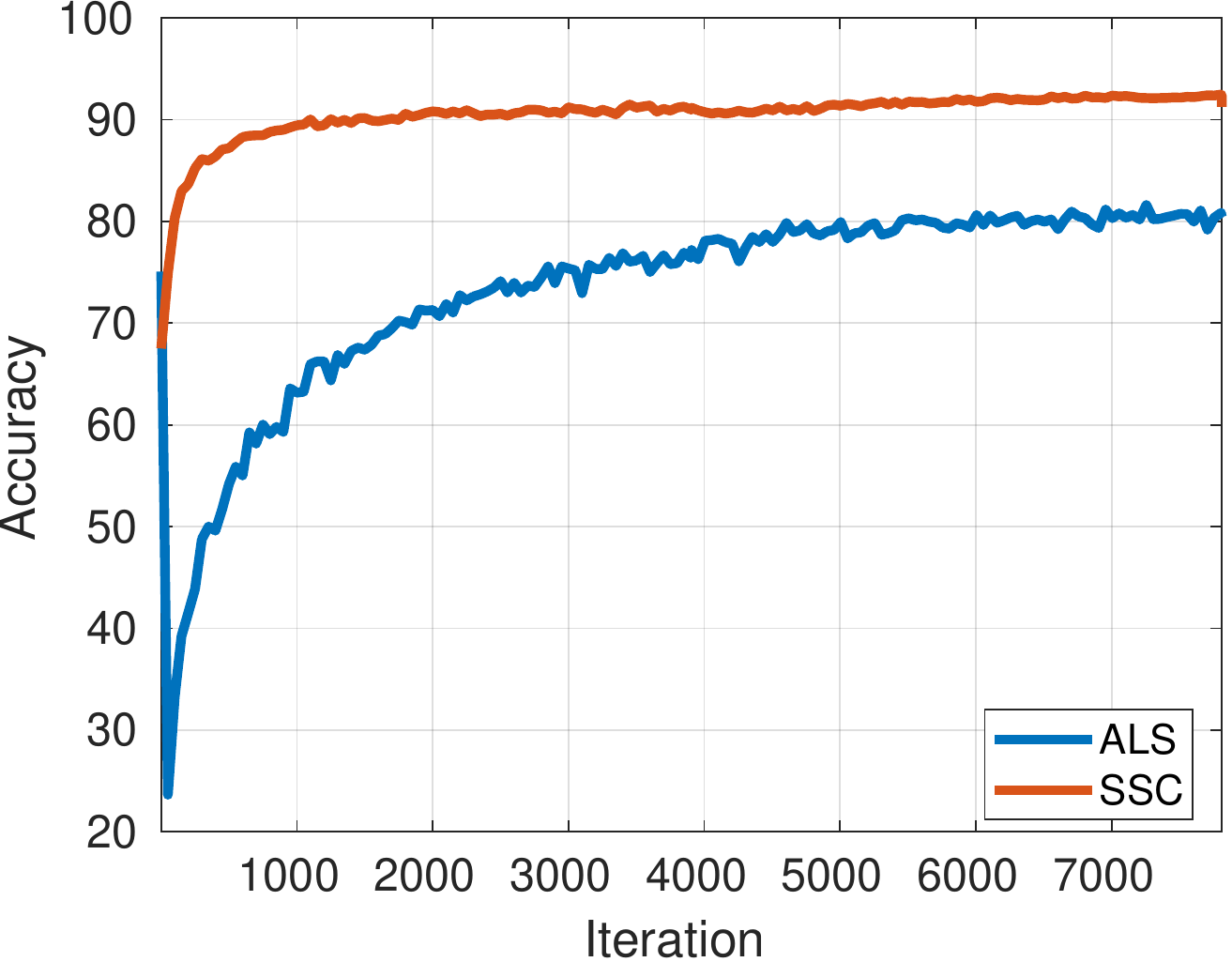}}
\caption{Accuracy of two new ResNet-18 with all convolutional kernels replaced by TC layers in Example~\ref{ex_resnet18_cifar10_v2}. The network whose kernels are estimated using ALS could not attain the original accuracy of ResNet-18 for CIFAR-10, 92.29\%. 
The network can reach the initial accuracy if the weights in all TC-layers are estimated using SSC.}\label{fig::resnet18_cifar_fullcompression}
\end{figure}

\begin{figure}[!ht]
\centerline{\includegraphics[width=.5\linewidth]{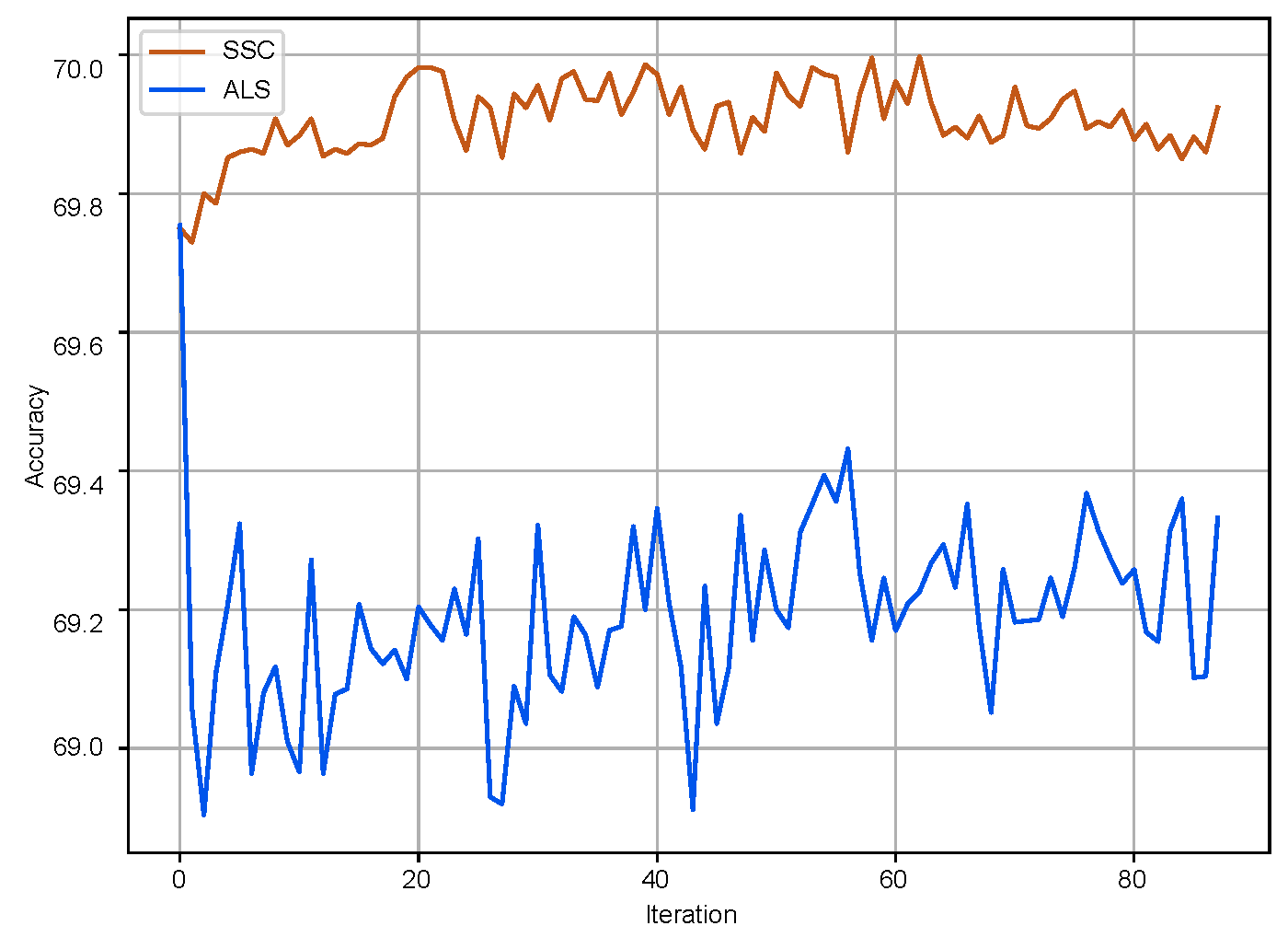}}
\caption{Accuracy vs iteration (one iteration is equal to 500 gradient steps with batch size 256) of single layer fine-tuning for layer4.1.conv2 trained on ILSVRC-12.}
\label{fig::resnet18_imagennet_layerwise}
\end{figure}

\end{document}